\documentclass{article}

\usepackage{microtype}
\usepackage{graphicx}
\usepackage{subfigure}
\usepackage{booktabs} 
\usepackage{multirow}
\usepackage{makecell}
\usepackage{bm}
\usepackage{colortbl}
\usepackage{pifont}

\usepackage{hyperref}

\usepackage[accepted]{icml2024}

\usepackage{amsmath}
\usepackage{amssymb}
\usepackage{mathtools}
\usepackage{amsthm}

\usepackage[capitalize,noabbrev]{cleveref}

\theoremstyle{plain}
\newtheorem{theorem}{Theorem}[section]
\newtheorem{proposition}[theorem]{Proposition}

\newtheorem{corollary}[theorem]{Corollary}
\theoremstyle{definition}

\theoremstyle{remark}

\usepackage[textsize=tiny]{todonotes}
            
\newcommand{\red}[1]{{\leavevmode\color{black}#1}}

\icmltitlerunning{Weakly-Supervised Residual Evidential Learning for Multi-Instance Uncertainty Estimation}

\begin{document}

\twocolumn[
\icmltitle{Weakly-Supervised Residual Evidential Learning for \\ Multi-Instance Uncertainty Estimation}



\icmlsetsymbol{equal}{*}

\begin{icmlauthorlist}
\icmlauthor{Pei Liu}{uestc}
\icmlauthor{Luping Ji}{uestc}
\end{icmlauthorlist}

\icmlaffiliation{uestc}{School of Computer Science and Engineering, University of Electronic Science and Technology of China, Chengdu, China}

\icmlcorrespondingauthor{Luping Ji}{jiluping@uestc.edu.cn}

\icmlkeywords{Multiple Instance Learning, MIL, Uncertainty Estimation, Evidential Deep Learning, Weakly-Supervised Learning}

\vskip 0.3in
]



\printAffiliationsAndNotice{}  

\begin{abstract}
Uncertainty estimation (UE), as an effective means to quantify predictive uncertainty, is crucial for safe and reliable decision-making, especially in high-risk scenarios.
Existing UE schemes usually assume that there are completely-labeled samples to support fully-supervised learning.
In practice, however, many UE tasks often have no sufficiently-labeled data to use, such as the Multiple Instance Learning (MIL) with only weak instance annotations.
To bridge this gap, this paper, for the first time, addresses the weakly-supervised issue of \emph{Multi-Instance UE} (MIUE) and proposes a new baseline scheme, \emph{Multi-Instance Residual Evidential Learning} (MIREL). 
Particularly, at the fine-grained instance UE with only weak supervision, we derive a multi-instance residual operator through the Fundamental Theorem of Symmetric Functions. On this operator derivation, we further propose MIREL to jointly model the high-order predictive distribution at bag and instance levels for MIUE. 
Extensive experiments empirically demonstrate that our MIREL not only could often make existing MIL networks perform better in MIUE, but also could surpass representative UE methods by large margins, especially in instance-level UE tasks. Our source code is available at \href{https://github.com/liupei101/MIREL}{https://github.com/liupei101/MIREL}.
\end{abstract}

\section{Introduction}
\label{intro}
Deep learning models have shown impressive capability and become ubiquitous in the last decade. 
However, they often tend to produce overconfident predictions, even for shifted or unseen samples \cite{Nguyen7298640,kendall2017uncertainties}. Such behaviour may lead to disastrous consequences in safety-critical scenarios, \textit{e.g.}, autonomous driving and medical diagnosis \cite{franchi2022muad,linmans2023predictive}, calling into question their real-world usability.  
Thus, it is particularly important to provide an accurate confidence level for the prediction of neural networks (NNs) through uncertainty estimation (UE) methods. 

To accomplish accurate UE, epistemic (model) uncertainty is considered generally \cite{pmlr-v162-postels22a,mukhoti2023deep}. From a principled Bayesian perspective \cite{neal2012bayesian}, it is characterized by the distribution over model parameters given training data $\mathcal{D}$, \textit{i.e.}, $p(\bm{\omega}|\mathcal{D})$. Due to its involvement, a model parameter $\bm{\omega}$ would be low in likelihood when it is incompatible with $\mathcal{D}$; as a result, less confidence would be yielded for the new samples shifted in distribution, thus alleviating overconfidence in prediction. 
However, most current practices often assume that there are \textit{completely-labeled} samples in $\mathcal{D}$ with which NNs can be trained to support $p(\bm{\omega}|\mathcal{D})$ \cite{mena2021survey}. 

In fact, there are many practical tasks involving \textit{weakly-annotated} data, in which no complete label can be directly utilized for training. 
While such tasks remain under-explored in UE, a fundamental machine learning problem we consider under this class is multiple instance learning (MIL) \cite{dietterich1997solving}. As a typical task of weakly-supervised learning, it is prominent in many labeling-intensive applications, \textit{e.g.}, histopathology diagnosis \cite{ilse2020deep,liu2024advmil}, \red{video anomaly detection \cite{sultani2018real,zhong2019graph} and} video analysis \cite{babenko2010robust,Rizve_2023_CVPR}, etc. \red{
Specifically, in histopathology diagnosis an image usually contains gigapixels, so it is often divided into thousands of small patches for MIL, where multiple patches (instances) are observed but only a general statement of their labels is given. In this case, a diagnosis model has to learn from weakly-annotated patches to make patch-level predictions. UE is highly anticipated to provide accurate uncertainty measures for these weakly-supervised predictions to make final diagnostic decisions safer and more reliable.}

Generally, a sample given in MIL is described as a bag $X$ and its label is known, $Y\in \{0,1\}$.
In particular, $X$ is composed of multiple instances, \textit{i.e.}, $X = \{\mathbf{x}_1, \dots , \mathbf{x}_{K}\}$, but instance labels $\{y_1, \dots , y_{K}\}$ are \textit{unknown} and instead connect with bag label via a classical MIL assumption, $Y=\mathop{\max}_k \{y_k\}$ and $k\in[1,K]$. Under these special settings, MIL is usually interested in two tasks, bag-level and instance-level prediction \cite{kandemir2015computer}. This means that, a MIL model needs to \textbf{i}) learn $p(\bm{\omega}|\mathcal{D})$ from the multi-instance bags with \textit{variable} sizes and meanwhile \textbf{ii}) jointly estimate a new \textit{weakly-supervised posterior} $p(\bm{\theta_{\mathrm{w}}}|\mathcal{D})$ from weakly-annotated instances. 
Therefore, in such MIL tasks, the common way of capturing epistemic uncertainty for UE seems less practical.
This motivates us to focus on the problem of Multi-Instance UE (termed \textbf{MIUE}) and study a baseline approach for it. 

\textbf{Bag-level}~The Fundamental Theorem of Symmetric Functions \cite{zaheer2017deep} provides a general strategy to score a bag of instances. It deals with size-varied bags by a permutation-invariant MIL pooling operator \cite{ilse2018attention}. Accordingly, $p(\bm{\omega}|\mathcal{D})$ could be estimated from fully-labeled bags using common UE techniques to capture bag-level epistemic uncertainty, like that in fully-supervised learning \cite{mena2021survey}.

\textbf{Instance-level}~Without complete instance labels, modeling predictive uncertainty at instance level would not be as straightforward as that at bag level. Nonetheless, attention-based MIL \cite{ilse2018attention,li2021dual} still could generate instance predictions with their instance scoring proxy---\textit{attention branch}. Following this approach, various strategies \cite{shi2020loss,qu2022bi,yufei2022bayes} are proposed to make instance prediction more accurate. 
We argue that such attention-dependent means may not be suitable for learning $p(\bm{\theta_{\mathrm{w}}}|\mathcal{D})$ jointly with $p(\bm{\omega}|\mathcal{D})$.
Because i) the parameter $\bm{\theta_{\mathrm{w}}}$ given by attention branch for instances is completely contained within the parameter $\bm{\omega}$ for bags, and notably ii) the attention scores given by that proxy are produced essentially for learning better bag representations, leaving a substantial gap to ideal instance predictions.

In this paper, with the Fundamental Theorem of Symmetric Functions, we demonstrate that the gap to ideal instance prediction can be narrowed by turning to exploit a good bag-level decision space. 
With this basic finding, we propose a new MIL scheme for MIUE. 
Concretely, (1) we devise a new instance estimator to jointly learn $p(\bm{\theta_{\mathrm{w}}}|\mathcal{D})$ by deriving a \textit{multi-instance residual operator}. This operator makes instance prediction separated from bag decision. 
(2) Further, we model high-order probability distribution at bag and instance levels to fulfill MIUE, by parameterizing two Dirichlet distributions with the evidences provided by general bag estimator and our residual instance estimator. 
(3) To optimize $\bm{\theta_{\mathrm{w}}}$ without complete instance labels, we propose a weakly-supervised evidence learning strategy and prove that it provides a tighter upper bound for ideal instance loss function than common strategies under given conditions. 

The main contributions of this paper are summarized as follows: (1) A new problem of uncertainty estimation, termed MIUE (\emph{Multi-Instance UE}), is introduced in this paper. To our knowledge, we are the first to study it in MIL. (2) With the Fundamental Theorem of Symmetric Functions, this paper demonstrates that a good estimator for instances can be directly deduced from a good bag-level decision space, no longer relying on the scoring proxy from attention-based MIL. (3) This paper further derives a residual estimator specially for instances, and proposes a new scheme, \emph{Multi-Instance Residual Evidential Learning} (MIREL), for MIUE. This scheme can jointly quantify the predictive uncertainty at bag and instance levels in MIL. 

\section{Preliminary}

\subsection{Multiple Instance Learning (MIL)}
\label{sec-mil}

\textbf{Definition}~Here we give the formal conventions and notations in MIL, following \citeauthor{ilse2018attention}~\citeyearpar{ilse2018attention}. A given sample (bag) is denoted as $X = \{\mathbf{x}_1, \dots , \mathbf{x}_{K}\}$, where $\mathbf{x}_1, \dots , \mathbf{x}_{K}$ are usually treated to be \textit{i.i.d.} 
Its label, $Y\in \{0, 1\}$, is accessible for training; its instance-level labels, $\{y_1, \dots , y_{K}\}$, are \textit{unknown} and $y_k\in \{0, 1\}$ for $k\in[1,K]$. A classical MIL assumption states that, a bag is positive ($Y=1$) \textit{iff} it has at least one positive instance; otherwise, it is negative ($Y=0$). Namely, there is $Y=\mathop{\max}_k \{y_k\}$. 

\textbf{Learning paradigm}~To learn from size-varied bags, a common practice is to leverage a \textit{permutation-invariant} pooling operator. It can be expressed by the Fundamental Theorem of Symmetric Functions as follows:

\begin{theorem}[\cite{zaheer2017deep,ilse2018attention}]
\label{thm:symfunc}
A scoring function for a set of instances $X=\{\mathbf{x}_1, \dots , \mathbf{x}_{K}\}$, $S(X)\in \mathbb{R}$, is a symmetric function (permutation-invariant to the elements in $X$), if and only if it can be written as 
\begin{equation}
S(X) = g\big(\sum_{k=1}^{K} f(\mathbf{x}_{k})\big),
\label{eq1}
\end{equation}
where $f$ and $g$ are suitable transformations.
\end{theorem}

This theorem holds as before or under weak conditions, when the form of instance pooling in Eq.(\ref{eq1}), $\sum_{k} f(\mathbf{x}_{k})$, is replaced by others, such as i) $\mathrm{mean}$, ii) $\mathrm{max}$ \cite{qi2017pointnet}, and iii) attention-based MIL pooling. 

\textbf{Attention-based MIL pooling}~The most representative one is ABMIL \cite{ilse2018attention}. It first proposes to leverage dynamic instance weights for pooling, written as $\sum_{k=1}^{K} a_{k}f(\mathbf{x}_k)$, where $a_k$ is called attention score and 
\begin{equation}
a_k = \mathop{\mathrm{softmax}}\big(t(\mathbf{h}_k)\big) = \frac{\mathop{\exp}\big(t(\mathbf{h}_k)\big)}{\sum_{\tau=1}^{K}\mathop{\exp}\big(t(\mathbf{h}_\tau)\big)}.
\label{eq2}
\end{equation}
$\mathbf{h}_k$ stands for the instance embedding given by $\mathbf{h}_k=f(\mathbf{x}_k)$ and $t(\cdot)$ is a transformation parameterized by NNs. 

Owing to attention mechanism, most attention-based MIL networks can provide a \textit{proxy} (\textit{i.e.}, attention score $a_k$) to estimate instance labels, as highlighted in \citeauthor{ilse2018attention}~\citeyearpar{ilse2018attention} and \citeauthor{li2021dual}~\citeyearpar{li2021dual}. This proxy is frequently used and improved afterwards \cite{qu2022bi,yufei2022bayes}, often taken as a reliable estimator for instances. 

\subsection{Evidential Deep Learning (EDL)}

As one of general UE methods, recently-proposed EDL \cite{sensoy2018evidential} models predictive uncertainty using the Dempster–Shafer Theory of Evidence (DST) \cite{dempster1968generalization}. It formalizes the belief assignment in DST with Subjective Logic (SL) \cite{josang2016subjective}. 

\textbf{Belief assignment}~Considering $C$ ($C\geq 2$) mutually exclusive singletons (\textit{e.g.}, class labels), SL assigns a belief mass $b_i$ to the $i$-th singleton for $i\in[1, C]$ and defines an overall uncertain mass $u$. Let $b_i=\frac{e_i}{\sum_{i=1}^{C} (e_i + 1)}\geq 0$ and $ u=\frac{C}{\sum_{i=1}^{C} (e_i + 1)}\geq 0$,
where $e_i$ is the evidence of the $i$-th singleton. There is $u + \sum_{i=1}^{C}b_i = 1$, \textit{i.e.}, a weaker belief over singletons indicates a higher overall uncertainty. 

\textbf{Posterior Dirichlet  distribution}~SL further formalizes the belief assignment stated above as a Dirichlet distribution, offering the potential approach to modeling predictive probability. Specifically, let $Dir(\bm{p}|\bm{\alpha})$ denote a Dirichlet distribution, where $\bm{p}\in \mathcal{S}^{C-1}$ (a probability simplex with $C-1$ dimensions), $\bm{\alpha}=[\alpha_1,\cdots,\alpha_C]$, and $\alpha_i\geq 0$. 
By definition, there are $Dir(\bm{p}|\bm{\alpha})=\frac{1}{B(\bm{\alpha})}\prod_{i=1}^{C}p_i^{\alpha_i-1}$ and $B(\bm{\alpha})=\frac{\prod_{i=1}^{C}\Gamma(\alpha_i)}{\Gamma(\alpha_0)}$, where $B(\bm{\alpha})$ is a multinomial Beta function, $\Gamma(\cdot)$ is a \textit{gamma} function, and $\alpha_0=\sum_{i=1}^C \alpha_i$ often called the precision or Dirichlet strength. 
Therefore, predictive probability can be expressed by a posterior Dirichlet distribution $Dir(\bm{p}|\bm{\alpha})$, by deriving evidences $\bm{e}$ for $C$ singletons and then adopting $\bm{\alpha} = \bm{e} + 1$ to parameterize $Dir(\bm{p}|\bm{\alpha})$, where $\bm{e}=[e_1,\cdots,e_C]$.

\textbf{Evidential learning}~For any sample $X$, EDL employs a NN-based transformation $\Phi$ to derive evidences, \textit{i.e.}, $\bm{e}=\Phi(X)$, forming a belief assignment over $C$ classes. 
This assignment is then formalized as $Dir(\bm{p}|\bm{\alpha})$ with $\bm{\alpha}=\Phi(X)+1$, to obtain the prediction of $X$.  
Compared with the standard neural classifiers that predict a Categorical distribution over classes, EDL predicts \textit{a distribution over Categorical distribution}, the conjugate prior of Categorical distribution, thus modeling second-order probability distribution for UE \cite{josang2016subjective,sensoy2018evidential}. 

By using NNs to parameterize $Dir(\bm{p}|\bm{\alpha})$, EDL provides a simple yet efficient \textit{deterministic} method for UE. It can distinguish different uncertainties originated from data, model, or distribution \cite{ulmer2023prior}. Moreover, \citeauthor{pmlr-v202-deng23b} \citeyearpar{pmlr-v202-deng23b} show that EDL can be cast as learning PAC-Bayesian generalization bounds. 
These favorable advantages motivate us to study a baseline approach for MIUE through EDL. 

\section{Problem Formulation: Multi-Instance Uncertainty Estimation (MIUE)}

First of all, we formalize MIUE from the perspective of Bayesian  \cite{neal2012bayesian}, as it offers a principled way to study predictive uncertainty and is widely adopted in UE. 

\textbf{Bag-level}~Given a \textit{bag dataset} $\mathcal{D}=\{(X_j, Y_j)\}_{j=1}^{N}$ and the parameter $\bm{\omega}$ of any MIL model, the prediction of a bag $X^{*}$ can be written as a posterior distribution:
\begin{equation}
P(Y^{*}|X^{*},\mathcal{D})=\int P(Y^{*}|X^{*},\bm{\omega}) p(\bm{\omega}|\mathcal{D}) d \bm{\omega}.
\label{eq5}
\end{equation}
This predictive form reflects that the uncertainty in bag prediction results from data (aleatoric) and model (epistemic) uncertainty \cite{der2009aleatory,kendall2017uncertainties}, captured by $P(Y^{*}|X^{*},\bm{\omega})$ and $p(\bm{\omega}|\mathcal{D})$, respectively. $\bm{\omega}$ parameterizes the mapping from $X$ to $Y$. 

\textbf{Instance-level}~For any $(X_j, Y_j)\in \mathcal{D}$, $X_j=\{\mathbf{x}_{jk}\}_{k=1}^{K_j}$,  where $K_j$ is the instance number of $X_j$. Following the assumption of MIL, there is $Y_j=\mathop{\max}\{y_{jk}\}_{k=1}^{K_j}$ and $y_{jk}$ is \textit{unknown}. 
The prediction of an instance $\mathbf{x}^{*}$ is expressed as  
\begin{equation}
P(y^{*}|\mathbf{x}^{*},\mathcal{D})=\int P(y^{*}|\mathbf{x}^{*},\bm{\theta_{\mathrm{w}}}) p(\bm{\theta_{\mathrm{w}}}|\mathcal{D}) d \bm{\theta_{\mathrm{w}}},
\label{eq6}
\end{equation}
where $\bm{\theta_{\mathrm{w}}}$ is the parameter estimated from \textit{weakly-annotated} instances to represent the mapping from $\mathbf{x}$ to $y$. $\bm{\theta_{\mathrm{w}}}$ and  $\bm{\omega}$ could share some parameters in their joint learning on $\mathcal{D}$. 

\textbf{Uncertainty measures}~To quantify the uncertainty in prediction, various measures could be adopted \cite{malinin2018predictive}. \textit{Max probability} and the \textit{entropy} of expected predictive distribution are the most frequently used two for measuring total uncertainty. Moreover, \textit{expected entropy} and \textit{mutual information} (MI) are often adopted as the measures to capture data uncertainty and model uncertainty, respectively. More details are elaborated in Appendix \ref{apx-unc-measure}. 

Traditional non-Bayesian MIL networks usually treat $\bm{\omega}$ and $\bm{\theta_{\mathrm{w}}}$ deterministically, ignoring the model uncertainty in bag and instance predictions. This could lead to failures on out-of-distribution samples \cite{blanchard2011generalizing}. By contrast, Bayesian frameworks account for both data and model uncertainties, allowing us to quantify multi-instance uncertainty more accurately and holistically. 

\section{Method}

\subsection{Bag-level Predictive Uncertainty}

Bag-level predictive uncertainty can be quantified via Eq.(\ref{eq5}). However, integrating over a high-dimensional space of $\bm{\omega}$ is often \textit{intractable}. To tackle this, we propose to model bag-level predictive probability with a posterior Dirichlet distribution, inspired by EDL. 

\textbf{Evidential learning}~Concretely, for a bag $X=\{\mathbf{x}_k\}_{k=1}^{K}$, we express its predictive probability as follows:
\begin{equation}
\begin{aligned}
p(\bm{\mu}|X) = &\ Dir(\bm{\mu}|\bm{\alpha}_{\text{bag}}),\\ 
\bm{\alpha}_{\text{bag}} = \bm{e}_{\text{bag}} + 1 = &\ \mathcal{A}(S(X)) + 1\\
= &\ \mathcal{A}\Big( g_{\phi}\big(\sum_{k} f_{\psi}(\mathbf{x}_k)\big)\Big) + 1,
\end{aligned}
\label{eq7}
\end{equation}
where $\bm{\alpha}_{\text{bag}}$ is the concentration parameter of bag-level Dirichlet distribution, $\bm{e}_{\text{bag}}$ is the bag evidence collected from $X$, $\mathcal{A}(\cdot)$ is an activation function for non-negative evidence outputs, and the transformations $f$ and $g$ are parameterized by the NNs with parameters $\psi$ and $\phi$, respectively. 
To optimize parameters, we adopt a Fisher Information-based objective function \cite{pmlr-v202-deng23b}, proven to be well-suited for EDL. It is denoted by $\mathcal{L}_{\mathcal{I}\text{-EDL}}$ as follows:
\begin{equation}
\mathop{\min}_{\psi,\phi} \mathbb{E}_{(X,Y)\sim \mathcal{P}} \mathbb{E}_{\bm{\mu}\sim Dir(\bm{\alpha}_{\text{bag}})}\big[-\mathop{\log}p(Y|\bm{\mu},\bm{\alpha}_{\text{bag}},\sigma^2)\big],
\label{eq-loss-bag}
\end{equation}
where $p(Y|\bm{\mu},\bm{\alpha}_{\text{bag}},\sigma^2)$ is assumed to be a multivariate Gaussian distribution $\mathcal{N}\big(Y|\bm{\mu},\sigma^2\mathcal{I}(\bm{\alpha}_{\text{bag}})^{-1}\big)$ and $\mathcal{I}(\bm{\alpha}_{\text{bag}})$ is the Fisher Information Matrix of $Dir(\bm{\alpha}_{\text{bag}})$. We give its details in Appendix \ref{apx-loss-func} for completeness.

\textbf{Justification}~\textbf{(1) $p(\bm{\mu}|X)$}: the predictive form of bag $X$ given in Eq.(\ref{eq7}), \textit{i.e.}, a posterior Dirichlet distribution rather than a conventional Categorical one, can be derived from Eq.(\ref{eq5}) approximately \cite{malinin2018predictive}, \textit{i.e.},
\begin{equation}
\begin{aligned}
\int P(Y|X,\bm{\omega}) p(\bm{\omega}|\mathcal{D}) d \bm{\omega} & = \int P(Y|\bm{\mu}) p(\bm{\mu}|X, \mathcal{D}) d \bm{\mu}\\
& \approx \int P(Y|\bm{\mu}) p(\bm{\mu}|X; \hat{\bm{\omega}}) d \bm{\mu},
\end{aligned}
\label{eq8}
\end{equation}
with a point estimation $\hat{\bm{\omega}}$ satisfying $p(\bm{\omega}|\mathcal{D})=\delta(\bm{\omega}-\hat{\bm{\omega}})$ where $\delta(\cdot)$ is a Dirac delta function. Eq.(\ref{eq8}) makes the full posterior in Eq.(\ref{eq5}) tractable by introducing a new posterior Dirichlet distribution, thus enabling us to quantify the uncertainty in $p(\bm{\mu}|X)$ deterministically. More details are provided in Appendix \ref{apx-der-bayes}.
\textbf{(2) $\bm{\alpha}_{\text{bag}}$}: its evidential learning formulation in Eq.(\ref{eq7}) still satisfies the condition of Theorem \ref{thm:symfunc}. This thereby provides broad MIL networks with an alternative means to model bag-level uncertainty. 

\subsection{Rethinking Instance-level Estimator}
\label{sec-ins-est}
It is often of interest to ask MIL models to jointly estimate instance labels. For this purpose, herein, we first consider the estimator for instances before modeling instance-level predictive uncertainty.
Instead of scoring instances with attention mechanism as written in Eq.(\ref{eq2}), we rethink and derive a new estimator for instance prediction through the Fundamental Theorem of Symmetric Functions. 

Given the unified form of permutation-invariant bag classifiers (Theorem \ref{thm:symfunc}), \textit{i.e.}, $S(X) = g\big(\sum_{k} f(\mathbf{x}_k)\big)$, it could be found that a single aggregated embedding (by pooling over instance embeddings) is transformed into a bag score by $g$. Obviously, it is also possible to transform a single instance embedding for scoring, thereby obtaining a \textit{feasible} estimator for instances, as summarized below: 

\begin{corollary}
\label{corol:ins}
Given a scoring function for a set of instances $X=\{\mathbf{x}_1, \dots , \mathbf{x}_{K}\}$, written as $S(X) = g\big(\sum_{k} f(\mathbf{x}_k)\big)\in \mathbb{R}$ where $k\in[1,K]$, a scoring function for any single instance can be written as
\begin{equation}
T=g\circ f,
\label{eq9}
\end{equation}
and $T(\mathbf{x})\in \mathbb{R}$ for an instance $\mathbf{x}$.
\end{corollary}

Its proof is shown in Appendix \ref{apx-der-ins}. This corollary allows us to make instance prediction by \textit{skipping} permutation-invariant pooling. Nonetheless, it is still insufficient to conclude that $T$ is also a good instance estimator, apart from a feasible one. We thereby give Proposition \ref{prop:ins}. 

\begin{proposition}
\label{prop:ins}
Let $S(\cdot)$ be a classifier for a bag of instances $X=\{\mathbf{x}_k\}_{k=1}^K$ and satisfy $S(X) = g\big(\sum_{k} f(\mathbf{x}_k)\big)$. For any bag $X$ and its label $Y\in\{0,1\}$, further assume $S$ can predict bags precisely: $S(X)=Y$. Then, there exists an estimator with $T=g\circ f$ for any single instance $\mathbf{x}$, such that $T(\mathbf{x})=y$, where $y\in\{0,1\}$ is the label of $\mathbf{x}$.
\end{proposition}

This proposition implies that a perfect $T$ can be directly deduced from a perfect $S$. Refer to Appendix \ref{apx-der-ins} for its proof.
Furthermore, by relaxing the ideal assumption, \textit{i.e.}, a perfect $S$, it suggests that a good instance-level estimator is likely to be obtained from a good bag classifier. 
Intuitively, if $S$ is good at distinguishing between positive and negative bags, it would recognize an instance $\mathbf{x}$ correctly as long as one \textit{duplicates} $\mathbf{x}$ to form a new bag for prediction, because this bag is expected to be classified correctly by $S$ and it exactly has the \textit{same} label with $\mathbf{x}$. 
Extensive experiments (Section \ref{sec-exp}) empirically demonstrate this finding. 

Compared with the commonly-used attention score $a_k$, our new instance-level estimator $T(\mathbf{x})$ exhibits the following merits. (1) It could be obtained from not merely attention-based but general MIL approaches, as stated in Corollary \ref{corol:ins}. (2) It is no longer a scoring proxy like $a_k$ but an estimator with a classification decision space same to that of $S(X)$, probably closer to an ideal instance estimator than $a_k$. 

\subsection{Weakly-supervised Instance-level Predictive Uncertainty}
\label{ins-edl}

Given $T(\mathbf{x})=g_{\phi}(f_{\psi}(\mathbf{x}))$, $T$ could be adopted as the estimator to infer instance probabilities, with the same $f_{\psi}$ and $g_{\phi}$ as $S$. However, it could be not desirable for UE in practice, because the NN-based $g_{\phi}$ is actually trained with those \textit{pooled} instance embeddings (\textit{i.e.}, $\sum_{k} f_{\psi}(\mathbf{x}_k)$) instead of raw ones (\textit{i.e.}, $f_{\psi}(\mathbf{x}_k)$), thus inclining to estimate high uncertainties for \textit{unseen} data. To this end, we propose to improve $T$ by learning instance-specific residuals, thereby obtaining a new estimator specially for instance-level UE. Moreover, we propose a weakly-supervised evidence learning objective to optimize our residual instance estimator. 

\textbf{Residual evidential modeling}~For any single instance $\mathbf{x}$, we assume there is a \textit{residual estimation} for the given $T$, denoted as $\epsilon= T^{\ast}(\mathbf{x})-T(\mathbf{x})$ where $T^{\ast}(\mathbf{x})$ is a target estimation. We model this instance-specific residual $\epsilon$ by leveraging a new NN-parameterized transformation $r_{\pi}(\cdot)$ whose input is $\mathbf{h}=f_{\psi}(\mathbf{x})$. It is encouraged to compensate for the initial biased $T(\mathbf{x})$, so as to approach $T^{\ast}(\mathbf{x})$.
Let $R(\cdot)$ denote our new residual instance estimator. Further, we introduce a Dirichlet distribution into $R$ for instance-level UE. 
Therefore, the predictive distribution of instance $\mathbf{x}$ can be summarized and written as follows: 
\begin{equation}
\begin{aligned}
p(\bm{\nu}|\mathbf{x}) = &\ Dir(\bm{\nu}|\bm{\alpha}_{\text{ins}}),\\ 
\bm{\alpha}_{\text{ins}} = \bm{e}_{\text{ins}} + 1 = &\ \mathcal{A}\big(R(\mathbf{x})\big) + 1,\\
R(\mathbf{x})= T(\mathbf{x})+r_{\pi}(\mathbf{h})=&\ g_{\phi}(f_{\psi}(\mathbf{x})) + r_{\pi}(f_{\psi}(\mathbf{x})),
\end{aligned}
\label{eq10}
\end{equation}
where $\bm{\alpha}_{\text{ins}}$ is instance-level concentration parameter and  $\bm{e}_{\text{ins}}$ indicates the evidence derived from $\mathbf{x}$. 

\textbf{Optimization strategy}~Given any bag sample $(X, Y)$ where $X=\{\mathbf{x}_k\}_{k=1}^{K}$, our optimization strategy for $R(\mathbf{x})$ is as follows. (1) $Y=0$: we have $y_k=0\ \forall k\in[1,K]$,
so the objective function given by Eq.(\ref{eq-loss-bag}) can be utilized similarly. (2) $Y=1$: since only $\mathop{\max}_k \{y_k\}=1$ is given, we propose a weakly-supervised evidential learning strategy. Concretely, we simply multiply $\bm{e}_{k}$ (the evidence of $\mathbf{x}_k$) by different weights to mimic the selection of positive instances, and then aggregate them into a single one that can be optimized by Eq.(\ref{eq-loss-bag}). Those weights are the expected instance probabilities given by $T$, likely to indicate positive instances as stated in Section \ref{sec-ins-est}. Therefore, our objective function for $R(\mathbf{x})$ can be summarized and given as follows:
\begin{equation}
\begin{aligned}
& \mathcal{L}_{\text{MIREL}}=\mathop{\min}_{\psi,\pi} \mathbb{E}_{(X,Y)\sim \mathcal{P}} \big[Y\mathcal{L}_{\text{ins}}^{+} + (1-Y)\mathcal{L}_{\text{ins}}^{-}\big], \\
& \mathcal{L}_{\text{ins}}^{-} = \mathbb{E}_{\bm{\nu}\sim Dir(\bm{\alpha}_{k})}  \big[-\mathop{\log}p(y=0|\bm{\nu},\bm{\alpha}_{k},\sigma^2)\big], \\
& \mathcal{L}_{\text{ins}}^{+} =  \mathbb{E}_{\bm{\nu}\sim Dir(\widetilde{\bm{\alpha}})} \big[-\mathop{\log}p(Y=1|\bm{\nu},\widetilde{\bm{\alpha}},\sigma^2)\big],
\end{aligned}
\label{eq-loss-ins}
\end{equation}
where $\bm{\alpha}_{k}$ is the $\bm{\alpha}_{\text{ins}}$ for $\mathbf{x}_k$, $\widetilde{\bm{\alpha}}=\sum_{k}\bm{e}_{k} \bar{w}_{k} + 1$, $\bar{w}_{k}$ is a normalized $w_{k}$, $w_{k} = \mathbb{E}_{\bm{\nu}\sim Dir\big(\bm{\alpha}_{k}^{(T)}\big)} p(y=1|\bm{\nu})$, and $\bm{\alpha}_{k}^{(T)}=\mathcal{A}(T(\mathbf{x}_k))+1$. Note that only $\psi$ and $\pi$ are optimized in Eq.(\ref{eq-loss-ins}), without the $\phi$ for bags, as we aim at encouraging $r_{\pi}\circ f_{\psi}$ to learn residuals for instances. 

\textbf{Justification}~\textbf{(1) $R(\mathbf{x})$}: this new residual instance estimator has independent parameters $\pi$ with $S(X)$, as written in Eq.(\ref{eq10}). This enables $R$ to separately learn instance-specific evidences and enhance instance-level UE. 
\textbf{(2) $\mathcal{L}_{\text{ins}}^{+}$}:
we analyze its upper bounds, given in Proposition \ref{prop:insloss}.
This proposition suggests that $\mathcal{L}_{\text{ins}}^{+}$ may provide a more suitable $\bm{\hat{\theta}_{\mathrm{w}}}$ than common objectives such that there is $p(\bm{\theta_{\mathrm{w}}}|\mathcal{D})\approx \delta(\bm{\theta_{\mathrm{w}}}-\bm{\hat{\theta}_{\mathrm{w}}})$ for accurate UE. Refer to Appendix \ref{apx-der-insloss} for proofs and further explanations. 
Therefore, similar to that done in Eq.(\ref{eq8}), the intractable posterior in Eq.(\ref{eq6}) can be approximated by $Dir(\bm{\nu}|\bm{\alpha}_{\text{ins}};\bm{\hat{\theta}_{\mathrm{w}}})$. As a result, we can obtain a closed-form analytical solution (Appendix \ref{apx-um-dir}) for instance-level uncertainty quantification. 

\begin{proposition}
\label{prop:insloss}
Let $\mathcal{L}(\bm{\alpha},y)$ be a loss function \textit{w.r.t} $\bm{\alpha}$ and $y$. For any positive bag $X=\{\mathbf{x}_1,\cdots,\mathbf{x}_K\}$, assume $\bar{w}_k\geq 0\ \forall k\in[1,K]$, $\sum_{k}\bar{w}_k=1$, and $\widetilde{\bm{\alpha}}=\sum_{k} \bar{w}_{k} \bm{\alpha}_{k}$. 
$
\mathcal{L}_{\text{ins}}^{+}=\mathcal{L}(\widetilde{\bm{\alpha}},1)\leq \sum_{k} \bar{w}_k\mathcal{L}(\bm{\alpha}_k,1)\leq \sum_{k} \frac{1}{K}\mathcal{L}(\bm{\alpha}_k,1)
$ 
holds in instance evidential learning, when $\mathcal{L}$ is a convex function \textit{w.r.t} $\bm{\alpha}$ and  there is $\bar{w}_1\geq \bar{w}_2\geq \cdots \geq \bar{w}_K$ for $\mathcal{L}(\bm{\alpha}_1,1)\leq \mathcal{L}(\bm{\alpha}_2,1) \leq \cdots \leq \mathcal{L}(\bm{\alpha}_K,1)$.  
\end{proposition}

\red{
\textbf{Complete objective function}~To jointly train instance-level $R(\mathbf{x})$ and bag-level $S(X)$ for MIUE, the complete objective function we adopt is summarized as follows:
\begin{equation}
\mathcal{L}=\mathcal{L}_{\mathcal{I}\text{-EDL}} + \mathcal{L}_{\text{MIREL}} + \mathcal{L}_{\text{RED}},
\label{eq-complete-loss}
\end{equation}
where $\mathcal{L}_{\text{RED}}$ is a RED loss \cite{pandey2023learn} serving as a regularization term in EDL. As $\mathcal{L}_{\text{RED}}$ shows to be effective in avoiding zero-evidence regions to improve EDL, we exploit this loss to regularize both the evidence output of $R(\mathbf{x})$ and $S(X)$. For an evidence output $\bm{\alpha}$, $\mathcal{L_\text{RED}}=-\frac{C}{\alpha_0}\mathrm{log}(\alpha_{\text{gt}}-1)$, where $\alpha_{\text{gt}}$ is the predicted evidence for ground truth class. Refer to Appendix \ref{apx-mnist-abl-study} and \ref{apx-cifar-abl-study} for the ablation studies on Eq.(\ref{eq-complete-loss}). Note that in joint training, the parameters $\psi$ and $\phi$ are optimized by $\mathcal{L}_{\mathcal{I}\text{-EDL}}$ in bag-level evidential learning. Meanwhile, the parameters $\psi$ and $\pi$ are optimized by $\mathcal{L_\text{MIREL}}$ in instance-level residual evidential learning. Thus, $\psi$ is jointly optimized to improve both instance-level and bag-level UE performance.
}

\section{Related Work}

\textbf{Uncertainty estimation (UE)} methods could be grouped into two categories. (1) Bayesian NNs (\textbf{BNNs}) generally adopt stochastic NN weights to model predictive uncertainty \cite{neal2012bayesian}. Many techniques \cite{gal2016dropout,lakshminarayanan2017simple,loquercio2020general} are proposed to approximate the intractable posterior of BNNs. Nonetheless, they usually require sampling and are demanding computationally. (2) Deterministic uncertainty methods (\textbf{BUMs}) emerge as a promising means to mitigate this using deterministic NN weights. They can accomplish UE with a single forward pass mainly by regularizing hidden feature spaces \cite{alemi2018uncertainty,charpentier2020posterior} or employing distance-aware output layers \cite{liu2020simple,van2020uncertainty,mukhoti2023deep}. BUM is rapidly developing and encompasses many interesting works beyond those mentioned here. Readers could refer to \citeauthor{pmlr-v162-postels22a} \citeyearpar{pmlr-v162-postels22a} and \citeauthor{mukhoti2023deep} \citeyearpar{mukhoti2023deep}. 

\textbf{Dirichlet-based uncertainty (DBU)} methods \cite{ulmer2023prior} belong to another type of BUMs. They predict the parameters of Dirichlet distribution, allowing us to distinguish different uncertainty sources and write common uncertainty measures with closed-form solutions. As one of them, EDL \cite{sensoy2018evidential} have received great attention due to its simplicity and impressive performance. It has inspired many real-world applications \cite{bao2022opental,chen2022dual,chen2023cascade,zhao2023open}. In addition, EDL has been further improved recently in learning strategies \cite{pmlr-v202-deng23b,pandey2023learn} and concepts \cite{fan2023flexible}, showing better results in UE. All of these motivate us to study a baseline approach for MIUE through EDL. 

\textbf{MIL with instance-level estimator}~predicts not only bag results but also instance responses. It contains two main classes. (1) Instance-level methods \cite{liu2012key}. For model interpretability, they directly predict instance scores and then aggregate these scores into a bag-level result \cite{ilse2020deep}. (2) Embedding-level methods with explicit instance scoring branches. They work on instance embeddings rather than scores, and aim to enhance the accuracy of both instance and bag prediction using two separate branches \cite{chikontwe2020multiple,qu2022bi} or two branches decoupled via attention \cite{ilse2018attention,shi2020loss,li2021dual,yufei2022bayes}. Most of them are specially designed for pathology applications. 
\red{In addition, some other methods concentrate on alleviating the negative effect of noisy instances on MIL to improve the accuracy, mainly by maximizing the gap between two representative instances from a pair of negative and positive bags \cite{tian2021weakly,sapkota2022bayesian}. These methods are often seen in video anomaly detection tasks.} This paper roughly follows the second class; yet we do not focus on classification accuracy but UE performance, and devise a new residual instance estimator independent of attention-based MIL. 

\begin{table*}[htbp]
\caption{Main results on \textbf{MNIST-bags}. \textbf{OOD-F} and \textbf{OOD-K} mean that FMNIST and KMNIST are used for generating OOD bags, respectively. The results colored in gray are from our derived instance estimator $T(\mathbf{x})$. \textbf{$\overline{\textit{UE}}$} is the average metrics on three UE tasks.
}
\label{table-mnist}
\vskip 0.1in
\begin{center}
\begin{scriptsize}
\tabcolsep=0.178cm
\begin{tabular}{l|ccccc|ccccc}
\toprule
\multirow{2}{*}{\textbf{Method}} & \multicolumn{5}{c|}{\underline{\textbf{\ \ Bag-level\ \ }}}            & \multicolumn{5}{c}{\underline{\textbf{\ \ Instance-level\ \ }}}       \\
 & \textbf{Acc.} & \textbf{Conf.} & \textbf{OOD-F} & \textbf{OOD-K} & \textbf{$\overline{\textit{UE}}$} & \textbf{Acc.} & \textbf{Conf.} & \textbf{OOD-F} & \textbf{OOD-K} & \textbf{$\overline{\textit{UE}}$} \\
\midrule
\multicolumn{11}{l}{- \textit{Combined with deep MIL networks}} \\ \midrule
Mean & 93.38{\tiny\ $\pm$ 0.90} & \textbf{87.02}{\tiny\ $\pm$ 1.04} & \textbf{77.57}{\tiny\ $\pm$ 2.46} & 54.66{\tiny\ $\pm$ 2.62} & 73.08 & \textcolor[RGB]{128,128,128}{86.52{\tiny\ $\pm$ 0.97}} & \textcolor[RGB]{128,128,128}{66.49{\tiny\ $\pm$ 1.37}} & \textcolor[RGB]{128,128,128}{\textbf{79.36}{\tiny\ $\pm$ 1.95}} & \textcolor[RGB]{128,128,128}{\textbf{57.43}{\tiny\ $\pm$ 1.50}} & 67.76     \\
\red{Mean + }MIREL & 93.50{\tiny\ $\pm$ 0.53} & 87.01{\tiny\ $\pm$ 1.04} & 75.26{\tiny\ $\pm$ 1.52} & \textbf{57.69}{\tiny\ $\pm$ 6.28} & \textbf{73.32} & 92.45{\tiny\ $\pm$ 1.22} & \textbf{91.49}{\tiny\ $\pm$ 1.76} & 69.98{\tiny\ $\pm$ 4.41} & 56.70{\tiny\ $\pm$ 4.97} & \textbf{72.72}    \\ \midrule
Max & 94.56{\tiny\ $\pm$ 0.46} & 87.82{\tiny\ $\pm$ 1.49} & 75.23{\tiny\ $\pm$ 1.32} & 62.44{\tiny\ $\pm$ 3.00} & 75.17 & \textcolor[RGB]{128,128,128}{92.53{\tiny\ $\pm$ 0.54}} & \textcolor[RGB]{128,128,128}{81.86{\tiny\ $\pm$ 1.54}} & \textcolor[RGB]{128,128,128}{76.97{\tiny\ $\pm$ 1.71}} & \textcolor[RGB]{128,128,128}{\textbf{62.53}{\tiny\ $\pm$ 1.61}} & 73.79      \\
\red{Max + }MIREL & 95.96{\tiny\ $\pm$ 0.29} & \textbf{87.85}{\tiny\ $\pm$ 2.23} & \textbf{84.17}{\tiny\ $\pm$ 3.32} & \textbf{66.75}{\tiny\ $\pm$ 5.70} & \textbf{79.59} & 96.82{\tiny\ $\pm$ 0.27} & \textbf{84.22}{\tiny\ $\pm$ 0.43} & \textbf{80.81}{\tiny\ $\pm$ 4.88} & 61.15{\tiny\ $\pm$ 3.28} & \textbf{75.40}   \\ \midrule
DSMIL & 96.22{\tiny\ $\pm$ 0.17} & \textbf{87.56}{\tiny\ $\pm$ 0.95} & 71.13{\tiny\ $\pm$ 5.20} & 60.71{\tiny\ $\pm$ 7.91} & 73.13 &  70.16{\tiny\ $\pm$ 3.56} & 64.64{\tiny\ $\pm$ 0.49} & 59.75{\tiny\ $\pm$ 2.35} & 57.50{\tiny\ $\pm$ 2.55} & 60.63     \\
\red{DSMIL + }MIREL  & 96.50{\tiny\ $\pm$ 0.37} & 87.26{\tiny\ $\pm$ 2.66} & \textbf{87.27}{\tiny\ $\pm$ 4.27} & \textbf{62.03}{\tiny\ $\pm$ 7.78} & \textbf{78.85} & 97.19{\tiny\ $\pm$ 0.29} & \textbf{73.79}{\tiny\ $\pm$ 15.68} & \textbf{73.29}{\tiny\ $\pm$ 10.85} & \textbf{57.58}{\tiny\ $\pm$ 3.44}  & \textbf{68.22} \\ \midrule
ABMIL & 95.74{\tiny\ $\pm$ 0.38} & \textbf{86.91}{\tiny\ $\pm$ 0.98} & 82.93{\tiny\ $\pm$ 4.81} & 74.37{\tiny\ $\pm$ 4.84} & 81.41 &  75.03{\tiny\ $\pm$ 0.28} & 61.28{\tiny\ $\pm$ 0.86} & 63.68{\tiny\ $\pm$ 1.00} & 52.63{\tiny\ $\pm$ 1.07} & 59.20       \\
\red{ABMIL + }MIREL  & 96.48{\tiny\ $\pm$ 0.22} & 86.63{\tiny\ $\pm$ 1.32} & \textbf{92.84}{\tiny\ $\pm$ 0.60} & \textbf{79.95}{\tiny\ $\pm$ 4.12} & \textbf{86.47} & 87.71{\tiny\ $\pm$ 0.67} & \textbf{90.73}{\tiny\ $\pm$ 1.31} & \textbf{78.13}{\tiny\ $\pm$ 2.19} & \textbf{67.02}{\tiny\ $\pm$ 1.94} & \textbf{78.63}   \\  \midrule 
\multicolumn{11}{l}{- \textit{Compared with related UE methods \red{using ABMIL as the base MIL network}}} \\ \midrule
Deep Ensemble & 96.06{\tiny\ $\pm$ 0.35} & 87.36{\tiny\ $\pm$ 0.59} & 80.07{\tiny\ $\pm$ 2.57} & 74.33{\tiny\ $\pm$ 3.97} & 80.59 & 75.56{\tiny\ $\pm$ 0.32} & 71.89{\tiny\ $\pm$ 0.91} & 70.48{\tiny\ $\pm$ 0.53} & 55.22{\tiny\ $\pm$ 1.16} & 65.87  \\
MC Dropout  & 96.28{\tiny\ $\pm$ 0.41} & \textbf{88.46}{\tiny\ $\pm$ 1.82} & 89.57{\tiny\ $\pm$ 3.84} & 78.24{\tiny\ $\pm$ 4.89} & 85.42 &  75.61{\tiny\ $\pm$ 0.66} & 68.40{\tiny\ $\pm$ 1.54} & 68.34{\tiny\ $\pm$ 1.06} & 58.61{\tiny\ $\pm$ 1.38} & 65.12  \\
$\mathcal{I}$-EDL & 96.08{\tiny\ $\pm$ 0.20} & 86.78{\tiny\ $\pm$ 0.87} & 85.51{\tiny\ $\pm$ 7.56} & 73.15{\tiny\ $\pm$ 3.87} & 81.82 &  75.45{\tiny\ $\pm$ 0.13} & 60.72{\tiny\ $\pm$ 1.46} & 63.91{\tiny\ $\pm$ 1.31} & 54.14{\tiny\ $\pm$ 2.19} & 59.59 \\
Bayes-MIL & 96.44{\tiny\ $\pm$ 0.33} & 85.63{\tiny\ $\pm$ 1.53} & 81.02{\tiny\ $\pm$ 11.71} & 57.04{\tiny\ $\pm$ 12.61} & 74.57 & 91.64{\tiny\ $\pm$ 1.25} & 82.24{\tiny\ $\pm$ 1.85} & 60.77{\tiny\ $\pm$ 6.59} & 42.06{\tiny\ $\pm$ 2.84} & 61.69   \\ 
\textbf{MIREL}  & 96.48{\tiny\ $\pm$ 0.22} & 86.63{\tiny\ $\pm$ 1.32} & \textbf{92.84}{\tiny\ $\pm$ 0.60} & \textbf{79.95}{\tiny\ $\pm$ 4.12} & \textbf{86.47} & 87.71{\tiny\ $\pm$ 0.67} & \textbf{90.73}{\tiny\ $\pm$ 1.31} & \textbf{78.13}{\tiny\ $\pm$ 2.19} & \textbf{67.02}{\tiny\ $\pm$ 1.94} & \textbf{78.63}  \\  
\bottomrule
\end{tabular}
\end{scriptsize}
\end{center}
\vskip -0.1in
\end{table*}

\textbf{Uncertainty estimation for MIL} is formally studied far less than that for standard single instance learning, largely due to the weak-supervision nature of MIL. A recent orthogonal work is based on BNNs \cite{schmidt2023probabilistic,yufei2022bayes}. It converts ABMIL networks into BNNs and model the uncertainty of attention scores by variational approximation. 
However, it focuses on improving classification performance, not specially for UE. Moreover, it also requires multiple forward passes like typical BNNs. By contrast, we propose a DBU method that quantifies the uncertainty in MIL with a single forward pass. In particular, our study focuses on UE and designs extensive experiments to assess the UE capability of MIL models. 

\section{Experiments}
\label{sec-exp}

\subsection{Experimental Setup}
\label{sec-exp-setup}

\textbf{Datasets}~(1) Two bag datasets
are \textbf{MNIST-bags} \cite{lecun1998mnist} and \textbf{CIFAR10-bags} \cite{krizhevsky2009learning}, following \citeauthor{ilse2018attention}~\citeyearpar{ilse2018attention} to generate bags for MIL. A bag is positive if it contains at least one instance with the class of interest. The class of interest is `9' for MNIST and `truck' for CIFAR10. FMNIST-bags \cite{xiao2017fashion}, KMNIST-bags \cite{clanuwat2018deep}, SVHN-bags \cite{yuval2011reading}, and Texture-bags \cite{cimpoi14describing} are taken as OOD (out-of-distribution) datasets. (2) One pathology dataset is \textbf{CAMELYON16} \cite{bejnordi2017diagnostic} for breast cancer metastasis detection. We synthesize its three distribution-shifted versions \cite{tellez2019quantifying} for detection and take TCGA-PRAD \cite{kandoth2013mutational} as OOD dataset. More details are provided in Appendix \ref{apx-gen-bag} and \ref{apx-dataset}. 

\textbf{Baselines}~(1) \textbf{Classical deep MIL networks}: Mean, Max, ABMIL \cite{ilse2018attention}, and DSMIL \cite{li2021dual}. They cover three popular MIL pooling operators, used to verify whether our MIREL could improve their UE performances. (2) \textbf{Related UE methods}. General ones, Deep Ensemble \cite{lakshminarayanan2017simple}, MC Dropout \cite{gal2016dropout}, and $\mathcal{I}$-EDL \cite{pmlr-v202-deng23b} are adopted. ABMIL is employed as the base network for them and our MIREL. Moreover, a BNN-based MIL method, Bayes-MIL \cite{yufei2022bayes}, is also compared. Refer to Appendix \ref{apx-net-imp-training} for implementation and training details. 

\textbf{Evaluation}~Both bag-level and instance-level uncertainty are quantified for evaluation. We mainly use two typical UE tasks, confidence evaluation (\textbf{Conf.}) and \textbf{OOD detection}, \textit{i.e.}, we assess whether a model could show more confidence (or less uncertainty) for those correctly-classified samples (\textit{vs.} misclassified ones) and those ID (in-distribution) samples (\textit{vs.} OOD ones). AUROC is their evaluation metrics. We calculate max probability as confidence measure by default. Following \citeauthor{pmlr-v202-deng23b} \citeyearpar{pmlr-v202-deng23b}, for EDL models, we adopt Max.$\alpha$ ($\mathop{\max}_{c}\alpha_c$) and $\alpha_0$ ($\sum_{c}\alpha_c$) as confidence measures in Conf. and OOD detection, respectively. Classification accuracy (\textbf{Acc.}) is listed \textit{only for reference}. Each model is run with 5 seeds, and we report the mean and standard deviation of evaluation metrics. 

\subsection{MNIST-bags}

\textbf{Main results}~are shown in Table \ref{table-mnist}. (1) \textbf{Comparing the deep MIL networks with and without our MIREL}, we have three main observations. 
(i) In instance-level Conf., our MIREL always helps existing MIL networks to perform better, with an average improvement of 16.49\%. 
(ii) In bag-level OOD detection, the MIL networks with our MIREL exceed their counterparts by 7.03\% on average, in 7 out of 8 comparisons. At instance level, they present an average improvement of 9.26\%, in 5 out of 8 comparisons. 
(iii) Overall, our MIREL could often enhance the performance of deep MIL networks by large margins in terms of UE, especially for Max, ABMIL, and DSMIL. 
(2) \textbf{Compared with related UE methods}, our MIREL always shows better UE performances except in bag-level Conf.; particularly, at instance level, our MIREL leads runner-up by 7.62\% $\sim$ 8.41\%. These comparative results suggest that our MIREL is an effective and preferable approach for MIUE. 

\begin{figure}[htbp]
\vskip 0.1in
\begin{center}
\centerline{\includegraphics[width=\columnwidth]{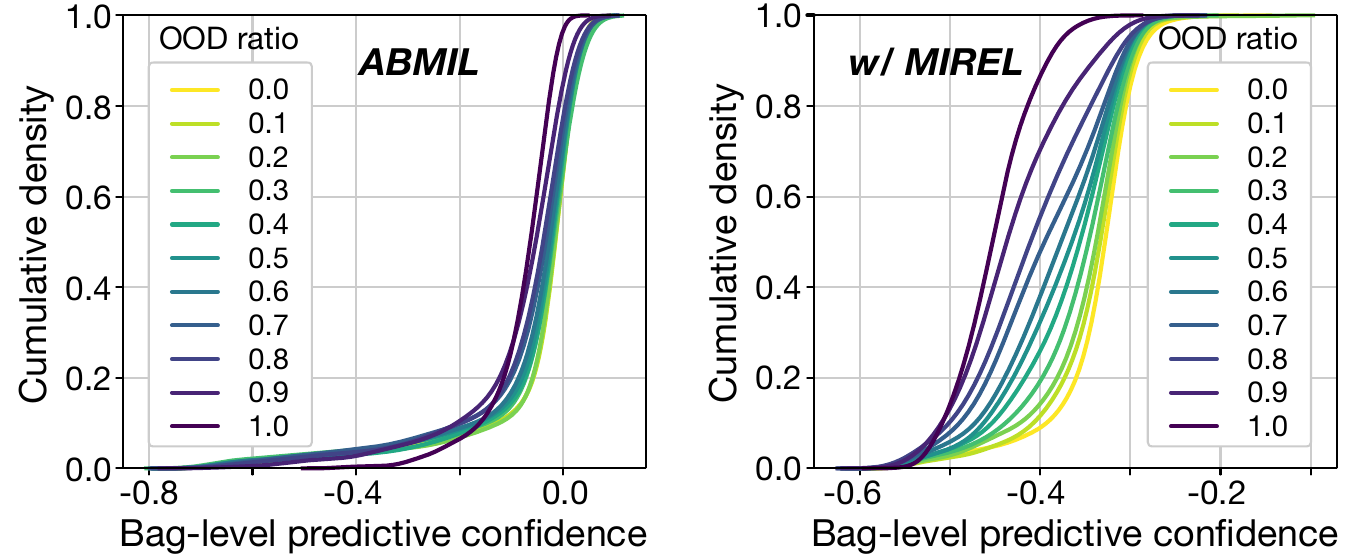}}
\caption{Distribution of bag-level predictive confidence (negative expected entropy). Different ratios of OOD (FMNIST) instances are set to assess the UE capability of MIL models.}
\label{mnist-abmil-bag-unc}
\end{center}
\vskip -0.1in
\end{figure}

\textbf{Uncertainty analysis}~is carried out to further examine the UE capability of our method. (1) \textbf{Bag-level}. We use OOD instances to randomly replace the instances from ID test bags, according to a specific target ratio of OOD instances. 
From the results shown in Fig. \ref{mnist-abmil-bag-unc}, we find that it is hard for ABMIL to identify the bags with different degrees of anomalies; but interestingly, when combined with our MIREL, ABMIL shows plausible reflections on those various abnormal bags. 
(2) \textbf{Instance-level}. We show results in Fig. \ref{mnist-abmil-ins-unc}. (i) From the result at top row, we observe that the ABMIL with our MIREL tends to predict clearly-higher uncertainty for numbers `4' and `7' while ABMIL does so only for `7'. The former result is in line with experiences since both `4' and `7' can be easily mistaken with `9' in hand-written numbers \cite{ilse2018attention}. 
(ii) From the result at bottom row, we see that the ABMIL w/ MIREL is more likely to predict lower confidences for OOD instances than ABMIL.
These results further confirm that our MIREL could assist classical MIL networks to capture uncertainty. More uncertainty analysis on KMNIST, $\alpha_0$, and DSMIL are given in Appendix \ref{apx-mnist-unc-analysis}.

\begin{figure}[htbp]
\vskip 0.1in
\begin{center}
\centerline{\includegraphics[width=\columnwidth]{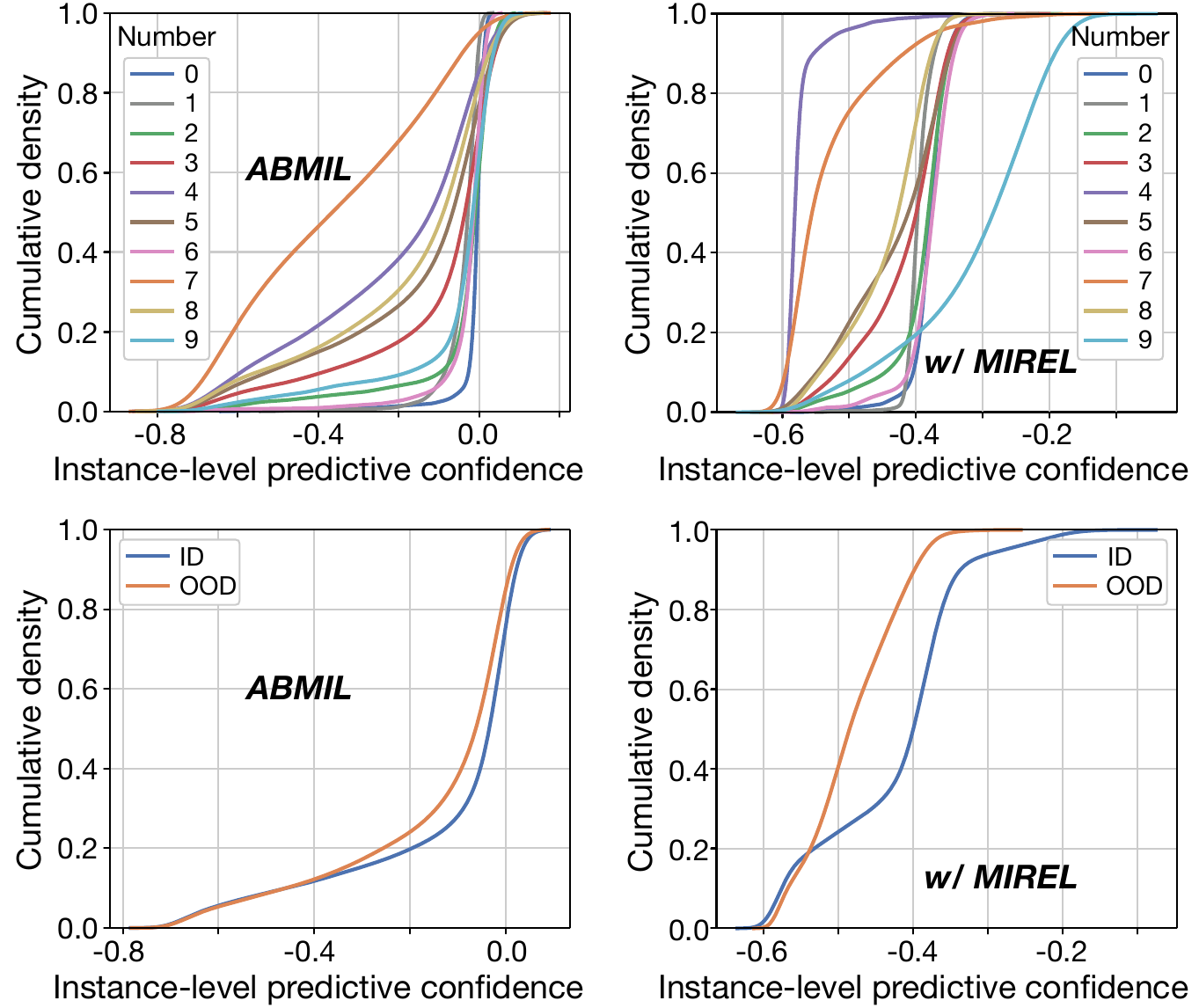}}
\caption{Distribution of instance-level predictive confidence (negative expected entropy). Top row is the result of MNIST instances, where `9' is the number of interest (positive instance). Bottom row is the result of ID (MNIST) and OOD (FMNIST) instances.}
\label{mnist-abmil-ins-unc}
\end{center}
\vskip -0.1in
\end{figure}

\begin{table*}[ht]
\caption{Ablation study on the ABMIL with our MIREL using \textbf{MNIST-bags}.}
\label{table-mnist-abl}
\vskip 0.1in
\begin{center}
\begin{scriptsize}
\tabcolsep=0.162cm
\begin{tabular}{ccc|ccccc|ccccc}
\toprule
\multicolumn{2}{c}{\underline{\textbf{\ \ Loss\ \ }}}   & \multirow{2}{*}{\textbf{Ins.}} & \multicolumn{5}{c|}{\underline{\textbf{\ \ Bag-level\ \ }}}            & \multicolumn{5}{c}{\underline{\textbf{\ \ Instance-level\ \ }}}       \\
 $\mathcal{L}_{\mathcal{I}\text{-EDL}}$ & $\mathcal{L}_{\text{MIREL}}$ & & \textbf{Acc.} & \textbf{Conf.} & \textbf{OOD-F} & \textbf{OOD-K} & \textbf{$\overline{\textit{UE}}$} & \textbf{Acc.} & \textbf{Conf.} & \textbf{OOD-F} & \textbf{OOD-K} & \textbf{$\overline{\textit{UE}}$} \\
\midrule
 & & $a_k$  & 95.74{\tiny\ $\pm$ 0.38} & \textbf{86.91}{\tiny\ $\pm$ 0.98} & 82.93{\tiny\ $\pm$ 4.81} & 74.37{\tiny\ $\pm$ 4.84} & 81.41 & 75.03{\tiny\ $\pm$ 0.28} & 61.28{\tiny\ $\pm$ 0.86} & 63.68{\tiny\ $\pm$ 1.00} & 52.63{\tiny\ $\pm$ 1.07} & 59.20  \\
 &  & $T$ & 95.74{\tiny\ $\pm$ 0.38} & \textbf{86.91}{\tiny\ $\pm$ 0.98} & 82.93{\tiny\ $\pm$ 4.81} & 74.37{\tiny\ $\pm$ 4.84} & 81.41 & 87.03{\tiny\ $\pm$ 1.42} & 84.30{\tiny\ $\pm$ 3.45} & 64.64{\tiny\ $\pm$ 4.92} & 52.67{\tiny\ $\pm$ 4.23} & 67.20  \\
 \checkmark &  & $a_k$ & 96.08{\tiny\ $\pm$ 0.20} & 86.78{\tiny\ $\pm$ 0.87} & 85.51{\tiny\ $\pm$ 7.56} & 73.15{\tiny\ $\pm$ 3.87} & 81.82 & 75.45{\tiny\ $\pm$ 0.13} & 60.72{\tiny\ $\pm$ 1.46} & 63.91{\tiny\ $\pm$ 1.31} & 54.14{\tiny\ $\pm$ 2.19} & 59.59 \\
 \checkmark & & $T$  & 96.08{\tiny\ $\pm$ 0.20} & 86.78{\tiny\ $\pm$ 0.87} & 85.51{\tiny\ $\pm$ 7.56} & 73.15{\tiny\ $\pm$ 3.87} & 81.82 & 85.19{\tiny\ $\pm$ 0.64} & 87.67{\tiny\ $\pm$ 1.11} & 73.52{\tiny\ $\pm$ 5.66} & 56.63{\tiny\ $\pm$ 1.66} & 72.60      \\
 \checkmark & \checkmark & $R$ & 96.48{\tiny\ $\pm$ 0.22} & 86.63{\tiny\ $\pm$ 1.32} & \textbf{92.84}{\tiny\ $\pm$ 0.60} & \textbf{79.95}{\tiny\ $\pm$ 4.12} & \textbf{86.47} & 87.71{\tiny\ $\pm$ 0.67} & \textbf{90.73}{\tiny\ $\pm$ 1.31} & \textbf{78.13}{\tiny\ $\pm$ 2.19} & \textbf{67.02}{\tiny\ $\pm$ 1.94} & \textbf{78.63} \\  
\bottomrule
\end{tabular}
\end{scriptsize}
\end{center}
\vskip -0.1in
\end{table*}

\begin{table*}[ht]
\caption{Comparison with related UE methods on \textbf{CAMELYON16}. \textbf{OOD-PRAD} means that TCGA-PRAD is taken as OOD data. The baseline of this experiment is \red{vanilla ABMIL without any additional UE techniques}. \textbf{$\overline{\textit{UE}}$} is the average metrics on two UE tasks. 
}
\label{table-c16}
\vskip 0.1in
\begin{center}
\begin{scriptsize}
\begin{tabular}{l|cccc|cccc}
\toprule
\multirow{2}{*}{\textbf{Method}} & \multicolumn{4}{c|}{\underline{\textbf{\ \ Bag-level\ \ }}}            & \multicolumn{4}{c}{\underline{\textbf{\ \ Instance-level\ \ }}}       \\
 & \textbf{Acc.} & \textbf{Conf.} & \textbf{OOD-PRAD} & \textbf{$\overline{\textit{UE}}$} & \textbf{Acc.} & \textbf{Conf.} & \textbf{OOD-PRAD} & \textbf{$\overline{\textit{UE}}$} \\
\midrule
Baseline & \textcolor[RGB]{128,128,128}{86.77{\tiny\ $\pm$ 0.77}} & \textcolor[RGB]{128,128,128}{\textbf{72.05}{\tiny\ $\pm$ 1.72}} & \textcolor[RGB]{128,128,128}{41.90{\tiny\ $\pm$ 2.91}} & \textcolor[RGB]{128,128,128}{56.98} & \textcolor[RGB]{128,128,128}{96.07{\tiny\ $\pm$ 0.01}} & \textcolor[RGB]{128,128,128}{49.87{\tiny\ $\pm$ 3.28}} & \textcolor[RGB]{128,128,128}{31.34{\tiny\ $\pm$ 0.85}} & \textcolor[RGB]{128,128,128}{40.60}   \\ \midrule
Deep Ensemble & 86.93{\tiny\ $\pm$ 0.63} & 70.44{\tiny\ $\pm$ 1.79} & 39.62{\tiny\ $\pm$ 3.32} & 55.03 & 96.08{\tiny\ $\pm$ 0.02} & 49.62{\tiny\ $\pm$ 2.53} & 28.16{\tiny\ $\pm$ 1.07} & 38.89  \\
MC Dropout  & 87.09{\tiny\ $\pm$ 1.37} & 67.50{\tiny\ $\pm$ 5.92} & 45.66{\tiny\ $\pm$ 5.48} & 56.58 & 96.05{\tiny\ $\pm$ 0.00} & 56.35{\tiny\ $\pm$ 2.20} & 33.93{\tiny\ $\pm$ 2.05} & 45.14  \\
$\mathcal{I}$-EDL & 87.72{\tiny\ $\pm$ 0.63} & 57.48{\tiny\ $\pm$ 8.07} & 72.43{\tiny\ $\pm$ 11.98} & 64.95 & 96.05{\tiny\ $\pm$ 0.01} & 45.41{\tiny\ $\pm$ 5.33} & 32.06{\tiny\ $\pm$ 1.49} & 38.74   \\  
Bayes-MIL & 86.61{\tiny\ $\pm$ 1.11} & 66.91{\tiny\ $\pm$ 6.73} & 48.77{\tiny\ $\pm$ 7.97} & 57.84 & 97.27{\tiny\ $\pm$ 0.28} & \textbf{82.12}{\tiny\ $\pm$ 9.12} & 54.53{\tiny\ $\pm$ 21.34}  & 68.33  \\ 
\textbf{MIREL}  & 87.09{\tiny\ $\pm$ 1.07} & 61.62{\tiny\ $\pm$ 6.85} & \textbf{82.51}{\tiny\ $\pm$ 8.34} & \textbf{72.06} & 97.79{\tiny\ $\pm$ 0.71} & 77.85{\tiny\ $\pm$ 5.69} & \textbf{67.85}{\tiny\ $\pm$ 5.71} & \textbf{72.85}   \\
\bottomrule
\end{tabular}
\end{scriptsize}
\end{center}
\vskip -0.1in
\end{table*}

\textbf{Ablation study}~is conducted on our MIREL to verify the effectiveness of its components. Different loss functions and instance estimators are adopted. More details are shown in Appendix \ref{apx-net-imp-training}.
From the results shown in Table \ref{table-mnist-abl}, there are three main findings. 
(1) \textbf{For our derived instance estimator $T$}, it leads attention-based scoring proxy ($a_k$) at instance level by large margins (8\% and 13.01\%) in terms of overall UE performance. 
(2) \textbf{For the adopted $\mathcal{L}_{\mathcal{I}\text{-EDL}}$}, it could help to enhance instance-level UE performances in the presence of $T$; no obvious effect is observed in other cases. 
(3) \textbf{For our new residual estimator $R$} trained with $\mathcal{L}_{\text{MIREL}}$, it boosts not only instance-level but also bag-level UE performances. Particularly, its performance improvements in OOD detection range from 4.61\% to 10.39\%. These findings could demonstrate the effectiveness of those components proposed in our MIREL. \textbf{More studies}, including i) the optimization strategies for $R(\mathbf{x})$ and ii) adopting $T(\mathbf{x})$ for related UE methods, are presented in Appendix \ref{apx-mnist-abl-study}. 

\textbf{More experiments}~\textbf{(1)} The results on \textbf{CIFAR10-bags}, including main results (\ref{apx-cifar-main}), uncertainty analysis (\ref{apx-cifar-unc-analysis}), and ablation studies (\ref{apx-cifar-abl-study}), are given in Appendix \ref{apx-cifar-res}. \textbf{(2) A synthetic MIUE experiment} is presented in Appendix \ref{apx-syn-exp-vis}, in order to intuitively understand the behavior of our weakly-supervised residual instance estimator $R(\mathbf{x})$. 

\begin{figure}[htbp]
\vskip 0.1in
\begin{center}
\centerline{\includegraphics[width=\columnwidth]{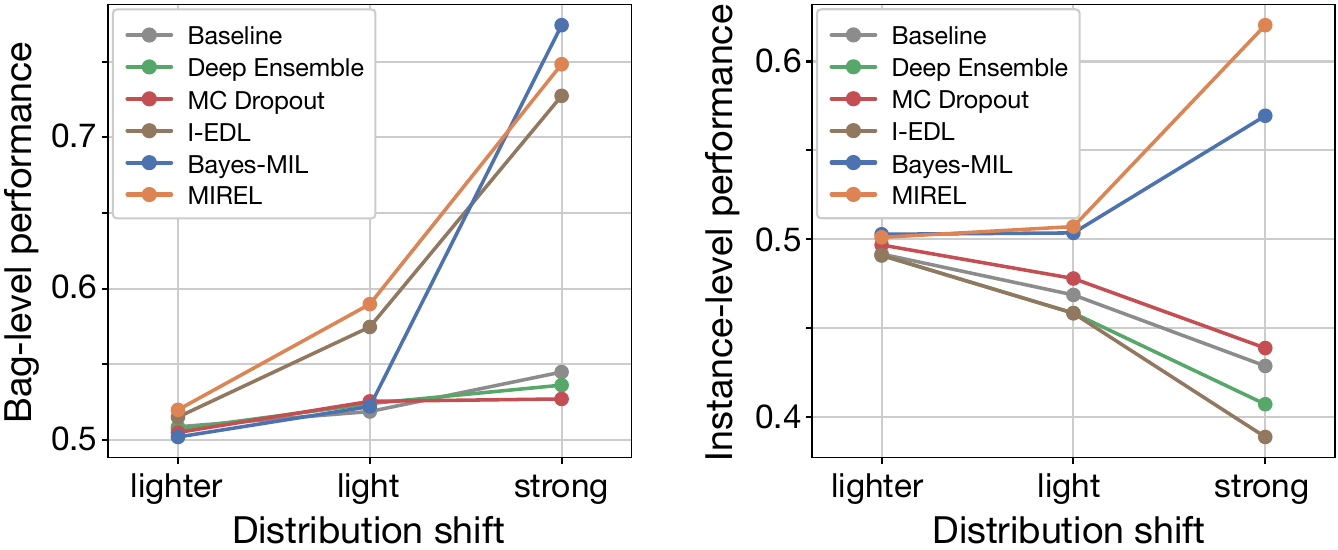}}
\caption{Performance of compared UE methods in distribution shift detection on \textbf{CAMELYON16}. The test samples of CAMELYON16 are shifted with different degrees of image noises for detection. Refer to Appendix \ref{apx-wsi-result-ds} for complete numerical results.}
\label{c16-dist-shift}
\end{center}
\vskip -0.1in
\end{figure}

\subsection{Histopathology Dataset}
\label{sec-c16-result}

\textbf{Main results} are exhibited in Table \ref{table-c16}. We mainly compare our method with related UE methods in this experiment.
There are four main observations from these results. 
(1) Our MIREL obtains the best overall UE performance at both bag and instance levels, surpassing the second best method by 7.11\% and 4.52\%, respectively. 
(2) Our MIREL improves the overall performance of ABMIL (baseline) in UE by considerable margins, 15.08\% and 32.35\% at bag and instance level, respectively. 
(3) Especially, our method shows impressive results in OOD detection, with a margin of more than 10\% over others at both instance and bag level. 
\red{(4) The UE methods not specially proposed for MIL often obtain the AUROC even less than 0.5 in OOD detection. It is because, the pathology images from TCGA-PRAD are near-OOD data and near-OOD detection is usually more challenging than far-OOD detection for these UE methods.}
These four observations suggest that our MIREL has the potential to be applied in real-world applications. 

\textbf{Distribution shift (DS) detection}~is a more challenging task than OOD detection, so it is further adopted to test our method. We generate three distribution-shifted test sets by adding routine noises to original CAMELYON16 test images, called \textit{lighter}, \textit{light}, and \textit{strong} according to noise strengths. Refer to Appendix \ref{apx-dataset} for experimental details.
Detection result is shown in Fig. \ref{c16-dist-shift}.
(1) \textbf{Bag-level}. Lighter DS is hard to detect for all the adopted UE methods (only about 0.52). Our MIREL can detect both light and strong DS with the best or the second best performance, while the competitive Bayes-MIL fails to detect light DS with an AUROC less than 0.53. 
(2) \textbf{instance-level}. All the compared methods \textit{not specially} for MIL, consistently show meaningless results (AUROC $<$ 0.5). By contrast, our MIREL detects strong DS with an AUROC of 0.62, better than the Bayes-MIL with an AUROC of 0.57. This experiment could further verify the superiority of our MIREL in MIUE. 

\section{\red{Limitation and }Future Work}

\red{Although our baseline approach MIREL shows promising results in MIUE, there are still some limitations in our experiments. 
First, since multi-instance bag is usually expensive in computation, our MIL datasets are limited in scale. Larger datasets (\textit{e.g.}, with more than 10,000 bags) would be better for more comprehensive validation. In addition, we take AUROC as the main metric to quantitatively evaluate the performance of UE methods. Additional calibration metrics, \textit{e.g.}, Expected Calibration Error (ECE) or Brier Score, would help to evaluate more holistically.}

In the future, there are some directions worth further research. 
(1) The optimization strategy for the weakly-supervised posterior $p(\bm{\theta_{\mathrm{w}}}|\mathcal{D})$. It could be further improved to provide a tighter upper bound for more accurate UE at instance level. (2) Seeking other efficient UE methods, such as distance-aware UE and feature space regularization, as stated in \citeauthor{pmlr-v162-postels22a}~\citeyearpar{pmlr-v162-postels22a}. (3) More general settings beyond binary MIL, \textit{e.g.}, multi-label MIL, as they cover more practical applications~\cite{zhou2012multi}.

\section{Conclusion}
This paper addresses a new MIUE problem and presents a baseline scheme, \textit{Multi-Instance Residual Evidential Learning}. 
In this scheme, we propose to model bag-level predictive uncertainty using a Dirichlet-based posterior distribution parameterized by general MIL networks. 
In particular, at weakly-supervised instance level, we derive a new residual estimator through the Fundamental Theorem of Symmetric Functions for instance-level UE. Moreover, without complete instance labels, we propose a weakly-supervised evidential optimization strategy for that residual estimator. 
Different UE tasks and extensive experiments demonstrate that our MIREL could often outperform other related UE methods. In addition, it can be applied to existing MIL networks, effectively assisting them in improving MIUE performances, especially at instance level. 
Since MIL has close connections with many real-world and safe-critical applications, it is of great importance and highly anticipated to enhance the reliability of MIL systems through MIUE. 
Our work could inspire more research on investigating MIUE, paving the way to explore uncertainty estimation in more weakly-supervised settings. 

\section*{Acknowledgements}

\red{This work was supported in part by the Fundamental Research Funds for the Central Universities under Grant ZYGX2022YGRH015 and in part by the National Natural Science Foundation of China (NSFC) under Grant 61972072. 
The authors would like to thank Jianghua Tian and Mao Dai for their encouragement and helpful feedback, and anonymous reviewers for their constructive comments to improve this work.}

\section*{Impact Statements}

This paper presents work whose goal is to advance the field of Machine Learning. There are many potential societal consequences of our work, none which we feel must be specifically highlighted here.


\bibliography{main}
\bibliographystyle{icml2024}

\newpage
\appendix
\onecolumn

\section{Derivation and Proof}

In this section, \ref{apx-der-bayes} gives the derivation of Eq.(\ref{eq8}) to clarify the connection of our bag-level predictive uncertainty modeling with Bayesian frameworks. \ref{apx-der-ins} presents the proof details of Corollary \ref{corol:ins} and Proposition \ref{prop:ins} (in Section \ref{sec-ins-est}). Lastly, in \ref{apx-der-insloss}, we justify our weakly-supervised optimization strategy (in Section \ref{ins-edl}) that aims to provide a more suitable $\bm{\hat{\theta}_{\mathrm{w}}}$. 

\subsection{Derivation of Eq.(\ref{eq8})}
\label{apx-der-bayes}
From the perspective of Bayesian \cite{neal2012bayesian}, for a new bag input $X^{*}$, its predictive distribution can be written by Eq.(\ref{eq5}): 
$$
P(Y^{*}|X^{*},\mathcal{D})=\int P(Y^{*}|X^{*},\bm{\omega}) p(\bm{\omega}|\mathcal{D}) d \bm{\omega}.
$$
By introducing a new distribution $p(\bm{\mu}|X^{*},\bm{\omega})$, we can re-write the equation above as follows:
\begin{equation}
\begin{aligned}
\int P(Y^{*}|X^{*},\bm{\omega}) p(\bm{\omega}|\mathcal{D}) d \bm{\omega} & =\int \int P(Y^{*}|\bm{\mu}) p(\bm{\mu}|X^{*},\bm{\omega}) p(\bm{\omega}|\mathcal{D}) d \bm{\mu} d \bm{\omega}\\
& = \int P(Y^{*}|\bm{\mu})\Bigl[\int p(\bm{\mu}|X^{*},\bm{\omega}) p(\bm{\omega}|\mathcal{D}) d \bm{\omega}\Bigr] d \bm{\mu}\\
& = \int P(Y^{*}|\bm{\mu}) p(\bm{\mu}|X^{*},\mathcal{D}) d \bm{\mu}\\
\end{aligned}
\end{equation}
As a result, $P(Y^{*}|\bm{\mu})$ can be taken as the new model. $p(\bm{\mu}|X^{*},\mathcal{D})$ is the distribution over model parameters conditioned the input bag $X$ and the given bag dataset $\mathcal{D}$. It is also referred to as the estimate of \textit{distributional uncertainty} given model uncertainty. 
However, the marginalization of the equation above is intractable. To tackle this, a point estimate of model parameters $\hat{\bm{\omega}}$ is often assumed to satisfy $p(\bm{\omega}|\mathcal{D})= \delta(\bm{\omega}-\hat{\bm{\omega}})$ \cite{malinin2018predictive}. Hence, 
$$
\int P(Y^{*}|X^{*},\bm{\omega}) p(\bm{\omega}|\mathcal{D}) d \bm{\omega} = \int P(Y^{*}|\bm{\mu}) p(\bm{\mu}|X^{*},\mathcal{D}) d \bm{\mu} \approx \int P(Y^{*}|\bm{\mu}) p(\bm{\mu}|X^{*}, \hat{\bm{\omega}}) d \bm{\mu}.
$$
$p(\bm{\mu}|X^{*}, \hat{\bm{\omega}})$ is exactly the posterior Dirichlet distribution given in our bag-level predictive uncertainty modeling. Accordingly, quantifying its uncertainty becomes tractable, as there is a closed-form analytical solution for those commonly-used uncertainty measures (refer to Appendix \ref{apx-unc-measure} for details). 

\subsection{Proof of Corollary \ref{corol:ins} and Proposition \ref{prop:ins}}
\label{apx-der-ins}

\textit{\textbf{Corollary} \ref{corol:ins}} (to Theorem \ref{thm:symfunc}). Given a scoring function for a set of instances $X=\{\mathbf{x}_1, \dots , \mathbf{x}_{K}\}$, written as $S(X) = g\big(\sum_{k} f(\mathbf{x}_k)\big)\in \mathbb{R}$ where $k\in[1,K]$, a scoring function for any single instance can be written as
$$
T=g\circ f,
$$
and $T(\mathbf{x})\in \mathbb{R}$ for an instance $\mathbf{x}$.
\begin{proof}
Without loss of generality, we assume $f(\mathbf{x})\in \mathbb{R}^M$. Given that the input of $g(\cdot)$ is $\sum_{k} f(\mathbf{x}_k)$ in $S(X)$, the domain of $g(\cdot)$ is in $\mathbb{R}^M$, since $\sum_{k} f(\mathbf{x}_k)\in \mathbb{R}^M.$
Therefore, $g(\cdot)$ can take $f(\mathbf{x})$ as input. Namely, there is a feasible function $T(\mathbf{x})=g(f(\mathbf{x}))\in \mathbb{R}$. Further, $S(X)$ is stated as a bag scoring function in Theorem \ref{thm:symfunc}. Hence, the $T(\mathbf{x})$, which has the same decision function $g(\cdot)$ as $S(X)$, can also be cast a scoring function specially for instances. 
\end{proof}

\textit{\textbf{Proposition} \ref{prop:ins}}. Let $S(\cdot)$ be a classifier for a bag of instances $X=\{\mathbf{x}_k\}_{k=1}^K$ and satisfy $S(X) = g\big(\sum_{k} f(\mathbf{x}_k)\big)$. For any bag $X$ and its label $Y\in\{0,1\}$, further assume $S$ can predict bags precisely: $S(X)=Y$. Then, there exists an estimator with $T=g\circ f$ for any single instance $\mathbf{x}$, such that $T(\mathbf{x})=y$, where $y\in\{0,1\}$ is the label of $\mathbf{x}$.
\begin{proof}
According to Corollary \ref{corol:ins}, given $S(X) = g\big(\sum_{k} f(\mathbf{x}_k)\big)$, there is an instance-level estimator $T$ that can be written as $T(\mathbf{x})=g(f(\mathbf{x}))$. 
With the existence of $T$, we need to prove $T(\mathbf{x})=y$, for any single instance $\mathbf{x}$ and its label $y$. 

Recall that, Theorem \ref{thm:symfunc}, which gives the bag scoring function $S(X)= g\big(\sum_{k} f(\mathbf{x}_k)\big)$, holds as before or under weak conditions, when the form of instance pooling, $\sum_{k} f(\mathbf{x}_{k})$, is replaced by others, such as i) $\mathrm{mean}$, ii) $\mathrm{max}$ \cite{qi2017pointnet}, and iii) attention-based MIL pooling \cite{ilse2018attention}. Next, we finish the proof with the basics of MIL. 

First of all, for any instance $\mathbf{x}$, we assume that there is a virtual bag as follows: $$X_{\text{vir}}=\{\underbrace{\mathbf{x},\cdots,\mathbf{x}}_{n}\}.$$ Based on the classical MIL assumption, \textit{i.e.}, $Y=\mathop{\max}_k \{y_k\}$, we have $$Y_{\text{vir}}=\mathop{\max}_{\mathbf{x}\in X_{\text{vir}}}\{y\}=y,$$
where $Y_{\text{vir}}$ is the ground-truth of $X_{\text{vir}}$. Then, we discuss three cases for the form of instance pooling in MIL:

\textbf{Case (1)}: $\mathrm{mean}$,
$$
S(X_{\text{vir}}) = g\Big(\frac{1}{n}\sum_{\mathbf{x}\in X_{\text{vir}}} f(\mathbf{x})\Big) = g(f(\mathbf{x})).
$$

\textbf{Case (2)}: $\mathrm{max}$ \cite{qi2017pointnet},
$$
S(X_{\text{vir}}) = g\Big(\mathop{\max}_{\mathbf{x}\in X_{\text{vir}}} \{f(\mathbf{x})\}\Big) = g(f(\mathbf{x})).
$$

\textbf{Case (3)}: attention \cite{ilse2018attention}. Combining with Eq.(\ref{eq2}), there is 
$$
S(X_{\text{vir}}) = g\Big(\sum_{\mathbf{x}\in X_{\text{vir}}} \frac{\mathop{\exp}\big(t(f(\mathbf{x}))\big)}{\sum_{\tau=1}^{n}\mathop{\exp}\big(t(f(\mathbf{x}))\big)} f(\mathbf{x})\Big) = g\Big(\sum_{\mathbf{x}\in X_{\text{vir}}} \frac{1}{n} f(\mathbf{x})\Big) =  g(f(\mathbf{x})).
$$

Since there is $S(X)=Y$, we have
$$\hat{Y}_{\text{vir}}=S(X_{\text{vir}})=Y_{\text{vir}},$$
where $\hat{Y}_{\text{vir}}$ is the prediction of $X_{\text{vir}}$. 
Eventually, 
\begin{equation}
\left.
\begin{aligned}
Y_{\text{vir}}=\mathop{\max}_{\mathbf{x}\in X_{\text{vir}}}\{y\} =  y\\
S(X_{\text{vir}}) =  g(f(\mathbf{x}))\\ 
\hat{Y}_{\text{vir}}=S(X_{\text{vir}})=Y_{\text{vir}}
\end{aligned}
\right\}\ \ \  \Longrightarrow\ \ \ \hat{y}=T(\mathbf{x}) =  g(f(\mathbf{x}))=y
\end{equation}

\end{proof}

\subsection{Justification for the Objective Function $\mathcal{L}_{\text{MIREL}}$ Given in Eq.(\ref{eq-loss-ins})}
\label{apx-der-insloss}

This section presents the details of our justification for $\mathcal{L}_{\text{MIREL}}$, including the proof of Proposition \ref{prop:insloss}.

Considering any bag $X\in \mathcal{D}$ and its label $Y\in\{0,1\}$, there are $X=\{\mathbf{x}_{1},\cdots, \mathbf{x}_{K}\}$ and $Y=\mathop{\max}\{y_{1},\cdots,y_{K}\}$ in MIL. To train an unbiased instance estimator $R(\mathbf{x})$ given instance labels, an ideal loss function could be written as follows:
\begin{equation}
\mathop{\min} \mathcal{L}_{R}=\mathop{\min} \sum\limits_{k=1}^{K}\frac{1}{K}\mathcal{L}(\bm{\alpha}_k,y_k),
\label{apx-eq-ideal}
\end{equation}
where $\mathcal{L}$ is a loss function derived from MLE for evidence learning, and $\bm{\alpha}_k$ is the concentration parameter of predictive Dirichlet distribution for the $k$-th instance. Note that here $\mathcal{L}$ is a loss function for $\bm{\alpha}_k$,
\textit{i.e.}, the prediction of the $k$-th instance, rather than a raw instance input.

When $Y=0$, we have $y_{k}=0\ \forall k=[1,K]$. In this case, $\mathcal{L}_{R}$ can be directly used for optimization given instance labels. When $Y=1$, we only know $\mathop{\max}_k\{y_k\}=1$, without complete instance labels. In this case, there are three alternative strategies for the training of $R(\mathbf{x})$.

\textbf{Strategy (1)}: naive instance label assignment. Assuming $y_{k}=1\ \forall k=[1,K]$, an alternative objective function is
$$\mathop{\min} \mathcal{L}_{1}=\mathop{\min} \sum\limits_{k=1}^{K}\frac{1}{K}\mathcal{L}(\bm{\alpha}_k,1).$$

\textbf{Strategy (2)}: naive instance label assignment and weighted loss. It is a weighted variant of strategy (1), written as follows:
$$\mathop{\min} \mathcal{L}_{2}=\mathop{\min} \sum\limits_{k=1}^{K}\frac{w_k}{\sum_{\tau=1}^{K} w_{\tau}}\ \mathcal{L}(\bm{\alpha}_k,1),$$
where $w_{k}$ is the probability of the $k$-th instance being positive. Here, $w_{k}$ could be estimated by $T(\mathbf{x}_k)=g(f(\mathbf{x}_k))$.

\textbf{Strategy (3)}: weighted instance evidence. $w_{k}$ is used to aggregate instance evidences in order to learn from key positive instances. This is the strategy we adopt, as written in Eq.(\ref{eq-loss-ins}). We simplify its objective function and re-write it by
$$\mathop{\min} \mathcal{L}_{\text{ins}}^{+}=\mathop{\min} \mathcal{L}\Big(\sum\limits_{k=1}^{K}\frac{w_k}{\sum_{\tau=1}^{K} w_{\tau}}\bm{\alpha}_k,1\Big).$$
Next, we prove that strategy (3) can provide a tighter upper bound for ideal Eq.(\ref{apx-eq-ideal}) than the other two under given conditions. 
Before that, we first give Proposition \ref{apx-prop-cmpl2} and Proposition \ref{apx-prop-cmpl1}, as well as their proof, as follows:

\begin{proposition}
\label{apx-prop-cmpl2}
$\mathcal{L}_{\text{ins}}^{+}\leq \mathcal{L}_2$ holds in instance evidential learning for a convex objective function $\mathcal{L}(\bm{\alpha},y=1)$.
\end{proposition}
\begin{proof}
By definition, $w_{k}\geq 0$. $\mathcal{L}(\bm{\alpha},y=1)$ is a given convex function \textit{w.r.t} $\bm{\alpha}$.
Therefore, by Jensen's inequality we have 
\begin{equation}
\mathcal{L}_{\text{ins}}^{+} = \mathcal{L}\Big(\sum\limits_{k=1}^{K}\frac{w_k}{\sum_{\tau=1}^{K} w_{\tau}}\bm{\alpha}_k,1\Big) \leq \sum\limits_{k=1}^{K}\frac{w_k}{\sum_{\tau=1}^{K} w_{\tau}}\mathcal{L}(\bm{\alpha}_k,1) = \mathcal{L}_2.
\label{apx-eq-cmp-l2}
\end{equation}
\end{proof}

\textbf{\ding{45}~Additional explanation:} The condition given in Proposition \ref{apx-prop-cmpl2}, a convex $\mathcal{L}$ \textit{w.r.t} model prediction, could be satisfied for common loss functions, \textit{e.g.}, the negative logarithm \textit{w.r.t} prediction, or the mean square error between prediction and target. 
Although the $\mathcal{L}(\bm{\alpha},y=1)$ used for our instance evidential learning (Appendix \ref{apx-loss-func}) contains non-convex terms and is not a strict convex function \textit{w.r.t} $\bm{\alpha}$, we still find that $\mathcal{L}_{\text{ins}}^{+}$ could be better than $\mathcal{L}_2$ in terms of overall UE performance, demonstrated by the ablation study on optimization strategy (refer to Appendix \ref{apx-mnist-abl-study} and Appendix \ref{apx-cifar-abl-study}).  
 
\begin{proposition}
\label{apx-prop-cmpl1}
$\mathcal{L}_{2}\leq \mathcal{L}_1$ holds in instance evidential learning when there is $w_1\geq w_2\geq \cdots \geq w_K$ for $\mathcal{L}(\bm{\alpha}_1,1)\leq \mathcal{L}(\bm{\alpha}_2,1) \leq \cdots \leq \mathcal{L}(\bm{\alpha}_K,1)$.
\end{proposition}
\begin{proof}
Let $$b_k=\frac{w_k}{\sum_{\tau=1}^{K} w_{\tau}}\ \forall k=1,\cdots,K.$$ 
Hence $\sum_{k=1}^{K}b_k=1$. Given $w_k\geq 0$ and $w_1\geq w_2\geq \cdots \geq w_K$, there is $$b_1\geq b_2\geq \cdots \geq b_K.$$ 

First of all, we prove the following inequality through proof by contradiction: $$\Delta_n = \sum_{k=1}^n b_k - \frac{n}{K} \geq 0.$$
Concretely, we assume $\Delta_n < 0$, \textit{i.e.}, $\sum_{k=1}^n b_k < \frac{n}{K}$. Given $b_1\geq b_2 \geq \cdots\geq b_n \geq b_q\ \forall q=n+1,\cdots, K$, we have
$$
\sum_{k=1}^n b_k < \frac{n}{K}\Longrightarrow \sum_{k=1}^n b_q\leq\sum_{k=1}^n b_k < \frac{n}{K}\Longrightarrow n b_q < \frac{n}{K} \Longrightarrow b_q < \frac{1}{K}.
$$
Further, by adding all $b_q$ into $\sum_{k=1}^n b_k$ and using $b_q<\frac{1}{K}$, there is 
$$
\sum_{k=1}^n b_k + \sum_{q=n+1}^K b_q < \frac{n}{K} + \sum_{q=n+1}^K \frac{1}{K}\Longrightarrow \sum_{k=1}^n b_k + \sum_{q=n+1}^K b_q < \frac{n}{K} + \frac{K-n}{K}\Longrightarrow\sum_{k=1}^n b_k + \sum_{q=n+1}^K b_q < 1.
$$
This contradicts $\sum_{k=1}^{K}b_k=1$. Namely, $\Delta_n = \sum_{k=1}^n b_k - \frac{n}{K} \geq 0$ holds.

Then, given $\mathcal{L}(\bm{\alpha}_1,1)\leq \mathcal{L}(\bm{\alpha}_2,1)$ and $\Delta_1=b_1-\frac{1}{K}\geq 0$, there is 
$$
\Big(b_1-\frac{1}{K}\Big) \mathcal{L}(\bm{\alpha}_1,1)\leq \Big(b_1-\frac{1}{K}\Big) \mathcal{L}(\bm{\alpha}_2,1)\ \ \Longrightarrow\ \  b_1\mathcal{L}(\bm{\alpha}_1,1) + \Big(\frac{2}{K}-b_1\Big)\mathcal{L}(\bm{\alpha}_2,1) \leq \frac{1}{K}\mathcal{L}(\bm{\alpha}_1,1) + \frac{1}{K}\mathcal{L}(\bm{\alpha}_2,1).
$$
Further, by introducing $\Delta_2\mathcal{L}(\bm{\alpha}_2,1)\leq \Delta_2\mathcal{L}(\bm{\alpha}_3,1)$ ($\Delta_2=b_1 + b_2 -\frac{2}{K}\geq 0$) into the inequality above, we have 
$$
b_1\mathcal{L}(\bm{\alpha}_1,1) + b_2\mathcal{L}(\bm{\alpha}_2,1) + \Big(\frac{3}{K}-b_1-b_2\Big)\mathcal{L}(\bm{\alpha}_3,1) \leq \frac{1}{K}\mathcal{L}(\bm{\alpha}_1,1) + \frac{1}{K}\mathcal{L}(\bm{\alpha}_2,1) + \frac{1}{K}\mathcal{L}(\bm{\alpha}_3,1).
$$
By analogy, we can obtain 
$$
\mathcal{L}_2 = b_1\mathcal{L}(\bm{\alpha}_1,1) + b_2\mathcal{L}(\bm{\alpha}_1,1) + \cdots + b_K\mathcal{L}(\bm{\alpha}_K,1) \leq
\frac{1}{K}\mathcal{L}(\bm{\alpha}_1,1) + \frac{1}{K}\mathcal{L}(\bm{\alpha}_2,1) + \cdots + \frac{1}{K}\mathcal{L}(\bm{\alpha}_K,1) = \mathcal{L}_1
$$
\end{proof}

\textbf{\ding{45}~Additional explanation:} The condition given in Proposition \ref{apx-prop-cmpl1}, $w_1\geq w_2\geq \cdots \geq w_K$ for $\mathcal{L}(\bm{\alpha}_1,1)\leq \mathcal{L}(\bm{\alpha}_2,1) \leq \cdots \leq \mathcal{L}(\bm{\alpha}_K,1)$, states that there is a higher weight for the instance whose prediction is closer to the expected evidence derived from positive instances. Here, we assume that the instance-level estimator $T=g\circ f$ could predict a higher $w_k$ for positive instances and a lower one for negative instances, in order to satisfy that condition approximately. 

With Proposition \ref{apx-prop-cmpl2} and Proposition \ref{apx-prop-cmpl1}, we give the proof of Proposition \ref{prop:insloss} as follows: 

\textit{\textbf{Proposition} \ref{prop:insloss}}. Let $\mathcal{L}(\bm{\alpha},y)$ be a loss function \textit{w.r.t} $\bm{\alpha}$ and $y$. For any positive bag $X=\{\mathbf{x}_1,\cdots,\mathbf{x}_K\}$, assume $\bar{w}_k\geq 0\ \forall k\in[1,K]$, $\sum_{k}\bar{w}_k=1$, and $\widetilde{\bm{\alpha}}=\sum_{k} \bar{w}_{k} \bm{\alpha}_{k}$. 
$
\mathcal{L}_{\text{ins}}^{+}=\mathcal{L}(\widetilde{\bm{\alpha}},1)\leq \sum_{k} \bar{w}_k\mathcal{L}(\bm{\alpha}_k,1)\leq \sum_{k} \frac{1}{K}\mathcal{L}(\bm{\alpha}_k,1)
$ 
holds in instance evidential learning, when $\mathcal{L}$ is a convex function \textit{w.r.t} $\bm{\alpha}$ and  there is $\bar{w}_1\geq \bar{w}_2\geq \cdots \geq \bar{w}_K$ for $\mathcal{L}(\bm{\alpha}_1,1)\leq \mathcal{L}(\bm{\alpha}_2,1) \leq \cdots \leq \mathcal{L}(\bm{\alpha}_K,1)$. 
\begin{proof}
Since $\sum_{k} \bar{w}_k\mathcal{L}(\bm{\alpha}_k,1)=\mathcal{L}_2$ and $\mathcal{L}$ is a convex function \textit{w.r.t} $\bm{\alpha}$, we have $\mathcal{L}_{\text{ins}}^{+}\leq\mathcal{L}_2$ according to Proposition \ref{apx-prop-cmpl2}. Further, $\sum_{k} \frac{1}{K}\mathcal{L}(\bm{\alpha}_k,1)=\mathcal{L}_1$ and those given conditions exactly satisfy the condition of Proposition \ref{apx-prop-cmpl1}, so $\mathcal{L}_2\leq\mathcal{L}_1$. Hence, there is $\mathcal{L}_{\text{ins}}^{+}\leq\mathcal{L}_2\leq\mathcal{L}_1$. 
Namely, $
\mathcal{L}_{\text{ins}}^{+}=\mathcal{L}(\widetilde{\bm{\alpha}},1)\leq \sum_{k} \bar{w}_k\mathcal{L}(\bm{\alpha}_k,1)\leq \sum_{k} \frac{1}{K}\mathcal{L}(\bm{\alpha}_k,1)
$ holds.
\end{proof}

Proposition \ref{prop:insloss} ensure that our objective function $\mathcal{L}_{\text{ins}}^{+}$ can provide a tighter upper bound than $\mathcal{L}_{1}$ and $\mathcal{L}_{2}$ for the ideal objective function Eq.(\ref{apx-eq-ideal}) under given conditions. This implies that our optimization strategy (3) could yield a more suitable weakly-supervised posterior $\bm{\hat{\theta}_{\mathrm{w}}}$ in instance evidential learning, such that $p(\bm{\theta_{\mathrm{w}}}|\mathcal{D})\approx \delta(\bm{\theta_{\mathrm{w}}}-\bm{\hat{\theta}_{\mathrm{w}}})$. Accordingly, we could approximate the intractable posterior in Eq.(\ref{eq6}) with $p(\bm{\nu}|\mathbf{x}^{*}, \mathcal{D})\approx p(\bm{\nu}|\mathbf{x}^{*}, \bm{\hat{\theta}_{\mathrm{w}}})$ for a new instance input $\mathbf{x}^{*}$, as follows: 
$$
P(y^{*}|\mathbf{x}^{*},\mathcal{D}) = \int P(y^{*}|\mathbf{x}^{*},\bm{\theta_{\mathrm{w}}}) p(\bm{\theta_{\mathrm{w}}}|\mathcal{D}) d \bm{\theta_{\mathrm{w}}} = \int P(y^{*}|\bm{\nu}) p(\bm{\nu}|\mathbf{x}^{*},\mathcal{D}) d \bm{\nu} \approx \int P(y^{*}|\bm{\nu}) p(\bm{\nu}|\mathbf{x}^{*}, \bm{\hat{\theta}_{\mathrm{w}}}) d \bm{\nu},
$$
where $\bm{\nu}$ is instance-level predictive probability. The above equation with Dirichlet distribution has a closed-form solution for instance-level uncertainty quantification, as elaborated in Appendix \ref{apx-um-dir}. The ablation studies on instance loss functions (refer to Appendix \ref{apx-mnist-abl-study} and Appendix \ref{apx-cifar-abl-study}) empirically demonstrate Proposition \ref{prop:insloss}. 

\section{Uncertainty Measures}
\label{apx-unc-measure}
This section mainly shows common measures for uncertainty quantification. Two predictive distributions, general Categorical distribution and related Dirichlet distribution, are considered here. This section is adapted from \citeauthor{malinin2018predictive} \citeyearpar{malinin2018predictive}, in order to provide readers with additional reference. 

\subsection{Measures for Predictive Categorical Distribution}
\label{apx-um-cate}
Given a predictive Categorical distribution for input $X$, $P(Y|X)$, its \textit{total uncertainty} could be quantified through two common measures as follows:

\textbf{(1) Max probability}~It is the probability of the predicted category, as a measure of confidence in prediction. Intuitively, a larger max probability means that model is more confident of its prediction. Max probability can be written as follows:
$$
\mathop{\max}_{i} P(Y=i|X),
$$
where $i\in[1,C]$ and $C$ is the total number of categories. 

\textbf{(2) Entropy}~By definition, it is calculated by
$$
\mathcal{H}[P(Y|X)] = - \sum\limits_{i=1}^{C} \big[P(Y=i|X) \mathop{\ln}P(Y=i|X)\big]
$$
A flat predictive distribution would yield a maximum $\mathcal{H}[P(Y|X)]$, implying high predictive uncertainty. 

Moreover, from the perspective of Bayesian \cite{neal2012bayesian}, consider the posterior distribution of model parameter given dataset $\mathcal{D}$, \textit{i.e.}, $p(\bm{\omega}|\mathcal{D})$. There is another measure commonly used for quantifying the \textit{model uncertainty} in prediction:
\begin{equation}
\underbrace{I[Y,\bm{\omega}|X,\mathcal{D}]}_{\text{Model Uncertainty}} = \underbrace{\mathcal{H}\big[\mathbb{E}_{p(\bm{\omega}|\mathcal{D})}P(Y|X,\bm{\omega})\big]}_{\text{Total Uncertainty}} - \underbrace{\mathbb{E}_{p(\bm{\omega}|\mathcal{D})}\big[\mathcal{H}[P(Y|X,\bm{\omega})]\big]}_{\text{Expected Data Uncertainty}}.
\label{apx-eq-mi-model}
\end{equation}

It is referred to as \textbf{Mutual information} (MI). As shown in the equation above, MI can be cast as the difference of \textit{total uncertainty} and \textit{expected data uncertainty}. The former is captured by $\mathcal{H}\big[\mathbb{E}_{p(\bm{\omega}|\mathcal{D})}P(Y|X,\bm{\omega})\big]$, \textit{i.e.}, the entropy of expected predictive distribution. The latter is captured by $\mathbb{E}_{p(\bm{\omega}|\mathcal{D})}\big[\mathcal{H}[P(Y|X,\bm{\omega})]\big]$, \textit{i.e.}, the expected entropy of predictive distribution. For traditional non-Bayesian NN models, MI is zero because their parameter is usually a point estimation.

\subsection{Measures for Predictive Dirichlet Distribution}
\label{apx-um-dir}
Consider a prediction with Dirichlet distribution (known as the conjugate prior of Categorical distribution):
$$
Dir(\bm{p}|\bm{\alpha})=\frac{\Gamma(\alpha_0)}{\prod_{i=1}^{C}\Gamma(\alpha_i)}\prod_{i=1}^{C}p_i^{\alpha_i-1},
$$
where $\bm{p}\in \mathcal{S}^{C-1}$ (a probability simplex with $C-1$ dimensions), $\bm{\alpha}=[\alpha_1,\cdots,\alpha_C]$, $\alpha_i\geq 0\ \forall i\in[1,C]$, $\Gamma(\cdot)$ is a \textit{gamma} function, and $\alpha_0=\sum_{i=1}^C \alpha_i$ often called the precision or Dirichlet strength. Its expected probability is as follows:
$$
\mathbb{E}[\bm{p}] = \Big[\frac{\alpha_1}{\alpha_0}, \frac{\alpha_2}{\alpha_0}, \cdots, \frac{\alpha_C}{\alpha_0}\Big].
$$
Next, we give common uncertainty measures for $Dir(\bm{p}|\bm{\alpha})$. All of them have a closed-form analytical solution. 

\textbf{Max probability and Entropy}~By simply taking $\mathbb{E}[\bm{p}]$ as prediction, they can be written as 
$$
\mathop{\max} \Big\{\frac{\alpha_i}{\alpha_0}\Big\}_{i=1}^{C} \text{\ \   and\ \  } \mathcal{H}\big[\mathbb{E}_{\bm{p}\sim Dir(\bm{\alpha})}P(Y|\bm{p})\big] = - \sum\limits_{i=1}^{C} \frac{\alpha_i}{\alpha_0} \mathop{\ln}\frac{\alpha_i}{\alpha_0},
$$
respectively, to capture the \textit{total uncertainty} in prediction. 

\textbf{Expected Entropy}~Different from the calculation of predictive entropy above, expected entropy is expressed as 
$$
\mathbb{E}_{\bm{p}\sim Dir(\bm{\alpha})}\big[\mathcal{H}[P(Y|\bm{p})]\big] = \int_{\mathcal{S}^{C-1}} Dir(\bm{p}|\bm{\alpha}) \Big(-\sum_{i=1}^{C}p_i\mathop{\ln}p_i\Big)  d \bm{p} = -\sum_{i=1}^{C} \frac{\alpha_i}{\alpha_0} \big(\psi(\alpha_i + 1) - \psi(\alpha_0+1) \big),
$$
where $\psi(\cdot)$ is a \textit{digamma} function defined as $\psi(x)=\frac{d}{dx}\mathop{\log}\Gamma(x)$. As mentioned in Eq.(\ref{apx-eq-mi-model}), expected entropy is usually utilized to measure \textit{data uncertainty}. Intuitively, it could capture the `peak' probabilities in $\mathbb{E}[\bm{p}]$. 

\textbf{Mutual Information}~By definition, mutual information (MI) can be written as follows:
$$
\underbrace{I[Y,\bm{p}|X,\mathcal{D}]}_{\text{Distributional Uncertainty}} = \underbrace{\mathcal{H}\big[\mathbb{E}_{\bm{p}\sim Dir(\bm{p}|X,\mathcal{D})}P(Y|\bm{p})\big]}_{\text{Total Uncertainty}} - \underbrace{\mathbb{E}_{\bm{p}\sim Dir(\bm{p}|X,\mathcal{D})}\big[\mathcal{H}[P(Y|\bm{p})]\big]}_{\text{Expected Data Uncertainty}}.
$$
This equation calculates the MI between $Y$ and the categorical $\bm{p}$, rather than the $\bm{\omega}$ written in Eq.(\ref{apx-eq-mi-model}). Thereby, $I[Y,\bm{p}|X,\mathcal{D}]$ is generally used to measure \textit{distributional uncertainty}. Assuming there is a sufficient point estimate $\hat{\bm{\omega}}$ satisfying $Dir(\bm{p}|X, \mathcal{D})\approx Dir(\bm{p}|X, \hat{\bm{\omega}})=Dir(\bm{p}|\bm{\alpha})$ where $\bm{\alpha}$ is the prediction of $X$ given the model parameter $\hat{\bm{\omega}}$, MI \cite{malinin2018predictive} could be approximated and calculated as follows:
\begin{equation}
\begin{aligned}
\underbrace{I[Y,\bm{p}|X,\mathcal{D}]}_{\text{Distributional Uncertainty}} & \approx \underbrace{\mathcal{H}\big[\mathbb{E}_{\bm{p}\sim Dir(\bm{\alpha})}P(Y|\bm{p})\big]}_{\text{Total Uncertainty}} - \underbrace{\mathbb{E}_{\bm{p}\sim Dir(\bm{\alpha})}\big[\mathcal{H}[P(Y|\bm{p})]\big]}_{\text{Expected Data Uncertainty}} \\
& = - \sum\limits_{i=1}^{C} \frac{\alpha_i}{\alpha_0} \mathop{\ln}\frac{\alpha_i}{\alpha_0} - \Bigg( -\sum_{i=1}^{C} \frac{\alpha_i}{\alpha_0} \big(\psi(\alpha_i + 1) - \psi(\alpha_0+1) \big) \Bigg)\\
& = -\sum_{i=1}^{C} \frac{\alpha_i}{\alpha_0} \Big(\mathop{\ln}\frac{\alpha_i}{\alpha_0} - \psi(\alpha_i + 1) + \psi(\alpha_0+1) \Big).
\end{aligned}
\label{apx-eq-mi-dist}
\end{equation}

\textbf{More measures}~involving concentration parameters $\bm{\alpha}=[\alpha_1, \cdots, \alpha_C]$, \textit{e.g.}, $\mathop{\max}_i \{\alpha_i\}$, $\alpha_0=\sum_{i=1}^{C} \alpha_i$, and $\frac{C}{\sum_{i=1}^{C} \alpha_i}$, are often utilized in Dirichlet-based uncertainty (DBU) models, to capture model uncertainty, as $\bm{\alpha}$ shapes the distribution of predictive probabilities. For more discussions about this, readers could refer to \citeauthor{ulmer2023prior} \citeyearpar{ulmer2023prior}. 

\section{Objective Function for Evidential Deep Learning}
\label{apx-loss-func}
For completeness, here we provide the details of Fisher Information-based objective function \cite{pmlr-v202-deng23b}. This section is adapted from \citeauthor{pmlr-v202-deng23b} \citeyearpar{pmlr-v202-deng23b}. 

The Fisher Information-based objective function employed by us, $\mathcal{L}_{\mathcal{I}\text{-EDL}}$, can be written as follows:
\begin{equation}
\begin{aligned}
\mathop{\min}\ \ & \mathbb{E}_{(X,Y)\sim \mathcal{P}} \mathbb{E}_{\bm{\mu}\sim Dir(\bm{\alpha})}\big[-\mathop{\log}p(Y|\bm{\mu},\bm{\alpha},\sigma^2)\big].\\
s.t.\ \ & \bm{\alpha}=\mathcal{F}(X)+1\\
& p(Y|\bm{\mu},\bm{\alpha},\sigma^2)=\mathcal{N}\big(Y|\bm{\mu},\sigma^2\mathcal{I}(\bm{\alpha})^{-1}\big)\\
& \mathcal{I}(\bm{\alpha})=\mathbb{E}_{Dir(\bm{\mu}|\bm{\alpha})} \Big[ - \frac{\partial^2\mathop{\log} Dir(\bm{\mu}|\bm{\alpha})}{\partial \bm{\alpha}\bm{\alpha}^{T}} \Big]
\end{aligned}
\end{equation}
In this function, $\bm{\alpha}$ is the concentration parameter of Dirichlet distribution predicted from $X$ by $\mathcal{F}(\cdot)$. Moreover, $p(Y|\bm{\mu},\bm{\alpha},\sigma^2)$ is assumed to be a multivariate Gaussian distribution $\mathcal{N}\big(Y|\bm{\mu},\sigma^2\mathcal{I}(\bm{\alpha})^{-1}\big)$, and $\mathcal{I}(\bm{\alpha})$ is referred to as Fisher Information Matrix (\textbf{FIM}) for $Dir(\bm{\alpha})$. 

Intuitively, FIM is introduced into the variance of predictive distribution to obtain a non-isotropic Normal distribution for the label generation of a specific sample. As a result, there is an adaptive weight provided by FIM. This weight is assigned to each class in eventual loss function, to adjust the information of each class contained in the sample. This could avoid the potential over-penalty of some classes in the supervision based on one-hot labels. 

Given a dataset $\{(X_j,Y_j)\}_{j=1}^N$, the eventual form of $\mathcal{L}_{\mathcal{I}\text{-EDL}}$ is as follows:
$$
\mathop{\min} \frac{1}{N} \sum_{j=1}^{N} \big( \mathcal{L}_j^{\mathcal{I}\text{-MSE}} - \lambda_1 \mathcal{L}_j^{\mathcal{|I|}} + \lambda_2 \mathcal{L}_j^{\text{KL}}  \big),
$$
where 
$$
\mathcal{L}_{j}^{\mathcal{I}\text{-MSE}}=\sum_{i=1}^{C}\left((y_{ji}-\frac{\alpha_{ji}}{\alpha_{j0}})^2+\frac{\alpha_{ji}(\alpha_{j0}-\alpha_{ji})}{\alpha_{j0}^2(\alpha_{j0}+1)}\right)\psi^{(1)}(\alpha_{ji}),
$$
$$
\mathcal{L}_j^{|\mathcal{I}|}=\sum_{i=1}^C\log\psi^{(1)}(\alpha_{ji})+\log\left(1-\sum_{i=1}^C\frac{\psi^{(1)}(\alpha_{j0})}{\psi^{(1)}(\alpha_{ji})}\right),
$$
$$
\mathcal{L}_j^\mathrm{KL}=\log\Gamma\Big(\sum_{i=1}^C\hat{\alpha}_{ji}\Big)-\log\Gamma(C)-\sum_{i=1}^C\log\Gamma(\hat{\alpha}_{ji}) +\sum_{i=1}^C(\hat{\alpha}_{ji}-1)\left[\psi(\hat{\alpha}_{ji})-\psi\Big(\sum_{c=1}^C\hat{\alpha}_{jc}\Big)\right],
$$
and $\lambda_1\geq 0, \lambda_2\geq 0$. Moreover, $\hat{\alpha}_{ji}=\alpha_{ji}(1-Y_{ji})+Y_{ji}\ \forall j\in[1,N], i\in [1,C]$ . $\psi(\cdot)$ is a \textit{digamma} function defined as $\psi(x)=\frac{d}{dx}\mathop{\log}\Gamma(x)$,
$\psi^{(1)}(\cdot)$ is a \textit{trigamma} function with $\psi^{(1)}(x)=\frac{d}{dx}\psi(x)$, and $\Gamma(\cdot)$ stands for a \textit{gamma} function. 
Readers could refer to \citeauthor{pmlr-v202-deng23b} \citeyearpar{pmlr-v202-deng23b} for derivation details. 

For better understanding, we briefly explain these terms as follows:

\textbf{(1) For the first term} $\mathcal{L}_{j}^{\mathcal{I}\text{-MSE}}$, it introduces a new $\psi^{(1)}(\alpha_{ji})$ to the MSE (mean square error) loss $\mathcal{L}_{j}^{\text{MSE}}$ frequently-used in EDL \cite{sensoy2018evidential}. Concretely, $$\mathcal{L}_{j}^{\text{MSE}}=\sum_{i=1}^{C}\left((y_{ji}-\frac{\alpha_{ji}}{\alpha_{j0}})^2+\frac{\alpha_{ji}(\alpha_{j0}-\alpha_{ji})}{\alpha_{j0}^2(\alpha_{j0}+1)}\right).$$ It is derived from a simple MSE-based Bayes risk function:
$$\mathcal{L}_{\text{MSE}}=\int ||Y-\bm{\mu}||_{2}^{2} Dir(\bm{\mu}|\bm{\alpha})d \bm{\mu}.$$ $\psi^{(1)}(\alpha_{ji})$ is specially added into $\mathcal{L}_{j}^{\text{MSE}}$, in order to encourage the model to focus more on the class with low evidence. 

\textbf{(2) For the second term} $-\mathcal{L}_j^{|\mathcal{I}|}$, it is equal to $-\mathop{\log}|\mathcal{I}(\bm{\alpha}_j)|$, \textit{i.e.},
$$
-\mathcal{L}_j^{|\mathcal{I}|} = -\mathop{\log}|\mathcal{I}(\bm{\alpha}_j)|.
$$
It is taken to avoid the overconfidence caused by excessive evidence. 

\textbf{(3) For the final term} $\mathcal{L}_j^\mathrm{KL}$, its original form \cite{sensoy2018evidential} is as follows:
$$
\mathcal{L}_j^\mathrm{KL}=\mathrm{KL}(Dir(\bm{\mu}|\hat{\bm{\alpha}}_j)||Dir(\bm{\mu}|\bm{1})),
$$
where $\mathrm{KL}(\cdot)$ is a function measuring Kullback-Leibler (KL) divergence. Moreover,  $\hat{\bm{\alpha}}_j=\bm{\alpha}_j \odot (1-\bm{Y}_j) + \bm{Y}_j$ where $\bm{Y}_j$ stands for the one-hot label of $Y_j$. It indicates manually masking the predicted parameter corresponding to the ground-truth class. Therefore, $\mathcal{L}_j^\mathrm{KL}$ can be view as a loss term aiming to suppress the evidence of irrelevant classes. 

\section{Experimental Details}

This section provides the additional details of experimental setup (Section \ref{sec-exp-setup}), including bag generation (\ref{apx-gen-bag}), datasets (\ref{apx-dataset}), and implementation and network training (\ref{apx-net-imp-training}). Our source code has been submitted as Supplementary Material.

\subsection{Bag Generation}
\label{apx-gen-bag}

Following \citeauthor{ilse2018attention}~\citeyearpar{ilse2018attention}, we generate a bag dataset for MIL from a given single-instance dataset. \textbf{(1) Steps}: At first, we set \textbf{a class of interest} (as positive class) and this class is from the given dataset. Then, we randomly select a certain number of samples from the given dataset to form a multi-instance bag. This bag is positive if it contains at least one sample from the class of interest; otherwise, it is negative. Bag length follows a Normal distribution $\mathcal{N}(10, 2)$.
\textbf{(2) More settings}: positive and negative bags are generated sequentially in a loop to obtain a balanced bag dataset. The ratio of positive instances roughly follows an Uniform distribution $\mathrm{U}(0,1)$
for positive bags. No that, we ensure that all the instances of training bags are only sampled from the training set of the given dataset, and do so for validation and test bags.  

\subsection{Datasets}
\label{apx-dataset}

\textbf{MNIST-bags} \cite{lecun1998mnist}~Following the dataset setting in \citeauthor{ilse2018attention}~\citeyearpar{ilse2018attention}, there are 500, 100, and 1000 generated MNIST bags in training, validation, and test set, respectively. Each bag contains multiple MNIST images (each image with the size of $1\times 28 \times 28$). The number `9' is set as the class of interest, as it is easily mistaken with `7' and `4' in hand-written numbers. In OOD detection tasks, \textbf{FMNIST-bags} \cite{xiao2017fashion} and \textbf{KMNIST-bags} \cite{clanuwat2018deep} are taken as OOD MIL datasets. Both of two contain 1000 OOD bags. The length of these OOD bags also follows $\mathcal{N}(10, 2)$. We show bag examples in Fig. \ref{apx-ex-bag}.

\textbf{CIFAR10-bags} \cite{krizhevsky2009learning}~We generate 6000, 1000 and 2000 CIFAR bags for training, validation, and test, respectively. This dataset is more complex than MNIST-bags since its instances, CIFAR10 images, are more diverse than MNIST ones. A single instance (image) in each bag is with the size of $3\times 32 \times 32$. We randomly select \textbf{`truck'} as the class of interest. In OOD detection tasks, \textbf{SVHN-bags} \cite{yuval2011reading} and \textbf{Texture-bags} \cite{cimpoi14describing} are two OOD MIL datasets; and each contains 2,000 OOD bags. OOD instances are resized to $3\times 32 \times 32$, in order to keep the same size as CIFAR10 instances. Similarly, the length of OOD bags follows $\mathcal{N}(10, 2)$. Bag examples are exhibited in Fig. \ref{apx-ex-bag}.

\begin{figure*}[tp]
\vskip 0.1in
\begin{center}
\centerline{\includegraphics[width=0.7\columnwidth]{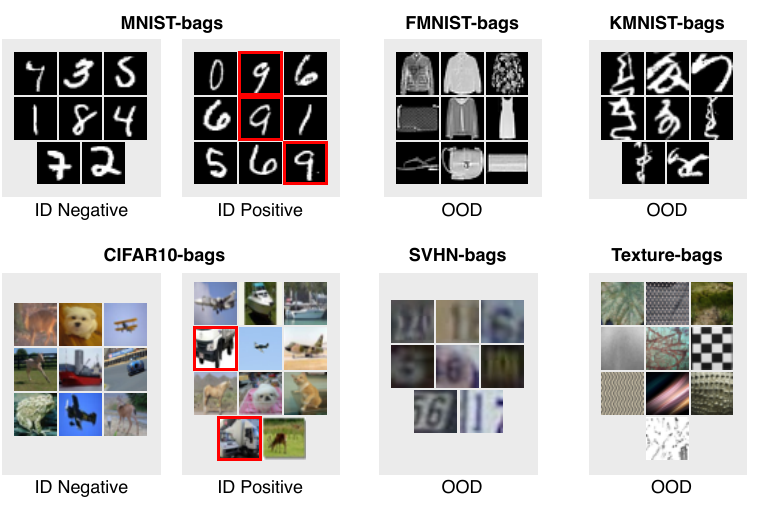}}
\caption{Bag examples of each bag dataset. Red box indicates the class of interest, `9' and `truck' for MNIST and CIFAR10, respectively.}
\label{apx-ex-bag}
\end{center}
\vskip -0.1in
\end{figure*}

\textbf{CAMELYON16} \cite{bejnordi2017diagnostic}~It is a real-world pathology dataset, originally proposed for breast cancer lymph node metastasis detection and frequently used for evaluating MIL algorithms. 
We obtain 270 and 129 histopathology WSIs (Whole-Slide Images) for training and test, respectively, provided by official organizers. There are 111 tumor slides and 159 normal slides in the training set, and 49 tumor slides and 80 normal slides in the test set.
We leave 15\% training samples as a validation set. Please refer to Fig. \ref{apx-ex-c16-bag}(a) for WSI examples. More details of CAMELYON16 are as follows:
\begin{itemize}
    \item \textbf{Preprocessing}:~Since a single WSI has extremely-high resolution (\textit{e.g.}, $40,000 \times 40,000$ pixels), we process each image into a bag of feature vectors with CLAM \cite{lu2021data} by three steps: i) tissue region selection, ii) image patching, and iii) patch feature extraction. Each patch is an image with $256\times 256$ pixels from the WSI at $20\times$ magnification. Feature vector is extracted from patch image by a fixed (frozen) deep network. This fixed network is pre-trained on the patches of \textit{training samples} by self-supervised learning, provided by \citeauthor{li2021dual}~\citeyearpar{li2021dual}. As a result, there are 11,753 instances in each WSI bag on average, and each instance is a feature vector with the length of 512.
    \item \textbf{OOD dataset}:~The histopathology WSIs of \textit{prostate cancer} are used as the OOD samples of CAMELYON16 (breast cancer), following \citeauthor{linmans2023predictive}~\citeyearpar{linmans2023predictive}. These WSIs are from \textbf{TCGA-PRAD} (The Cancer Genome Atlas Prostate Adenocarcinoma\footnote{Available at https://portal.gdc.cancer.gov/projects/TCGA-PRAD}) \cite{kandoth2013mutational}, containing 449 diagnostic images. Their preprocessing is the same as that of CAMELYON16. Finally, there are 3,484 instances in each bag on average. TCGA-PRAD samples are shown in Fig. \ref{apx-ex-c16-bag}(a). They often present differences with CAMELYON16 in cell distribution and tissue morphology. 
    \item \textbf{Distribution shift dataset}:~Given the test WSIs of CAMELYON16, we synthesize its three shifted versions using the image noises with different strengths, called \textit{lighter}, \textit{light}, and \textit{strong}. Specifically, Gaussian Blurring or HED (Hematoxylin-Eosin-DAB) color variation is applied to the patch images of each test WSI, to simulate the possible noises in digital pathology, following \citeauthor{tellez2019quantifying}~\citeyearpar{tellez2019quantifying} and \citeauthor{liu2024advmil}~\citeyearpar{liu2024advmil}. The patch image samples with different noises are shown in Fig. \ref{apx-ex-c16-bag}(b). Eventually, all the patch images with noises are transformed into instances (feature vectors) using the same feature extractor mentioned in WSI preprocessing. 
\end{itemize}

\begin{figure*}[tp]
\vskip 0.1in
\begin{center}
\centerline{\includegraphics[width=0.8\columnwidth]{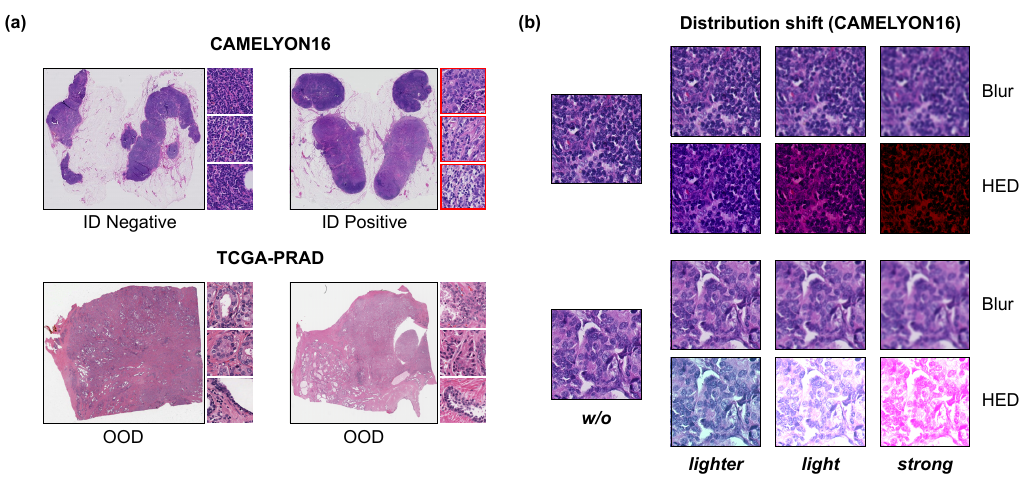}}
\caption{(a) Samples of CAMELYON16 and TCGA-PRAD. Red box indicates the patch with tumorous cells. (b) Patch samples of distribution shift CAMELYON16. Blur and HED mean the image noise of Gaussian Blurring and HED color variation, respectively.}
\label{apx-ex-c16-bag}
\end{center}
\vskip -0.1in
\end{figure*}

\subsection{Implementation and Network Training}
\label{apx-net-imp-training}

\textbf{Deep MIL networks}~The most representative deep MIL networks, \textit{e.g.}, Mean, Max, ABMIL \cite{ilse2018attention}, and DSMIL \cite{li2021dual}, are adopted in our experiments. From a unified perspective, these networks comprise three key parts as follows. (1) \textbf{Instance encoder}. We employ LeNet \cite{lecun1998mnist} as the encoder to transform MNIST images into instance embeddings, following \citeauthor{ilse2018attention}~\citeyearpar{ilse2018attention}. For CIFAR10-bags, a modified AlexNet \cite{NIPS2012_c399862d} is adopted. For CAMELYON16, we directly use an MLP (Multi-Layer Perceptron) layer, since image patches have been transformed into feature vectors. (2) \textbf{MIL pooling operator}. Mean and max-based pooling are used for Mean and Max MIL networks, respectively. For ABMIL and DSMIL, we follow their respective implementation in MIL pooling. \red{Specifically, for ABMIL, a standard attention mechanism, rather than its gated variant, is adopted in MIL pooling, because it is more efficient in computation and is often competitive in performance, compared to its gated variant \cite{ilse2018attention,shi2020loss}.} (3) \textbf{Classification head}. It is a fully-connected layer with negative and positive output nodes. 

\textbf{Related UE methods}~For the classical UE method, \textbf{Deep Ensemble} \cite{lakshminarayanan2017simple}, by default we train 10 ABMIL networks with different random seeds. For \textbf{MC Dropout} \cite{gal2016dropout}, Dropout layers are used in the instance encoder of ABMIL, with a drop rate of 0.25; 10 estimates are sampled from the network and their mean is taken as prediction. For BNN-based \textbf{Bayes-MIL} \cite{yufei2022bayes}, we follow its implementation to sample 16 estimates. Lastly, for \textbf{$\mathcal{I}$-EDL} \cite{pmlr-v202-deng23b}, we modify the classification head of ABMIL into an evidential output layer \cite{sensoy2018evidential}, and adopt a $\mathcal{I}$-EDL loss function for evidential learning. 

\textbf{MIREL}~Its implementation details are as follows. \textbf{(1) Bag-level network:} Our method could be combined with existing deep MIL networks for MIUE. Thus, we directly follow their implementations and employ them to implement our bag-level networks. In particular, we replace their conventional classification head with an evidential output layer \cite{sensoy2018evidential}, and utilize $\mathop{\exp}(\cdot)$ for the implementation of $\mathcal{A}(\cdot)$.
\textbf{(2) instance-level network:} For our proposed residual instance estimator, $R(\mathbf{x})=T(\mathbf{x}) + r_{\pi}(f_{\psi}(\mathbf{x}))$, $T(\cdot)$ is exactly the bag-level network, $f_{\psi}(\cdot)$ is the instance encoder of bag-level network, and $r_{\pi}(\cdot)$ is simply implemented by an MLP layer. To make instance-level evidential learning more stable, we adopt $\mathop{\mathrm{tanh}}(\cdot)$ to let $r_{\pi}$ output a scale value in $[-1, 1]$. This scale value expresses a residual estimate proportional to $T(\mathbf{x})$. 
\textbf{(3) Loss function:} Apart from Fisher Information-based objective function (Appendix \ref{apx-loss-func}), a RED loss  \cite{pandey2023learn} is also adopted to avoid zero-evidence regions in EDL, \red{as stated in the last paragraph of Section \ref{ins-edl}}.
\red{\textbf{(4) Optimization strategy for $\mathcal{L_{\text{MIREL}}}$:} In the experiments on MNIST-bags and CAMELYON16, the bag-level parameter $\psi$ and the instance-level parameter $\pi$ are optimized by $\mathcal{L_{\text{MIREL}}}$ in weakly-supervised instance residual evidential learning. While on CIFAR10-bags, only the instance-level parameter $\pi$ is involved in the optimization of $\mathcal{L_{\text{MIREL}}}$. Namely, we specially freeze the bag-level parameter $\psi$ in optimizing $\mathcal{L_{\text{MIREL}}}$. We will elaborate on this setting in Appendix \ref{apx-cifar-abl-study}.}
\textbf{(5) Hyper-parameters:} Following \citeauthor{pmlr-v202-deng23b} \citeyearpar{pmlr-v202-deng23b}, the coefficient $\lambda_1$ of $-\mathcal{L}_j^{|\mathcal{I}|}$ is set by a grid-search over (0.1, 0.05, 0.01, 0.005, 0.001), and the coefficient $\lambda_2$ of $\mathcal{L}_j^{\text{KL}}$ is set to $\mathop{\min}(1,\frac{t}{10})\in[0,1]$, where $t$ is the index of current training epoch. 
 
\textbf{Network training}~Learning rate, by default, is set to 0.0001 and it decays by a factor of 0.5 when the criterion on validation set does not decrease within 10 epochs. The other default settings are as follows: an epoch number of 200, a batch size of 1 (bag), a gradient accumulation step of 8, and an optimizer of Adam with a weight decay rate of $1 \times 10^{-5}$. Early stopping is applied when the criterion on validation set does not decrease within 20 epochs by default. The sum of loss and error is adopted as the criterion. Moreover, EDL-based models, \textit{e.g.}, $\mathcal{I}\text{-EDL}$ and our MIREL, are trained using the same $\mathcal{L}_{\mathcal{I}\text{-EDL}}$; while the other standard classification models use $\mathcal{L}_{\text{BCE}}$, \textit{i.e.}, a BCE (binary cross-entropy) loss. In ablation study, three types of models are trained. Their details are shown in Table \ref{table-baselines}. 

\begin{table*}[htbp]
\caption{Details of the models used in ablation study.}
\label{table-baselines}
\begin{center}
\begin{small}
\begin{tabular}{l|cc}
\toprule
Network & Loss & Instance prediction \\
\midrule
\multirow{2}{*}{Deep MIL} & \multirow{2}{*}{$\mathcal{L}_{\text{BCE}}$} & $a_k$ (attention score) \\
 &  & $T(\mathbf{x})$ \\ \cmidrule(lr){1-3}
\multirow{2}{*}{Deep MIL + EDL} & \multirow{2}{*}{$\mathcal{L}_{\mathcal{I}\text{-EDL}} \red{ + \mathcal{L}_{\text{RED}}}$} & $a_k$ (attention score) \\
 &  & $T(\mathbf{x})$ \\ \cmidrule(lr){1-3}
Deep MIL + MIREL & $\mathcal{L}_{\mathcal{I}\text{-EDL}} + \mathcal{L}_{\text{MIREL}} \red{ + \mathcal{L}_{\text{RED}}}$ & $R(\mathbf{x})$ \\ 
\bottomrule
\end{tabular}
\end{small}
\end{center}
\end{table*}

\section{Additional Results on MNIST-bags}

\subsection{Uncertainty Analysis}
\label{apx-mnist-unc-analysis}

\textbf{ABMIL}~(1) The result of uncertainty analysis on KMNIST-bags is shown in Fig. \ref{apx-abmil-kmnist}. From this result, we could see that i) the ABMIL w/ MIREL can provide more discriminative predictive uncertainty for the bags with different OOD ratios, compared to original ABMIL; ii) our method can assist ABMIL in distinguishing ID (MNIST) instances and OOD (KMNIST) ones more accurately.
(2) $\alpha_0$ is another commonly-used uncertainty measure for EDL models. Its distribution at instance and bag levels are shown in Fig. \ref{apx-abmil-alpha0}. These results suggest that our MIREL could also detect OOD samples through the concentration parameter $\alpha_0$.  

\begin{figure*}[ht]
\vskip 0.1in
\begin{center}
\centerline{\includegraphics[width=\columnwidth]{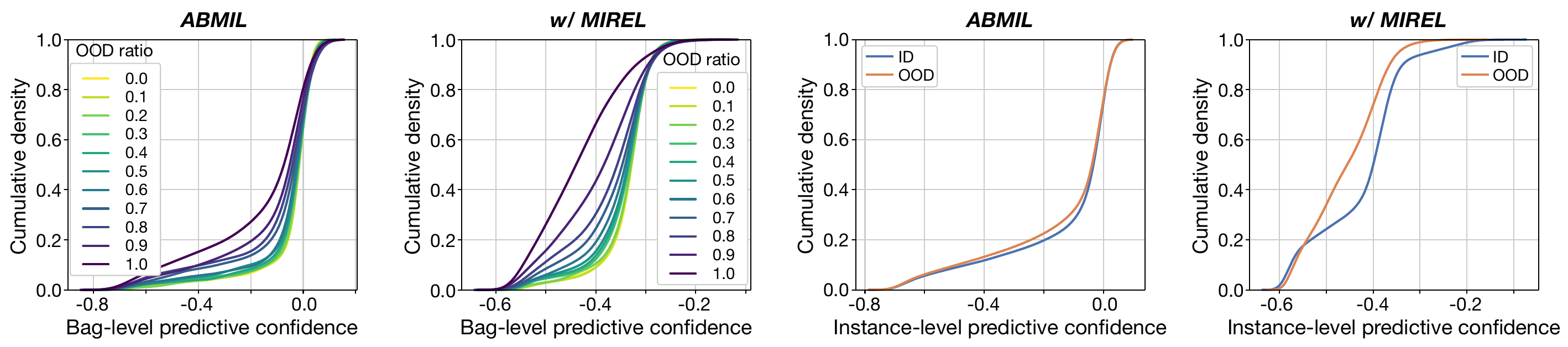}}
\caption{Distribution of bag-level and instance-level predictive confidence (negative expected entropy). \textbf{MNIST-bags} is ID dataset. The OOD instances used in this experiment are from KMNIST.}
\label{apx-abmil-kmnist}
\end{center}
\vskip -0.1in
\end{figure*}

\begin{figure*}[htbp]
\vskip 0.1in
\begin{center}
\centerline{\includegraphics[width=\columnwidth]{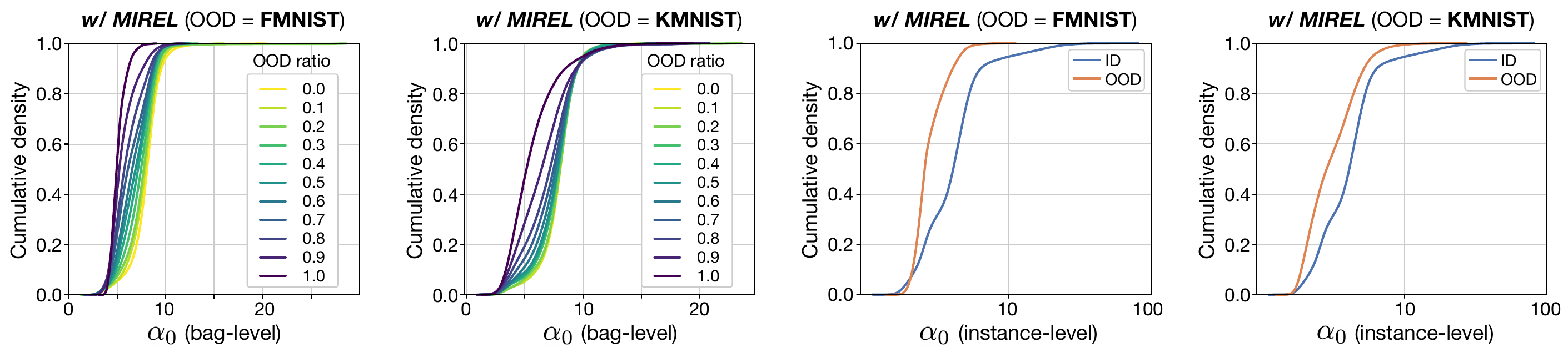}}
\caption{Distribution of bag-level and instance-level $\alpha_0$ output by the ABMIL models with our MIREL. \textbf{MNIST-bags} is ID dataset.}
\label{apx-abmil-alpha0}
\end{center}
\vskip -0.1in
\end{figure*}

\begin{figure*}[htbp]
\vskip 0.1in
\begin{center}
\centerline{\includegraphics[width=\columnwidth]{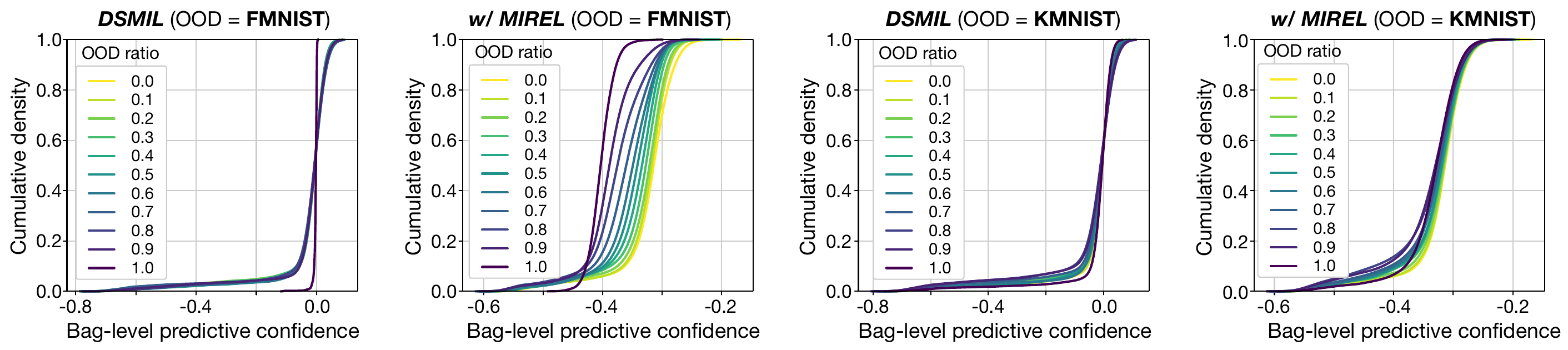}}
\caption{Distribution of bag-level predictive confidence (negative expected entropy). DSMIL is the base MIL network in this experiment. ID dataset is \textbf{MNIST-bags}.}
\label{apx-dsmil-bag-unc}
\end{center}
\vskip -0.1in
\end{figure*}

\begin{figure*}[htbp]
\vskip 0.1in
\begin{center}
\centerline{\includegraphics[width=0.8\columnwidth]{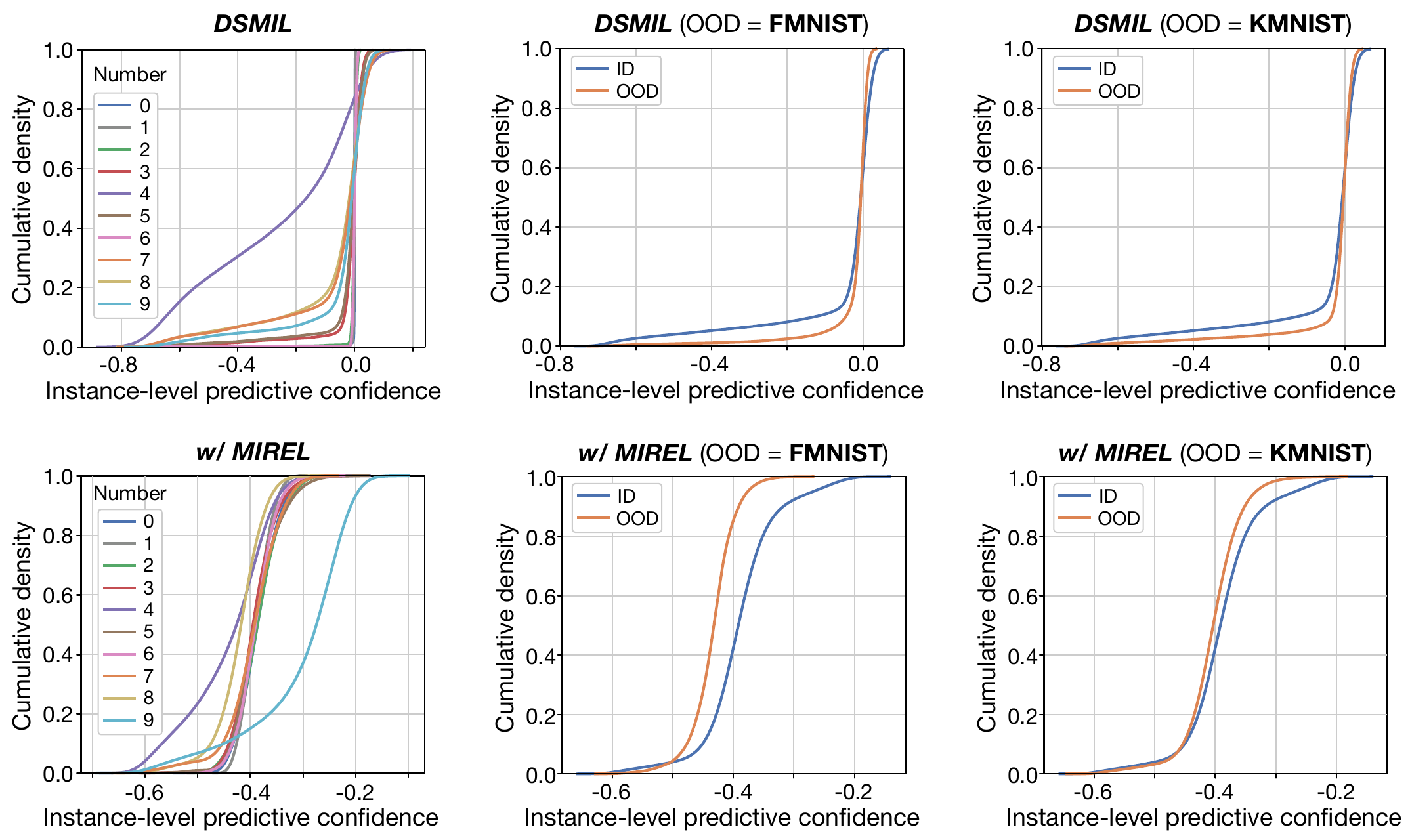}}
\caption{Distribution of instance-level predictive confidence (negative expected entropy). DSMIL is the base MIL network in this experiment. ID dataset is \textbf{MNIST-bags}.}
\label{apx-dsmil-ins-unc}
\end{center}
\vskip -0.1in
\end{figure*}

\textbf{DSMIL}~We show more results of uncertainty analysis, in which DSMIL is taken as the base MIL network. These results contain bag-level UE (Fig. \ref{apx-dsmil-bag-unc}), instance-level UE (Fig. \ref{apx-dsmil-ins-unc}), and $\alpha_0$ estimate (Fig. \ref{apx-dsmil-alpha0}). We summarize our observations from these as follows. (1) When OOD dataset is FMNIST-bags, our MIREL helps DSMIL to provide more accurate uncertainty for the bags with different OOD ratios, while vanilla DSMIL often shows overconfident prediction and cannot response to abnormal bags. When using the bags with different OOD instance ratios, there is no obvious change in bag-level predictive confidence for both DSMIL and its MIREL counterpart. (2) At instance level, DSMIL often mistakenly predicts more confidence for OOD instances than ID ones. After combining with our MIREL, DSMIL tends to assign ID instances with more confidence. (3) There are similar findings from the results of $\alpha_0$ for the DSMIL models with our MIREL. 

\begin{figure*}[htbp]
\vskip 0.1in
\begin{center}
\centerline{\includegraphics[width=\columnwidth]{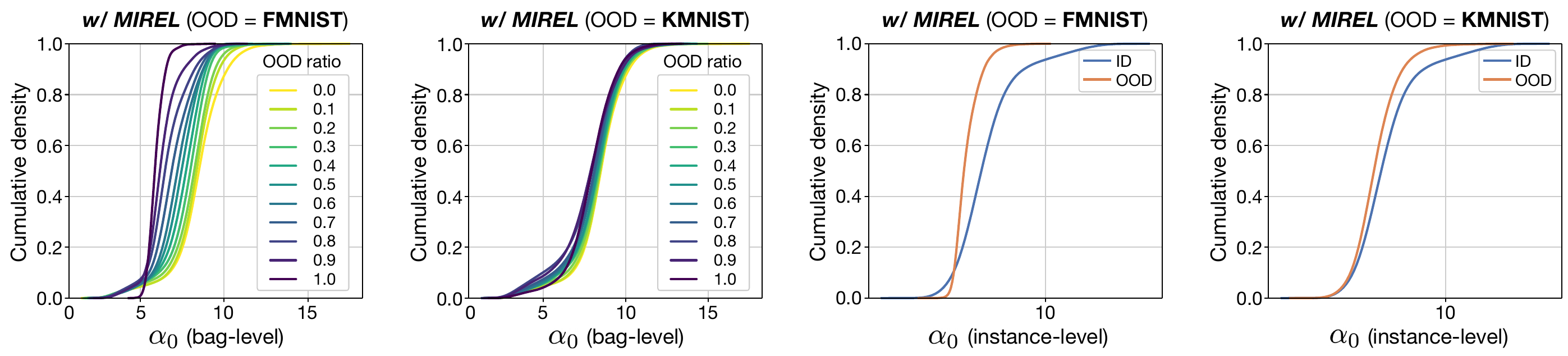}}
\caption{Distribution of bag-level and instance-level $\alpha_0$ output by the DSMIL models with our MIREL. ID dataset is \textbf{MNIST-bags}.}
\label{apx-dsmil-alpha0}
\end{center}
\vskip -0.1in
\end{figure*}

\subsection{Ablation Study}
\label{apx-mnist-abl-study}

\textbf{Optimization strategy for $R(\mathbf{x})$}~We adopt three different optimization strategies for \textit{the instances from positive bags}. They can be represented by three loss functions, $\mathcal{L}_{1}$, $\mathcal{L}_{2}$, and $\mathcal{L}_{\text{ins}}^{+}$, as described in Appendix \ref{apx-der-insloss}.
Test results are shown in Table \ref{apx-table-mnist-ins-loss}. Our main findings are as follows.
(1) Compared to $\mathcal{L}_{\text{ins}}^{+}$, the UE performance obtained by $\mathcal{L}_{1}$ often lags far behind. (2) Compared to $\mathcal{L}_{\text{ins}}^{+}$, $\mathcal{L}_{2}$ leads to the worse overall UE performance at bag level, with a drop of 2.84\%, although it performs slightly better at instance level, with a narrow increase of 0.77\%. These findings empirically demonstrate the effectiveness of our weakly-supervised evidential learning strategy. 

\begin{table*}[ht]
\caption{Ablation study on the loss function used for training $R(\mathbf{x})$. The base MIL network is ABMIL and it is trained on \textbf{MNIST-bags}.}
\label{apx-table-mnist-ins-loss}
\begin{center}
\begin{scriptsize}
\begin{tabular}{l|ccccc|ccccc}
\toprule
\multirow{2}{*}{\textbf{Loss}} & \multicolumn{5}{c|}{\underline{\textbf{\ \ Bag-level\ \ }}}            & \multicolumn{5}{c}{\underline{\textbf{\ \ Instance-level\ \ }}}       \\
 & \textbf{Acc.} & \textbf{Conf.} & \textbf{OOD-F} & \textbf{OOD-K} & \textbf{$\overline{\textit{UE}}$} & \textbf{Acc.} & \textbf{Conf.} & \textbf{OOD-F} & \textbf{OOD-K} & \textbf{$\overline{\textit{UE}}$} \\
\midrule
$\mathcal{L}_{1}$ & 96.46{\tiny\ $\pm$ 0.59} & 83.30{\tiny\ $\pm$ 4.43} & 90.91{\tiny\ $\pm$ 2.46} & 76.89{\tiny\ $\pm$ 2.76} & 83.70 & 85.70{\tiny\ $\pm$ 0.50} & 85.42{\tiny\ $\pm$ 5.48} & 66.63{\tiny\ $\pm$ 4.77} & 63.00{\tiny\ $\pm$ 4.82} & 71.68  \\
$\mathcal{L}_{2}$ & 96.46{\tiny\ $\pm$ 0.30} & 84.40{\tiny\ $\pm$ 2.41} & 91.15{\tiny\ $\pm$ 2.91} & 75.33{\tiny\ $\pm$ 3.14} & 83.63 & 86.29{\tiny\ $\pm$ 0.59} & 90.67{\tiny\ $\pm$ 1.71} & \textbf{81.27}{\tiny\ $\pm$ 3.63} & 66.26{\tiny\ $\pm$ 3.53} & \textbf{79.40}  \\
$\mathcal{L}_{\text{ins}}^{+}$ & 96.48{\tiny\ $\pm$ 0.22} & \textbf{86.63}{\tiny\ $\pm$ 1.32} & \textbf{92.84}{\tiny\ $\pm$ 0.60} & \textbf{79.95}{\tiny\ $\pm$ 4.12} & \textbf{86.47} & 87.71{\tiny\ $\pm$ 0.67} & \textbf{90.73}{\tiny\ $\pm$ 1.31} & 78.13{\tiny\ $\pm$ 2.19} & \textbf{67.02}{\tiny\ $\pm$ 1.94} & 78.63 \\  
\bottomrule
\end{tabular}
\end{scriptsize}
\end{center}
\end{table*}

\red{\textbf{The effect of $\mathcal{L}_{\text{RED}}$ on our MIREL}~As shown in Table \ref{apx-table-mnist-red-loss}, we could find that the involvement of RED loss often obtains performance improvements over its counterpart. This is largely because $\mathcal{L}_{\text{RED}}$ can effectively mitigate zero-evidence regions to improve evidential learning, as highlighted in \citeauthor{pandey2023learn}~\citeyearpar{pandey2023learn}.}

\begin{table*}[htbp]
\caption{\red{Ablation study on the effect of RED loss on our MIREL (\textbf{MNIST-bags}). The base MIL network is ABMIL.}}
\label{apx-table-mnist-red-loss}
\begin{center}
\begin{scriptsize}
\red{
\begin{tabular}{c|ccccc|ccccc}
\toprule
\multirow{2}{*}{$\mathcal{L}_{\text{RED}}$} & \multicolumn{5}{c|}{\underline{\textbf{\ \ Bag-level\ \ }}}            & \multicolumn{5}{c}{\underline{\textbf{\ \ Instance-level\ \ }}}       \\
 & \textbf{Acc.} & \textbf{Conf.} & \textbf{OOD-F} & \textbf{OOD-K} & \textbf{$\overline{\textit{UE}}$} & \textbf{Acc.} & \textbf{Conf.} & \textbf{OOD-F} & \textbf{OOD-K} & \textbf{$\overline{\textit{UE}}$} \\
\midrule
$\times$ & 96.12{\tiny\ $\pm$ 0.37} & 84.89{\tiny\ $\pm$ 3.49} & 88.41{\tiny\ $\pm$ 5.15} & 78.46{\tiny\ $\pm$ 1.76} & 83.92 & 85.40{\tiny\ $\pm$ 1.78} & 96.89{\tiny\ $\pm$ 1.25} & 63.50{\tiny\ $\pm$ 4.50} & 57.72{\tiny\ $\pm$ 1.86} & 72.70   \\
\checkmark & 96.48{\tiny\ $\pm$ 0.22} & 86.63{\tiny\ $\pm$ 1.32} & 92.84{\tiny\ $\pm$ 0.60} & 79.95{\tiny\ $\pm$ 4.12} & \textbf{86.47} & 87.71{\tiny\ $\pm$ 0.67} & 90.73{\tiny\ $\pm$ 1.31} & 78.13{\tiny\ $\pm$ 2.19} & 67.02{\tiny\ $\pm$ 1.94} & \textbf{78.63} \\  
\bottomrule
\end{tabular}
}
\end{scriptsize}
\end{center}
\end{table*}

\begin{table*}[htbp]
\caption{Additional instance-level ablation study on $T(\mathbf{x})$ for related UE methods (\textbf{MNIST-bags}). $\dagger$ These methods directly adopt our $T(\mathbf{x})$ derived from $S(X)$ for instance-level estimation. The other results are copied from Table \ref{table-mnist} for comparisons.}
\label{apx-table-mnist}
\begin{center}
\begin{small}
\begin{tabular}{l|c|ccccc}
\toprule
\multirow{2}{*}{\textbf{Method}} & \multirow{2}{*}{\textbf{Ins.}} & \multicolumn{5}{c}{\underline{\textbf{\ \ Instance-level\ \ }}}       \\
 & & \textbf{Acc.} & \textbf{Conf.} & \textbf{OOD-F} & \textbf{OOD-K} & \textbf{$\overline{\textit{UE}}$} \\
\midrule
Deep Ensemble & $a_k$ & 75.56{\tiny\ $\pm$ 0.32} & 71.89{\tiny\ $\pm$ 0.91} & 70.48{\tiny\ $\pm$ 0.53} & 55.22{\tiny\ $\pm$ 1.16} & 65.87 \\
Deep Ensemble $\dagger$ & $T$ & 85.97{\tiny\ $\pm$ 1.47} & 84.57{\tiny\ $\pm$ 2.62} & 63.75{\tiny\ $\pm$ 2.44} & 54.22{\tiny\ $\pm$ 3.01} & \textbf{67.51} \\ \midrule
MC Dropout  & $a_k$ & 75.61{\tiny\ $\pm$ 0.66} & 68.40{\tiny\ $\pm$ 1.54} & 68.34{\tiny\ $\pm$ 1.06} & 58.61{\tiny\ $\pm$ 1.38} & 65.12 \\
MC Dropout $\dagger$ & $T$ & 88.85{\tiny\ $\pm$ 1.54} & 85.19{\tiny\ $\pm$ 3.44} & 71.61{\tiny\ $\pm$ 3.18} & 57.80{\tiny\ $\pm$ 1.53} & \textbf{71.53} \\ \midrule
$\mathcal{I}$-EDL & $a_k$ & 75.45{\tiny\ $\pm$ 0.13} & 60.72{\tiny\ $\pm$ 1.46} & 63.91{\tiny\ $\pm$ 1.31} & 54.14{\tiny\ $\pm$ 2.19} & 59.59 \\
$\mathcal{I}$-EDL $\dagger$ & $T$ & 85.19{\tiny\ $\pm$ 0.64} & 87.67{\tiny\ $\pm$ 1.11} & 73.52{\tiny\ $\pm$ 5.66} & 56.63{\tiny\ $\pm$ 1.66} & \textbf{72.60} \\ 
\bottomrule
\end{tabular}
\end{small}
\end{center}
\end{table*}

\textbf{Related UE methods}~As shown in Table \ref{apx-table-mnist}, our derived $T(\mathbf{x})$ could often boost the performance of related UE methods in instance-level UE tasks. Moreover, our $T(\mathbf{x})$ surpasses attention-based scoring proxy ($a_k$) in overall UE performance by 1.64\%, 6.41\%, and 13.01\% for Deep Ensemble, MC Dropout, and $\mathcal{I}\text{-EDL}$, respectively. This study further demonstrates the superiority of our $T(\mathbf{x})$ to conventional attention-based instance scoring proxy.  

\subsection{\red{More Experiments with Different Settings}}
\label{apx-mnist-more-settings}

\red{To investigate the effect of different experimental settings on MIREL's performance, we conduct more experiments and show their results in this section.}

\red{\textbf{Adopting gated attention mechanism in ABMIL}~When using ABMIL as the base network for our MIREL, we compare two different attention operators proposed in ABMIL, namely, standard attention mechanism and gated attention mechanism. Their results are presented in Table \ref{apx-table-mnist-gated-attn}. These results show that the standard attention operator is competitive with its gated variant in terms of average UE performance. The standard attention mechanism is our default setting in ABMIL.}

\begin{table*}[htbp]
\caption{\red{Performance of our MIREL when using standard or gated attention mechanism for ABMIL (\textbf{MNIST-bags}).}}
\label{apx-table-mnist-gated-attn}
\begin{center}
\begin{scriptsize}
\red{
\begin{tabular}{c|ccccc|ccccc}
\toprule
\multirow{2}{*}{\textbf{Attention}} & \multicolumn{5}{c|}{\underline{\textbf{\ \ Bag-level\ \ }}}            & \multicolumn{5}{c}{\underline{\textbf{\ \ Instance-level\ \ }}}       \\
 & \textbf{Acc.} & \textbf{Conf.} & \textbf{OOD-F} & \textbf{OOD-K} & \textbf{$\overline{\textit{UE}}$} & \textbf{Acc.} & \textbf{Conf.} & \textbf{OOD-F} & \textbf{OOD-K} & \textbf{$\overline{\textit{UE}}$} \\
\midrule
Gated & 96.52{\tiny\ $\pm$ 0.29} & 87.57{\tiny\ $\pm$ 2.51} & 93.84{\tiny\ $\pm$ 2.37} & 70.67{\tiny\ $\pm$ 4.74} & 84.03 & 87.96{\tiny\ $\pm$ 0.86} & 87.95{\tiny\ $\pm$ 2.46} & 81.15{\tiny\ $\pm$ 2.13} & 70.45{\tiny\ $\pm$ 1.13} & \textbf{79.85}   \\
Standard & 96.48{\tiny\ $\pm$ 0.22} & 86.63{\tiny\ $\pm$ 1.32} & 92.84{\tiny\ $\pm$ 0.60} & 79.95{\tiny\ $\pm$ 4.12} & \textbf{86.47} & 87.71{\tiny\ $\pm$ 0.67} & 90.73{\tiny\ $\pm$ 1.31} & 78.13{\tiny\ $\pm$ 2.19} & 67.02{\tiny\ $\pm$ 1.94} & 78.63 \\  
\bottomrule
\end{tabular}
}
\end{scriptsize}
\end{center}
\end{table*}

\red{\textbf{Comparison with UE methods on DSMIL}~The results of this experiment are shown in Table \ref{apx-table-mnist-cmp-ue-on-dsmil}. From these results, we observe that our MIREL could still perform better than compared methods in terms of overall UE performance, even when changing the base MIL network from ABMIL to DSMIL. This implies that our method is of good adaptability.}

\begin{table*}[htbp]
\caption{\red{Comparison with UE methods when using DSMIL as the base MIL network (\textbf{MNIST-bags}). The baseline of this experiment is vanilla DSMIL without any additional UE techniques. Bayes-MIL is not compared here because it is not compatible with DSMIL.} 
}
\label{apx-table-mnist-cmp-ue-on-dsmil}
\begin{center}
\begin{scriptsize}
\tabcolsep=0.19cm
\red{
\begin{tabular}{l|ccccc|ccccc}
\toprule
\multirow{2}{*}{\textbf{Method}} & \multicolumn{5}{c|}{\underline{\textbf{\ \ Bag-level\ \ }}}            & \multicolumn{5}{c}{\underline{\textbf{\ \ Instance-level\ \ }}}       \\
 & \textbf{Acc.} & \textbf{Conf.} & \textbf{OOD-F} & \textbf{OOD-K} & \textbf{$\overline{\textit{UE}}$} & \textbf{Acc.} & \textbf{Conf.} & \textbf{OOD-F} & \textbf{OOD-K} & \textbf{$\overline{\textit{UE}}$} \\
\midrule
Baseline & \textcolor[RGB]{128,128,128}{96.22{\tiny\ $\pm$ 0.17}} & \textcolor[RGB]{128,128,128}{87.56{\tiny\ $\pm$ 0.95}} & \textcolor[RGB]{128,128,128}{71.13{\tiny\ $\pm$ 5.20}} & \textcolor[RGB]{128,128,128}{60.71{\tiny\ $\pm$ 7.91}} & \textcolor[RGB]{128,128,128}{73.13} &  \textcolor[RGB]{128,128,128}{70.16{\tiny\ $\pm$ 3.56}} & \textcolor[RGB]{128,128,128}{64.64{\tiny\ $\pm$ 0.49}} & \textcolor[RGB]{128,128,128}{59.75{\tiny\ $\pm$ 2.35}} & \textcolor[RGB]{128,128,128}{57.50{\tiny\ $\pm$ 2.55}} & \textcolor[RGB]{128,128,128}{60.63}     \\ \midrule
Deep Ensemble & 96.66{\tiny\ $\pm$ 0.17} & 87.15{\tiny\ $\pm$ 0.99} & 76.06{\tiny\ $\pm$ 2.12} & 64.94{\tiny\ $\pm$ 1.49} & 76.05 & 72.68{\tiny\ $\pm$ 0.84} & 70.18{\tiny\ $\pm$ 0.64} & 70.15{\tiny\ $\pm$ 2.27} & 64.01{\tiny\ $\pm$ 1.65} & 68.11  \\
MC Dropout  & 96.36{\tiny\ $\pm$ 0.43} & 88.13{\tiny\ $\pm$ 0.61} & 77.82{\tiny\ $\pm$ 2.85} & 66.56{\tiny\ $\pm$ 6.59} & 77.50 & 70.27{\tiny\ $\pm$ 3.01} & 64.78{\tiny\ $\pm$ 1.33} & 64.88{\tiny\ $\pm$ 5.38} & 60.78{\tiny\ $\pm$ 4.81} & 63.48    \\
$\mathcal{I}$-EDL & 96.60{\tiny\ $\pm$ 0.44} & 89.53{\tiny\ $\pm$ 2.03} & 79.69{\tiny\ $\pm$ 9.72} & 57.77{\tiny\ $\pm$ 6.01} & 75.67 & 69.04{\tiny\ $\pm$ 2.84} & 63.68{\tiny\ $\pm$ 1.43} & 62.08{\tiny\ $\pm$ 2.35} & 57.93{\tiny\ $\pm$ 2.62} & 61.23  \\ 
\textbf{MIREL}  & 96.50{\tiny\ $\pm$ 0.37} & 87.26{\tiny\ $\pm$ 2.66} & 87.27{\tiny\ $\pm$ 4.27} & 62.03{\tiny\ $\pm$ 7.78} & \textbf{78.85} & 97.19{\tiny\ $\pm$ 0.29} & 73.79{\tiny\ $\pm$ 15.68} & 73.29{\tiny\ $\pm$ 10.85} & 57.58{\tiny\ $\pm$ 3.44}  & \textbf{68.22} \\ 
\bottomrule
\end{tabular}
}
\end{scriptsize}
\end{center}
\end{table*}

\section{Results on CIFAR10-bags}
\label{apx-cifar-res}

\subsection{Main Results}
\label{apx-cifar-main}

The main comparative results on CIFAR10-bags are shown in Table \ref{apx-table-cifar}. \textbf{(1) Classical deep MIL networks:} Our MIREL could often assist them to perform better in UE. Especially for Max, DSMIL, and ABMIL, the improvements in overall UE performance are 10.45\%, 5.11\%, and 9.75\% at bag level, and 2.80\%, 12.06\%, and 20.85\% at instance level, respectively. \textbf{(2) Related UE methods:} With the same base MIL network (ABMIL), our MIREL could often obtain better UE performance than others. Especially at instance level, there is an improvement of 2.40\% over the runner-up method in overall UE performance. Moreover, It is worth mentioning that, our MIREL only requires a single forward pass for UE, different from the compared Deep Ensemble, MC Dropout, and Bayes-MIL involving multiple forward passes. 

\begin{table*}[ht]
\caption{Main results on \textbf{CIFAR10-bags}. OOD-S and OOD-T mean that SVHN and Texture are used for generating OOD bags, respectively. The results colored in gray are from our derived instance estimator $T(\mathbf{x})$. \textbf{$\overline{\textit{UE}}$} is the average metrics on three UE tasks.
}
\label{apx-table-cifar}
\begin{center}
\begin{scriptsize}
\tabcolsep=0.178cm
\begin{tabular}{l|ccccc|ccccc}
\toprule
\multirow{2}{*}{\textbf{Method}} & \multicolumn{5}{c|}{\underline{\textbf{\ \ Bag-level\ \ }}}            & \multicolumn{5}{c}{\underline{\textbf{\ \ Instance-level\ \ }}}       \\
 & \textbf{Acc.} & \textbf{Conf.} & \textbf{OOD-S} & \textbf{OOD-T} & \textbf{$\overline{\textit{UE}}$} & \textbf{Acc.} & \textbf{Conf.} & \textbf{OOD-S} & \textbf{OOD-T} & \textbf{$\overline{\textit{UE}}$} \\
\midrule
\multicolumn{11}{l}{- \textit{Combined with deep MIL networks}} \\ \midrule
Mean & 91.69{\tiny\ $\pm$ 0.65} & 84.79{\tiny\ $\pm$ 1.71} & 87.99{\tiny\ $\pm$ 3.09} & 73.15{\tiny\ $\pm$ 4.32} & \textbf{81.98} & \textcolor[RGB]{128,128,128}{91.92{\tiny\ $\pm$ 0.58}} & \textcolor[RGB]{128,128,128}{56.62{\tiny\ $\pm$ 3.32}} & \textcolor[RGB]{128,128,128}{88.03{\tiny\ $\pm$ 2.46}} & \textcolor[RGB]{128,128,128}{69.10{\tiny\ $\pm$ 3.42}} & \textcolor[RGB]{128,128,128}{\textbf{71.25}}   \\
\red{Mean + }MIREL & 92.07{\tiny\ $\pm$ 0.80} & 83.07{\tiny\ $\pm$ 1.64} & 81.05{\tiny\ $\pm$ 7.82} & 73.38{\tiny\ $\pm$ 6.47} & 79.17 & 92.09{\tiny\ $\pm$ 0.75} & 53.70{\tiny\ $\pm$ 11.02} & 67.06{\tiny\ $\pm$ 10.75} & 63.01{\tiny\ $\pm$ 5.68}  & 61.26 \\ \midrule
Max & 92.17{\tiny\ $\pm$ 0.52} & 84.36{\tiny\ $\pm$ 0.74} & 72.39{\tiny\ $\pm$ 6.95} & 65.87{\tiny\ $\pm$ 3.95} & 74.21 & \textcolor[RGB]{128,128,128}{90.22{\tiny\ $\pm$ 0.68}} & \textcolor[RGB]{128,128,128}{69.13{\tiny\ $\pm$ 1.14}} & \textcolor[RGB]{128,128,128}{76.65{\tiny\ $\pm$ 5.36}} & \textcolor[RGB]{128,128,128}{67.86{\tiny\ $\pm$ 2.70}} & 71.21    \\
\red{Max + }MIREL & 93.21{\tiny\ $\pm$ 0.60} & 84.96{\tiny\ $\pm$ 2.03} & 90.04{\tiny\ $\pm$ 1.40} & 78.99{\tiny\ $\pm$ 3.81} & \textbf{84.66} & 92.81{\tiny\ $\pm$ 0.56} & 68.32{\tiny\ $\pm$ 4.98} & 82.26{\tiny\ $\pm$ 5.71} & 71.44{\tiny\ $\pm$ 2.84} & \textbf{74.01}      \\ \midrule
DSMIL & 92.15{\tiny\ $\pm$ 0.85} & 84.70{\tiny\ $\pm$ 0.55} & 61.81{\tiny\ $\pm$ 4.57} & 62.89{\tiny\ $\pm$ 5.04} & 69.80 & 71.42{\tiny\ $\pm$ 2.24} & 58.22{\tiny\ $\pm$ 0.88} & 59.10{\tiny\ $\pm$ 3.19} & 53.84{\tiny\ $\pm$ 3.42} & 57.05  \\
\red{DSMIL + }MIREL  & 92.69{\tiny\ $\pm$ 0.44} & 83.02{\tiny\ $\pm$ 3.83} & 77.54{\tiny\ $\pm$ 13.29} & 64.16{\tiny\ $\pm$ 7.84} & \textbf{74.91} & 92.71{\tiny\ $\pm$ 0.50} & 60.81{\tiny\ $\pm$ 0.55} & 78.21{\tiny\ $\pm$ 8.07} & 68.30{\tiny\ $\pm$ 4.57} & \textbf{69.11}     \\ \midrule
ABMIL & 91.62{\tiny\ $\pm$ 0.62} & 86.15{\tiny\ $\pm$ 1.15} & 65.56{\tiny\ $\pm$ 10.45} & 63.43{\tiny\ $\pm$ 2.84} & 71.71 & 76.89{\tiny\ $\pm$ 1.07} & 60.33{\tiny\ $\pm$ 0.77} & 51.60{\tiny\ $\pm$ 1.29} & 44.65{\tiny\ $\pm$ 2.43} & 52.19   \\
\red{ABMIL + }MIREL  & 92.47{\tiny\ $\pm$ 0.19} & 78.43{\tiny\ $\pm$ 3.57} & 88.72{\tiny\ $\pm$ 2.78} & 77.22{\tiny\ $\pm$ 6.68} & \textbf{81.46} & 93.18{\tiny\ $\pm$ 0.32} & 66.40{\tiny\ $\pm$ 2.48} & 80.22{\tiny\ $\pm$ 4.93} & 72.49{\tiny\ $\pm$ 5.52} & \textbf{73.04}   \\  \midrule
\multicolumn{11}{l}{- \textit{Compared with related UE methods \red{using ABMIL as the base MIL network}}} \\ \midrule
Deep Ensemble & 93.33{\tiny\ $\pm$ 0.29} & 86.37{\tiny\ $\pm$ 0.91} & 65.00{\tiny\ $\pm$ 6.83} & 64.86{\tiny\ $\pm$ 4.69} & 72.08 & 78.80{\tiny\ $\pm$ 1.20} & 71.97{\tiny\ $\pm$ 0.64} & 54.65{\tiny\ $\pm$ 4.55} & 46.07{\tiny\ $\pm$ 2.11} & 57.57   \\
MC Dropout  & 92.37{\tiny\ $\pm$ 0.42} & 86.26{\tiny\ $\pm$ 1.53} & 62.36{\tiny\ $\pm$ 6.62} & 63.77{\tiny\ $\pm$ 3.31} & 70.79 & 81.91{\tiny\ $\pm$ 1.57} & 75.81{\tiny\ $\pm$ 2.09} & 64.91{\tiny\ $\pm$ 8.65} & 49.61{\tiny\ $\pm$ 3.05} & 63.44   \\
$\mathcal{I}$-EDL & 92.47{\tiny\ $\pm$ 0.19} & 78.43{\tiny\ $\pm$ 3.57} & 88.72{\tiny\ $\pm$ 2.78} & 77.22{\tiny\ $\pm$ 6.68} & 81.46 & 77.82{\tiny\ $\pm$ 0.78} & 51.61{\tiny\ $\pm$ 1.19} & 62.79{\tiny\ $\pm$ 4.12} & 54.48{\tiny\ $\pm$ 2.61} & 56.29       \\
Bayes-MIL & 92.46{\tiny\ $\pm$ 0.77} & 84.52{\tiny\ $\pm$ 1.57} & 83.74{\tiny\ $\pm$ 3.82} & 71.19{\tiny\ $\pm$ 5.28} & 79.82 & 65.14{\tiny\ $\pm$ 32.68} & 72.44{\tiny\ $\pm$ 13.36} & 74.81{\tiny\ $\pm$ 8.78} & 64.67{\tiny\ $\pm$ 6.34} & 70.64  \\  
\textbf{MIREL}  & 92.47{\tiny\ $\pm$ 0.19} & 78.43{\tiny\ $\pm$ 3.57} & 88.72{\tiny\ $\pm$ 2.78} & 77.22{\tiny\ $\pm$ 6.68} & \textbf{81.46} & 93.18{\tiny\ $\pm$ 0.32} & 66.40{\tiny\ $\pm$ 2.48} & 80.22{\tiny\ $\pm$ 4.93} & 72.49{\tiny\ $\pm$ 5.52} & \textbf{73.04}  \\  
\bottomrule
\end{tabular}
\end{scriptsize}
\end{center}
\end{table*}

\begin{figure*}[htbp]
\vskip 0.1in
\begin{center}
\centerline{\includegraphics[width=\columnwidth]{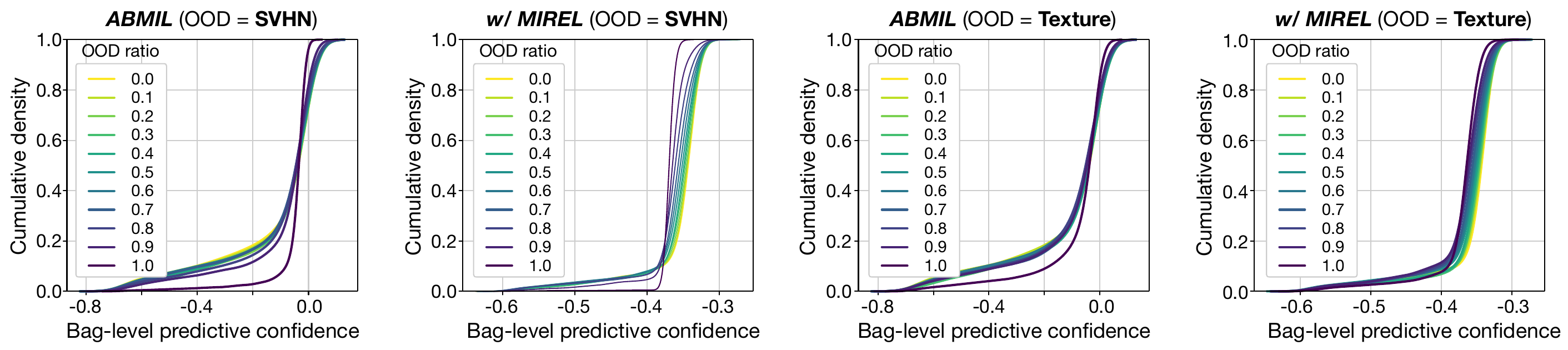}}
\caption{Distribution of bag-level predictive confidence (negative expected entropy). ID dataset is \textbf{CIFAR10-bags}.}
\label{apx-cifar-abmil-bag-unc}
\end{center}
\vskip -0.1in
\end{figure*}

\begin{figure*}[htbp]
\vskip 0.1in
\begin{center}
\centerline{\includegraphics[width=0.80\columnwidth]{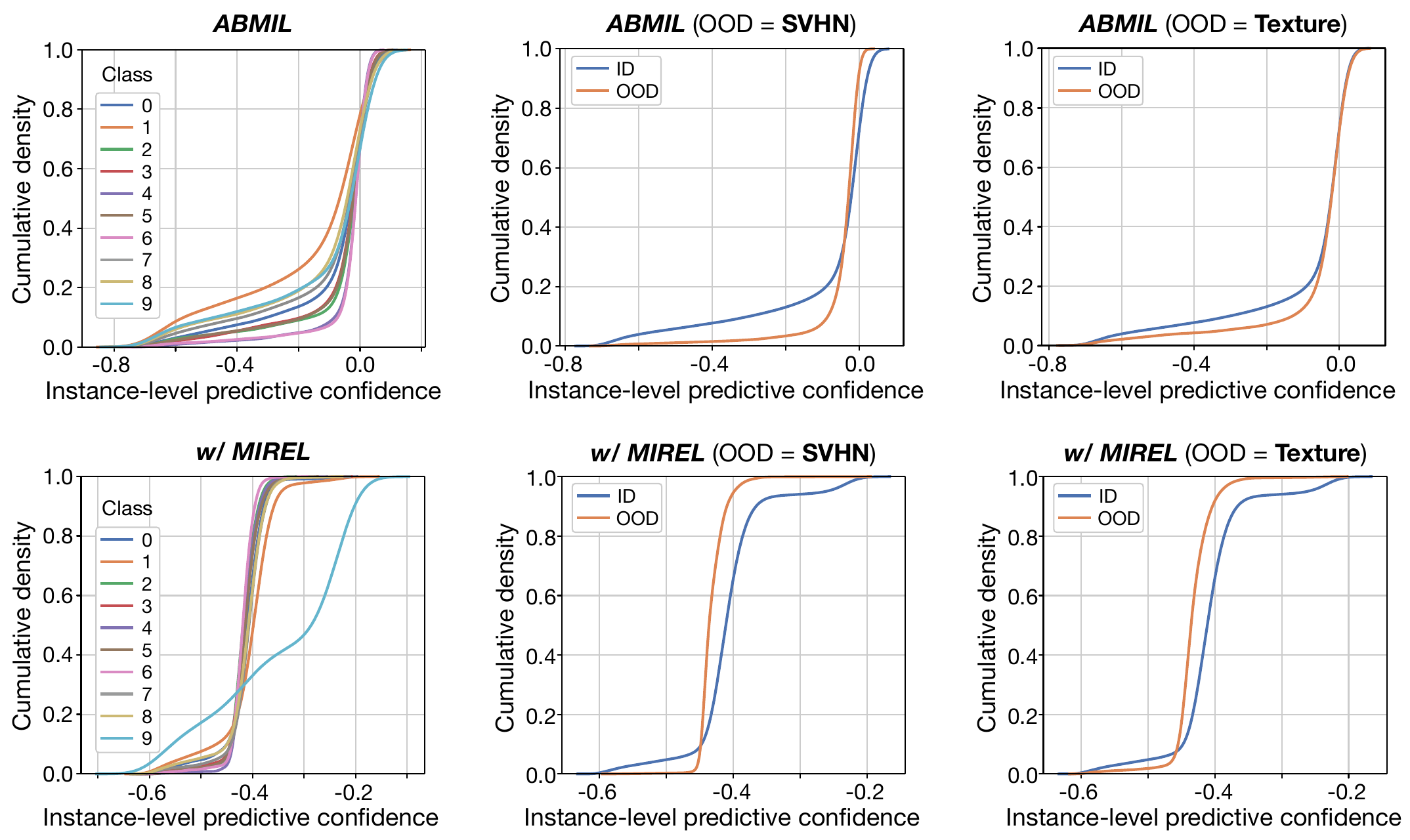}}
\caption{Distribution of instance-level predictive confidence (negative expected entropy). \textbf{CIFAR10-bags} is ID dataset.}
\label{apx-cifar-abmil-ins-unc}
\end{center}
\vskip -0.1in
\end{figure*}

\begin{figure*}[htbp]
\vskip 0.1in
\begin{center}
\centerline{\includegraphics[width=\columnwidth]{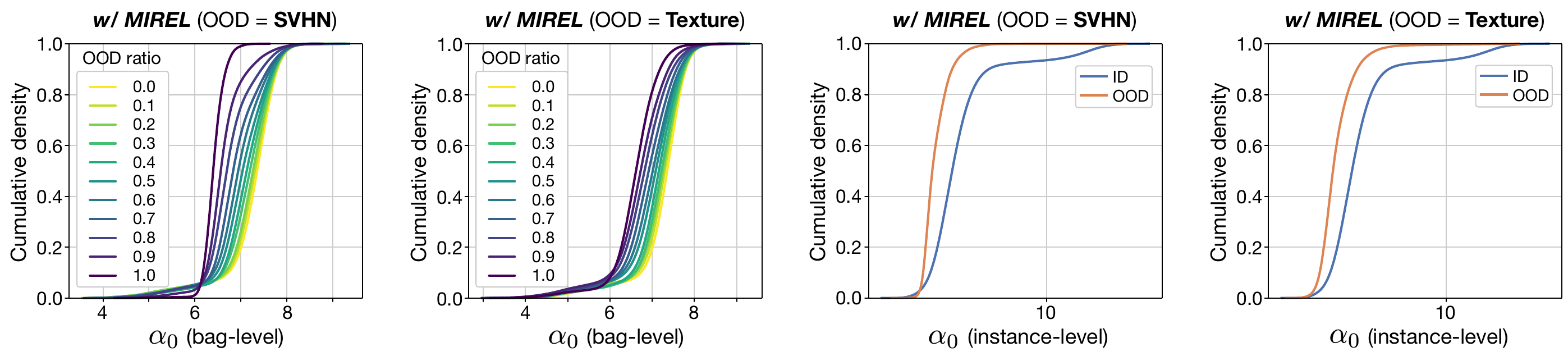}}
\caption{Distribution of bag-level and instance-level $\alpha_0$ output by the ABMIL models with our MIREL. ID dataset is \textbf{CIFAR10-bags}.}
\label{apx-cifar-abmil-alpha0}
\end{center}
\vskip -0.1in
\end{figure*}

\subsection{Uncertainty Analysis}
\label{apx-cifar-unc-analysis}

Using ABMIL as the base MIL network, here we show the results of uncertainty analysis on CIFAR10-bags, including bag-level uncertainty (Fig. \ref{apx-cifar-abmil-bag-unc}), instance-level uncertainty (Fig. \ref{apx-cifar-abmil-ins-unc}), and $\alpha_0$ distribution (Fig. \ref{apx-cifar-abmil-alpha0}). Our main findings are briefly summarized as follows. (1) The ABMIL models with our MIREL performs slightly better than vanilla ABMIL, in the predictive confidence of abnormal bags. (2) Our MIREL improves the UE capability of ABMIL at instance level by estimating less confidence for OOD instances and more confidence for ID ones. (3) The uncertainty measure of $\alpha_0$ seems better in detecting the bags with different OOD instance ratios, than that of negative expected entropy.  

\subsection{Ablation Study}
\label{apx-cifar-abl-study}

\textbf{ABMIL with our MIREL}~The result of this experiment is exhibited in Table \ref{apx-table-cifar-abl}. 
\begin{itemize}
    \item \textbf{Result analysis:} (1) EDL improves the UE capability of vanilla ABMIL models by a large margin (9.75\%) at bag level. (2) adopting our derived $T(\mathbf{x})$ rather than $a_k$ for instance prediction often leads to large improvements in overall UE performance, 16.40\% and 16.67\% for the ABMIL without and with EDL, respectively. (3) our residual instance estimator $R(\mathbf{x})$ shows comparable UE performance with $T(\mathbf{x})$ on CIFAR10-bags. 
    \item \textbf{\red{Explanation for} the same bag-level performance:} Note that our MIREL obtains the same bag-level performance as its counterparts, \textit{i.e.}, the EDL-based ABMIL network without our $R(\mathbf{x})$. It is because we \textbf{only optimize the parameter $\pi$}, instead of optimizing both $\pi$ and $\psi$, in $\mathcal{L}_{\text{MIREL}}$. \red{We choose to do so} as we empirically find that a \textbf{deeper instance encoder}, \textit{e.g.}, the network with more than 4 convolutional layers, often leads to unstable training in the weakly-supervised instance-level estimator. One possible reason is that the \textit{weak supervision signals} used for training the instance-level estimator are more likely to vanish in its gradient back-propagation to the deeper layers of instance encoder. \red{Such behavior is also discussed and highlighted in \citeauthor{li2023task}~\citeyearpar{li2023task}.} We leave its investigation as future work. 
    \red{\item \textbf{Clarification on the setting of instance encoder:} In fact, a deep instance encoder is not a common choice in most MIL applications; instead, a high-dimensional single instance is usually first transformed into a low-dimensional vector and then \textbf{a shallow network}, \textit{e.g.} shallow MLP, is utilized as the \textit{real} instance encoder for MIL. This fact can be seen from many real-world MIL applications \cite{lu2021data,liu10385148,tian2021weakly,sapkota2022bayesian,Rizve_2023_CVPR}. This means that, in most cases, a shallow instance encoder is a universal setting so our $\mathcal{L}_{\text{MIREL}}$ can be leveraged as expected to optimize both $\pi$ and $\psi$ and enhance both instance-level and bag-level UE performance.} 
\end{itemize}

\begin{table*}[ht]
\caption{Ablation study on the ABMIL with our MIREL. \textbf{CIFAR10-bags} is ID dataset.}
\label{apx-table-cifar-abl}
\begin{center}
\begin{scriptsize}
\tabcolsep=0.16cm
\begin{tabular}{ccc|ccccc|ccccc}
\toprule
\multicolumn{2}{c}{\underline{\textbf{\ \ Loss\ \ }}}   & \multirow{2}{*}{\textbf{Ins.}} & \multicolumn{5}{c|}{\underline{\textbf{\ \ Bag-level\ \ }}}            & \multicolumn{5}{c}{\underline{\textbf{\ \ Instance-level\ \ }}}       \\
 $\mathcal{L}_{\mathcal{I}\text{-EDL}}$ & $\mathcal{L}_{\text{MIREL}}$ & & \textbf{Acc.} & \textbf{Conf.} & \textbf{OOD-S} & \textbf{OOD-T} & \textbf{$\overline{\textit{UE}}$} & \textbf{Acc.} & \textbf{Conf.} & \textbf{OOD-S} & \textbf{OOD-T} & \textbf{$\overline{\textit{UE}}$} \\
\midrule
 & & $a_k$ & 91.62{\tiny\ $\pm$ 0.62} & 86.15{\tiny\ $\pm$ 1.15} & 65.56{\tiny\ $\pm$ 10.45} & 63.43{\tiny\ $\pm$ 2.84} & 71.71 &  76.89{\tiny\ $\pm$ 1.07} & 60.33{\tiny\ $\pm$ 0.77} & 51.60{\tiny\ $\pm$ 1.29} & 44.65{\tiny\ $\pm$ 2.43} & 52.19      \\
 &  & $T$ & 91.62{\tiny\ $\pm$ 0.62} & 86.15{\tiny\ $\pm$ 1.15} & 65.56{\tiny\ $\pm$ 10.45} & 63.43{\tiny\ $\pm$ 2.84} & 71.71 &  91.76{\tiny\ $\pm$ 0.40} & 80.54{\tiny\ $\pm$ 2.28} & 67.79{\tiny\ $\pm$ 7.60} & 57.45{\tiny\ $\pm$ 5.24} & 68.59        \\
 \checkmark &  & $a_k$ & 92.47{\tiny\ $\pm$ 0.19} & 78.43{\tiny\ $\pm$ 3.57} & 88.72{\tiny\ $\pm$ 2.78} & 77.22{\tiny\ $\pm$ 6.68} & 81.46 & 77.82{\tiny\ $\pm$ 0.78} & 51.61{\tiny\ $\pm$ 1.19} & 62.79{\tiny\ $\pm$ 4.12} & 54.48{\tiny\ $\pm$ 2.61} & 56.29       \\
 \checkmark &  & $T$  & 92.47{\tiny\ $\pm$ 0.19} & 78.43{\tiny\ $\pm$ 3.57} & 88.72{\tiny\ $\pm$ 2.78} & 77.22{\tiny\ $\pm$ 6.68} & 81.46 &  92.54{\tiny\ $\pm$ 0.40} & 63.83{\tiny\ $\pm$ 3.65} & 82.25{\tiny\ $\pm$ 10.73} & 72.81{\tiny\ $\pm$ 8.04} & 72.96      \\
 \checkmark & \checkmark & $R$  & 92.47{\tiny\ $\pm$ 0.19} & 78.43{\tiny\ $\pm$ 3.57} & 88.72{\tiny\ $\pm$ 2.78} &  77.22{\tiny\ $\pm$ 6.68} & \textbf{81.46} &  93.18{\tiny\ $\pm$ 0.32} & 66.40{\tiny\ $\pm$ 2.48} & 80.22{\tiny\ $\pm$ 4.93} & 72.49{\tiny\ $\pm$ 5.52} & \textbf{73.04}   \\  
\bottomrule
\end{tabular}
\end{scriptsize}
\end{center}
\end{table*}

\textbf{Optimization strategy for $R(\mathbf{x})$}~Similar to that done on MNIST-bags, we test different optimization strategies on CIFAR10-bags. 
Test results are shown in Table \ref{apx-table-cifar-ins-loss}. Note that, bag-level results are dropped, since only $\pi$ is involved in the training of $R(\mathbf{x})$ (as explained above) and different strategies lead to the same bag-level performance. From Table \ref{apx-table-cifar-ins-loss}, we could find that $\mathcal{L}_{\text{ins}}^{+}$ often obtains the best UE performance at instance level, surpassing the second-placed $\mathcal{L}_{2}$ by 1.86\% on average. This could further confirm the superiority of our weakly-supervised evidential learning strategy. 

\begin{table*}[htbp]
\caption{Ablation study on the loss function used for training $R(\mathbf{x})$. The base MIL network is ABMIL and it is trained on \textbf{CIFAR10-bags}.}
\label{apx-table-cifar-ins-loss}
\begin{center}
\begin{small}
\begin{tabular}{l|ccccc}
\toprule
\multirow{2}{*}{\textbf{Loss}} & \multicolumn{5}{c}{\underline{\textbf{\ \ Instance-level\ \ }}}       \\
 & \textbf{Acc.} & \textbf{Conf.} & \textbf{OOD-S} & \textbf{OOD-T} & \textbf{$\overline{\textit{UE}}$} \\
\midrule
$\mathcal{L}_{1}$ & 93.18{\tiny\ $\pm$ 0.32} & 61.26{\tiny\ $\pm$ 2.65} & 75.80{\tiny\ $\pm$ 4.89} & 69.31{\tiny\ $\pm$ 5.69} & 68.79  \\
$\mathcal{L}_{2}$ & 93.19{\tiny\ $\pm$ 0.34} & 64.22{\tiny\ $\pm$ 3.25} & 78.42{\tiny\ $\pm$ 3.53} & 70.91{\tiny\ $\pm$ 5.28} & 71.18  \\
$\mathcal{L}_{\text{ins}}^{+}$ & 93.18{\tiny\ $\pm$ 0.32} & \textbf{66.40}{\tiny\ $\pm$ 2.48} & \textbf{80.22}{\tiny\ $\pm$ 4.93} & \textbf{72.49}{\tiny\ $\pm$ 5.52} & \textbf{73.04}  \\  
\bottomrule
\end{tabular}
\end{small}
\end{center}
\end{table*}

\red{\textbf{The effect of $\mathcal{L}_{\text{RED}}$ on our MIREL}~The results of this experiment are shown in Table \ref{apx-table-cifar-red-loss}. From these results, we observe that on CIFAR10-bags, using $\mathcal{L}_{\text{RED}}$ often leads to worse performances in UE, although it is better in overall bag-level UE performance. Nevertheless, we choose to use $\mathcal{L}_{\text{RED}}$ in our baseline approach by default for simplicity. In other words, the setting of $\mathcal{L}_{\text{RED}}$ is simply shared between all experiments and is not fine-tuned for different datasets, although fine-tuning it could lead to better performances in MIUE.}

\begin{table*}[ht]
\caption{\red{Ablation study on the effect of RED loss on our MIREL (\textbf{CIFAR10-bags}). The base MIL network is ABMIL.}}
\label{apx-table-cifar-red-loss}
\begin{center}
\begin{scriptsize}
\red{
\begin{tabular}{c|ccccc|ccccc}
\toprule
\multirow{2}{*}{$\mathcal{L}_{\text{RED}}$} & \multicolumn{5}{c|}{\underline{\textbf{\ \ Bag-level\ \ }}}            & \multicolumn{5}{c}{\underline{\textbf{\ \ Instance-level\ \ }}}       \\
 & \textbf{Acc.} & \textbf{Conf.} & \textbf{OOD-S} & \textbf{OOD-T} & \textbf{$\overline{\textit{UE}}$} & \textbf{Acc.} & \textbf{Conf.} & \textbf{OOD-S} & \textbf{OOD-T} & \textbf{$\overline{\textit{UE}}$} \\
\midrule
$\times$ & 92.80{\tiny\ $\pm$ 0.41} & 84.23{\tiny\ $\pm$ 2.50} & 69.05{\tiny\ $\pm$ 23.26} & 78.27{\tiny\ $\pm$ 6.84} & 77.18 & 93.19{\tiny\ $\pm$ 0.38} & 73.10{\tiny\ $\pm$ 2.80} & 81.57{\tiny\ $\pm$ 5.80} & 74.56{\tiny\ $\pm$ 3.23} & \textbf{76.41}   \\
\checkmark & 92.47{\tiny\ $\pm$ 0.19} & 78.43{\tiny\ $\pm$ 3.57} & 88.72{\tiny\ $\pm$ 2.78} &  77.22{\tiny\ $\pm$ 6.68} & \textbf{81.46} &  93.18{\tiny\ $\pm$ 0.32} & 66.40{\tiny\ $\pm$ 2.48} & 80.22{\tiny\ $\pm$ 4.93} & 72.49{\tiny\ $\pm$ 5.52} & 73.04  \\  
\bottomrule
\end{tabular}
}
\end{scriptsize}
\end{center}
\end{table*}

\textbf{Related UE methods}~As shown in Table \ref{apx-table-cifar-ins}, there are large improvements in overall UE performance for compared UE methods, when turning to adopt our $T(\mathbf{x})$ derived from $S(X)$ as instance-level estimator. These improvements are 11.22\%, 3.00\%, and 16.67\% for Deep Ensemble, MC Dropout, and $\mathcal{I}\text{-EDL}$, respectively. These again demonstrate our argument, \textit{i.e.}, attention-dependent scoring proxies may not be suitable for instance-level prediction.

\begin{table*}[htbp]
\caption{Additional instance-level ablation study on $T(\mathbf{x})$ for related UE methods (\textbf{CIFAR10-bags}). $\dagger$ These methods directly adopt our $T(\mathbf{x})$ derived from $S(X)$ for instance-level estimation. Other results are copied from Table \ref{apx-table-cifar} for comparisons.}
\label{apx-table-cifar-ins}
\begin{center}
\begin{small}
\begin{tabular}{l|c|ccccc}
\toprule
\multirow{2}{*}{\textbf{Method}} & \multirow{2}{*}{\textbf{Ins.}} & \multicolumn{5}{c}{\underline{\textbf{\ \ Instance-level\ \ }}}       \\
 & & \textbf{Acc.} & \textbf{Conf.} & \textbf{OOD-S} & \textbf{OOD-T} & \textbf{$\overline{\textit{UE}}$} \\
\midrule
Deep Ensemble & $a_k$ & 78.80{\tiny\ $\pm$ 1.20} & 71.97{\tiny\ $\pm$ 0.64} & 54.65{\tiny\ $\pm$ 4.55} & 46.07{\tiny\ $\pm$ 2.11}  & 57.57 \\
Deep Ensemble $\dagger$ & $T$ & 93.28{\tiny\ $\pm$ 0.16} & 85.57{\tiny\ $\pm$ 1.14} & 64.92{\tiny\ $\pm$ 6.61} & 55.88{\tiny\ $\pm$ 4.33} & \textbf{68.79}  \\ \midrule
MC Dropout  & $a_k$ & 81.91{\tiny\ $\pm$ 1.57} & 75.81{\tiny\ $\pm$ 2.09} & 64.91{\tiny\ $\pm$ 8.65} & 49.61{\tiny\ $\pm$ 3.05}  & 63.44 \\
MC Dropout $\dagger$  & $T$ & 92.62{\tiny\ $\pm$ 0.73} & 82.65{\tiny\ $\pm$ 1.00} & 62.64{\tiny\ $\pm$ 5.92} & 54.03{\tiny\ $\pm$ 3.10} & \textbf{66.44}  \\ \midrule
$\mathcal{I}$-EDL & $a_k$ & 77.82{\tiny\ $\pm$ 0.78} & 51.61{\tiny\ $\pm$ 1.19} & 62.79{\tiny\ $\pm$ 4.12} & 54.48{\tiny\ $\pm$ 2.61}  & 56.29 \\
$\mathcal{I}$-EDL $\dagger$ & $T$ & 92.54{\tiny\ $\pm$ 0.40} & 63.83{\tiny\ $\pm$ 3.65} & 82.25{\tiny\ $\pm$ 10.73} & 72.81{\tiny\ $\pm$ 8.04}  & \textbf{72.96}  \\ 
\bottomrule
\end{tabular}
\end{small}
\end{center}
\end{table*}

\subsection{\red{More Experiments with Different Settings}}
\label{apx-cifar-more-settings}

\red{Similar to those experiments shown in Section \ref{apx-mnist-more-settings}, in this section we conduct more experiments with different settings to investigate the effect of these settings on our MIREL scheme.}

\red{\textbf{Gated attention mechanism for ABMIL}~As shown in Table \ref{apx-table-cifar-gated-attn}, there is no significant difference in average UE performance between the two attention mechanisms. We choose the standard attention operator by default for ABMIL network in all experiments, because it is more efficient in computation and is adopted more frequently than its gated version although it sometimes performs slightly worse than its gated version in UE tasks.}

\begin{table*}[ht]
\caption{\red{Performance of our MIREL when using standard or gated attention mechanism for ABMIL (\textbf{CIFAR10-bags}).}}
\label{apx-table-cifar-gated-attn}
\begin{center}
\begin{scriptsize}
\red{
\begin{tabular}{c|ccccc|ccccc}
\toprule
\multirow{2}{*}{\textbf{Attention}} & \multicolumn{5}{c|}{\underline{\textbf{\ \ Bag-level\ \ }}}            & \multicolumn{5}{c}{\underline{\textbf{\ \ Instance-level\ \ }}}       \\
 & \textbf{Acc.} & \textbf{Conf.} & \textbf{OOD-S} & \textbf{OOD-T} & \textbf{$\overline{\textit{UE}}$} & \textbf{Acc.} & \textbf{Conf.} & \textbf{OOD-S} & \textbf{OOD-T} & \textbf{$\overline{\textit{UE}}$} \\
\midrule
Gated & 92.25{\tiny\ $\pm$ 0.64} & 81.76{\tiny\ $\pm$ 1.42} & 88.73{\tiny\ $\pm$ 4.26} & 76.54{\tiny\ $\pm$ 6.56} & \textbf{82.35} & 93.08{\tiny\ $\pm$ 0.17} & 62.96{\tiny\ $\pm$ 3.76} & 85.98{\tiny\ $\pm$ 1.14} & 73.10{\tiny\ $\pm$ 3.23} & \textbf{74.01}   \\
Standard & 92.47{\tiny\ $\pm$ 0.19} & 78.43{\tiny\ $\pm$ 3.57} & 88.72{\tiny\ $\pm$ 2.78} &  77.22{\tiny\ $\pm$ 6.68} & 81.46 &  93.18{\tiny\ $\pm$ 0.32} & 66.40{\tiny\ $\pm$ 2.48} & 80.22{\tiny\ $\pm$ 4.93} & 72.49{\tiny\ $\pm$ 5.52} & 73.04  \\  
\bottomrule
\end{tabular}
}
\end{scriptsize}
\end{center}
\end{table*}

\red{\textbf{Comparison with UE methods on DSMIL}~As shown in Table \ref{apx-table-cifar-cmp-ue-on-dsmil}, our MIREL also could often perform better than other UE methods by a large margin at instance level on DSMIL. This result further suggests the adaptability of our method.}

\begin{table*}[htbp]
\caption{\red{Comparison with UE methods when using DSMIL as the base MIL network (\textbf{CIFAR10-bags}). The baseline of this experiment is \red{vanilla DSMIL without any additional UE techniques}. Bayes-MIL is not compared here because it is not compatible with DSMIL.} 
}
\label{apx-table-cifar-cmp-ue-on-dsmil}
\begin{center}
\begin{scriptsize}
\tabcolsep=0.19cm
\red{
\begin{tabular}{l|ccccc|ccccc}
\toprule
\multirow{2}{*}{\textbf{Method}} & \multicolumn{5}{c|}{\underline{\textbf{\ \ Bag-level\ \ }}}            & \multicolumn{5}{c}{\underline{\textbf{\ \ Instance-level\ \ }}}       \\
 & \textbf{Acc.} & \textbf{Conf.} & \textbf{OOD-S} & \textbf{OOD-T} & \textbf{$\overline{\textit{UE}}$} & \textbf{Acc.} & \textbf{Conf.} & \textbf{OOD-S} & \textbf{OOD-T} & \textbf{$\overline{\textit{UE}}$} \\
\midrule
Baseline & \textcolor[RGB]{128,128,128}{92.15{\tiny\ $\pm$ 0.85}} & \textcolor[RGB]{128,128,128}{84.70{\tiny\ $\pm$ 0.55}} & \textcolor[RGB]{128,128,128}{61.81{\tiny\ $\pm$ 4.57}} & \textcolor[RGB]{128,128,128}{62.89{\tiny\ $\pm$ 5.04}} & \textcolor[RGB]{128,128,128}{69.80} & \textcolor[RGB]{128,128,128}{71.42{\tiny\ $\pm$ 2.24}} & \textcolor[RGB]{128,128,128}{58.22{\tiny\ $\pm$ 0.88}} & \textcolor[RGB]{128,128,128}{59.10{\tiny\ $\pm$ 3.19}} & \textcolor[RGB]{128,128,128}{53.84{\tiny\ $\pm$ 3.42}} & \textcolor[RGB]{128,128,128}{57.05}     \\ \midrule
Deep Ensemble & 93.20{\tiny\ $\pm$ 0.14} & 86.45{\tiny\ $\pm$ 0.74} & 70.50{\tiny\ $\pm$ 4.34} & 63.57{\tiny\ $\pm$ 3.52} & 73.51 & 74.25{\tiny\ $\pm$ 1.30} & 66.49{\tiny\ $\pm$ 1.01} & 63.78{\tiny\ $\pm$ 4.72} & 57.86{\tiny\ $\pm$ 4.81} & 62.71  \\
MC Dropout  & 92.39{\tiny\ $\pm$ 0.59} & 84.81{\tiny\ $\pm$ 1.56} & 72.61{\tiny\ $\pm$ 8.90} & 67.15{\tiny\ $\pm$ 6.69} & 74.86 & 73.82{\tiny\ $\pm$ 1.80} & 63.18{\tiny\ $\pm$ 1.53} & 66.93{\tiny\ $\pm$ 10.46} & 53.46{\tiny\ $\pm$ 6.01} & 61.19    \\
$\mathcal{I}$-EDL & 92.69{\tiny\ $\pm$ 0.44} & 83.02{\tiny\ $\pm$ 3.83} & 77.54{\tiny\ $\pm$ 13.29} & 64.16{\tiny\ $\pm$ 7.84} & 74.91 & 69.07{\tiny\ $\pm$ 8.63} & 57.98{\tiny\ $\pm$ 0.62} & 52.22{\tiny\ $\pm$ 4.02} & 49.66{\tiny\ $\pm$ 4.51} & 53.29  \\ 
\textbf{MIREL}  & 92.69{\tiny\ $\pm$ 0.44} & 83.02{\tiny\ $\pm$ 3.83} & 77.54{\tiny\ $\pm$ 13.29} & 64.16{\tiny\ $\pm$ 7.84} & \textbf{74.91} & 92.71{\tiny\ $\pm$ 0.50} & 60.81{\tiny\ $\pm$ 0.55} & 78.21{\tiny\ $\pm$ 8.07} & 68.30{\tiny\ $\pm$ 4.57} & \textbf{69.11} \\ 
\bottomrule
\end{tabular}
}
\end{scriptsize}
\end{center}
\end{table*}

\section{Additional Results on Histopathology Dataset}
\label{apx-wsi-result}

\subsection{Related UE Methods}
As shown in Table \ref{apx-table-c16-ins}, our $T(\mathbf{x})$ obtains considerable improvements over $a_k$ in overall UE performance, even better than MIREL at instance level. These improvements are 40.08\%, 38.57\%, and 40.91\% for Deep Ensemble, MC Dropout, and $\mathcal{I}\text{-EDL}$, respectively. Such impressive results may result from two main factors:
\begin{itemize}
    \item $a_k$ is obtained by $\mathop{\mathrm{softmax}}$, so its values for negative instances would be extremely small when instance number is very large (recall that there are 11,753 instances in each CAMELYON16 bag on average), thus more likely to yield overconfident estimations. 
    \item our $T(\mathbf{x})$ is directly deduced from $S(X)$, with the ability of distinguishing between negative and positive instances when $S(X)$ is good enough at classifying bags, as stated in Section \ref{sec-ins-est}. 
\end{itemize}

\begin{table*}[ht]
\caption{Additional instance-level results of related UE methods on \textbf{CAMELYON16}. $\dagger$ These methods directly adopt our $T(\mathbf{x})$ derived from $S(X)$ for instance-level estimation. The other results are copied from Table \ref{table-c16}. Particularly, the AUROC of ID instance classification, along with Acc., is reported due to the dominant negative patches ($\sim$ 97.7\%) in the slides of CAMELYON16.}
\label{apx-table-c16-ins}
\begin{center}
\begin{small}
\begin{tabular}{l|c|ccccc}
\toprule
\multirow{2}{*}{\textbf{Method}} & \multirow{2}{*}{\textbf{Ins.}} & \multicolumn{5}{c}{\underline{\textbf{\ \ Instance-level\ \ }}}       \\
 &  & \textbf{Acc.} & \textbf{AUROC} & \textbf{Conf.} & \textbf{OOD-PRAD} & \textbf{$\overline{\textit{UE}}$} \\
\midrule
Deep Ensemble & $a_k$ & 96.08{\tiny\ $\pm$ 0.02} & 50.34{\tiny\ $\pm$ 2.53} & 49.62{\tiny\ $\pm$ 2.53} & 28.16{\tiny\ $\pm$ 1.07} & 38.89 \\
Deep Ensemble $\dagger$ & $T$ & 89.38{\tiny\ $\pm$ 5.23} & 95.09{\tiny\ $\pm$ 0.17} & 86.36{\tiny\ $\pm$ 3.31} & 71.58{\tiny\ $\pm$ 7.25} & \textbf{78.97} \\ \midrule
MC Dropout  & $a_k$ & 96.05{\tiny\ $\pm$ 0.00} & 56.25{\tiny\ $\pm$ 2.16} & 56.35{\tiny\ $\pm$ 2.20} & 33.93{\tiny\ $\pm$ 2.05} & 45.14 \\
MC Dropout $\dagger$  & $T$ & 94.06{\tiny\ $\pm$ 2.08} & 94.07{\tiny\ $\pm$ 0.37} & 88.83{\tiny\ $\pm$ 0.85} & 78.60{\tiny\ $\pm$ 3.05} & \textbf{83.71}  \\ \midrule
$\mathcal{I}$-EDL & $a_k$ & 96.05{\tiny\ $\pm$ 0.01} & 45.39{\tiny\ $\pm$ 4.78} & 45.41{\tiny\ $\pm$ 5.33} & 32.06{\tiny\ $\pm$ 1.49} & 38.74  \\
$\mathcal{I}$-EDL $\dagger$ & $T$ & 87.53{\tiny\ $\pm$ 5.61} & 95.11{\tiny\ $\pm$ 0.26} & 87.28{\tiny\ $\pm$ 4.12} & 72.02{\tiny\ $\pm$ 7.87} & \textbf{79.65}  \\ 
\bottomrule
\end{tabular}
\end{small}
\end{center}
\end{table*}

\subsection{Distribution Shift Detection}
\label{apx-wsi-result-ds}
The numerical results of Fig. \ref{c16-dist-shift} are presented in Table \ref{apx-c16-shift}. Similar to that provided in Section \ref{sec-c16-result}, our result analysis of Table \ref{apx-c16-shift} is as follows. (1) \textbf{Bag-level}. The AUROC performance on lighter DS is often less than 0.52, namely, lighter DS is hard to detect for all presented UE methods. Our MIREL can detect light DS with an AUROC of 0.59 and strong DS with an AUROC of 0.75, obtaining the best or the second best performance. Moreover, its overall UE performance is the best (0.62). 
(2) \textbf{instance-level}. All compared methods not specially for MIL, consistently obtain an AUROC less than 0.5, indicating meaningless detection results on three shift datasets. On strong DS, our MIREL obtains an AUROC of 0.62, exceeding Bayes-MIL by 5.09\%. This experiment could further verify the superiority of our MIREL scheme in MIUE. 

\begin{table*}[ht]
\caption{Comparison with related UE methods on histopathology dataset (\textbf{CAMELYON16}). Three shifted versions of CAMELYON16 test set are used for detection. \textbf{DS} means Distribution Shift, and `\textbf{lighter}', `\textbf{light}', and `\textbf{strong}' indicate three degrees of shift. The baseline of this experiment is \red{vanilla ABMIL without any additional UE techniques}.  \textbf{$\overline{\textit{UE}}$} is the average metrics on three DS detection tasks.}
\label{apx-c16-shift}
\begin{center}
\begin{small}
\begin{tabular}{l|cccc|cccc}
\toprule
\multirow{2}{*}{\textbf{Method}} & \multicolumn{4}{c|}{\underline{\textbf{\ \ Bag-level\ \ }}}            & \multicolumn{4}{c}{\underline{\textbf{\ \ Instance-level\ \ }}}       \\
 & \textbf{DS-lighter} & \textbf{DS-light} & \textbf{DS-strong} & \textbf{$\overline{\textit{UE}}$} & \textbf{DS-lighter} & \textbf{DS-light} & \textbf{DS-strong} & \textbf{$\overline{\textit{UE}}$} \\
\midrule
Baseline & \textcolor[RGB]{128,128,128}{50.86{\tiny\ $\pm$ 0.38}} & \textcolor[RGB]{128,128,128}{51.87{\tiny\ $\pm$ 1.36}} & \textcolor[RGB]{128,128,128}{54.48{\tiny\ $\pm$ 6.31}} & \textcolor[RGB]{128,128,128}{52.40} & \textcolor[RGB]{128,128,128}{49.17{\tiny\ $\pm$ 0.35}} & \textcolor[RGB]{128,128,128}{46.87{\tiny\ $\pm$ 0.53}} & \textcolor[RGB]{128,128,128}{42.88{\tiny\ $\pm$ 1.87}} & \textcolor[RGB]{128,128,128}{46.31} \\ \midrule
Deep Ensemble & 50.70{\tiny\ $\pm$ 0.40} & 52.39{\tiny\ $\pm$ 1.23} & 53.62{\tiny\ $\pm$ 9.06} & 52.24 & 49.09{\tiny\ $\pm$ 0.18} & 45.86{\tiny\ $\pm$ 0.87} & 40.73{\tiny\ $\pm$ 0.94} & 45.23   \\
MC Dropout  & 50.49{\tiny\ $\pm$ 0.49} & 52.54{\tiny\ $\pm$ 1.42} & 52.70{\tiny\ $\pm$ 5.35} & 51.91 & 49.68{\tiny\ $\pm$ 0.61} & 47.80{\tiny\ $\pm$ 1.23} & 43.89{\tiny\ $\pm$ 2.07} & 47.12   \\
$\mathcal{I}$-EDL & 51.50{\tiny\ $\pm$ 0.34} & 57.46{\tiny\ $\pm$ 2.48} & 72.74{\tiny\ $\pm$ 3.58} & 60.57 & 49.11{\tiny\ $\pm$ 0.33} & 45.84{\tiny\ $\pm$ 1.51} & 38.89{\tiny\ $\pm$ 2.84} & 44.61    \\  
Bayes-MIL & 50.19{\tiny\ $\pm$ 0.46} & 52.20{\tiny\ $\pm$ 1.45} & \textbf{77.42}{\tiny\ $\pm$ 6.32} & 59.94 & \textbf{50.29}{\tiny\ $\pm$ 0.21} & 50.36{\tiny\ $\pm$ 1.19} & 56.95{\tiny\ $\pm$ 3.92} & 52.53   \\
\textbf{MIREL}  & \textbf{51.98}{\tiny\ $\pm$ 0.89} & \textbf{58.97}{\tiny\ $\pm$ 2.14} & 74.84{\tiny\ $\pm$ 1.71} & \textbf{61.93} & 50.11{\tiny\ $\pm$ 0.18} & \textbf{50.72}{\tiny\ $\pm$ 0.75} & \textbf{62.04}{\tiny\ $\pm$ 1.19} & \textbf{54.29}   \\
\bottomrule
\end{tabular}
\end{small}
\end{center}
\end{table*}

\section{Synthetic MIUE Experiment}
\label{apx-syn-exp-vis}

To understand the UE behavior of our \textit{weakly-supervised} instance estimator, we synthesize a simple bag dataset with 2-dimensional instances. The surface of predictive probability and predictive uncertainty are visualized in a 2D plane for intuitive interpretation. 

\subsection{MIL Dataset}

\textbf{2D instance generation}~We generate 2D instances using three isotropic Gaussian with $\sigma^2=0.1$. The centroid of three Gaussian are located in (0, 1.5), (--1.5, --0.5), and (1.5, --0.5). Each Gaussian contains 1,000 points (instances). 
The Gaussian with a centroid of (0, 1.5) is the class of interest (positive), and the remaining two are negative. Note that, actually, instance labels are \textit{unknown} in training. 

\textbf{Bag generation}~Following the process described in Appendix \ref{apx-gen-bag}, we synthesize bags using the 2D instances generated above. Finally, there are 2,000 bags for training and 500 bags for validation and early stopping. Only bag-level labels are utilized for training MIL networks. 

\subsection{Implementation Details}

ABMIL is adopted as the base MIL network in this experiment.
Its instance encoder is implemented by an MLP with two layers. In network training, learning rate is set to $5\times 10^{-5}$ and $\lambda_1=0.4$. The patience step for learning rate decay and early stopping are 5 and 10 epochs, respectively. Other settings are the same as those given in Appendix \ref{apx-net-imp-training}.

\subsection{Result Visualization}

Similar to the settings of ablation study, there are three models used for analysis and comparison, as explained in Table \ref{table-baselines}. 
Their results are visualized in Fig. \ref{apx-syn-abmil-alpha0}. Our main observations are as follows:
\begin{itemize}
    \item For the $a_k$ of ABMIL, it shows high confidence in the region near and below negative instances, but low in the region near or above positive ones. It is mainly caused by the $\mathop{\mathrm{softmax}}$ operator in attention score calculation. Generally, $\mathop{\mathrm{softmax}}$ would lead to small $a_k$ for positive instances when multiple positive instances are contained in a bag. As a result, there would be relatively large entropy and high uncertainty for positive instances, compared to negative ones. 
    \item For the $T(\mathbf{x})$ of ABMIL, it seems better than $a_k$ in instance classification. However, 
    it only predicts high uncertainty near the boundary between positive and negative instances, showing overconfidence in the region far from ID instances. This behavior is very similar to that of standard classification models, possibly caused by the ignorance of epistemic uncertainty or distributional uncertainty in predictive modeling. 
    \item For the $T(\mathbf{x})$ of EDL-based ABMIL, it captures the uncertainty in some regions far from ID instances, owing to its Dirichlet-based predictive uncertainty modeling. For Dirichlet-based models, the uncertainty caused by the distributional mismatch between training and test is specially considered and incorporated into model prediction \cite{malinin2018predictive,ulmer2023prior}. 
    \item For the $R(\mathbf{x})$ of MIREL-based ABMIL, it further improve the quality of predictive uncertainty. Especially in the region near positive instances, $R(\mathbf{x})$ often predicts less uncertainty than the $T(\mathbf{x})$ of EDL-based ABMIL. This improvement is largely due to our residual evidential learning scheme. As stated in Section \ref{ins-edl}, our $R(\mathbf{x})$ is specially proposed to learn instance-specific residuals and is encouraged to compensate for the initial biased $T(\mathbf{x})$. 
\end{itemize}

\textbf{\ding{45}~Discussion}~This synthetic MIL experiment could assist us in understanding the UE behavior of different weakly-supervised instance estimators. At the same time, it could be found that there is still room for further improvements. For example, $R(\mathbf{x})$ cannot obtain desirable UE results in some regions far from negative instances. This could be one of the main challenges posed by weak supervision. We leave its solution as future work.

\begin{figure*}[ht]
\vskip 0.2in
\begin{center}
\centerline{\includegraphics[width=0.9\columnwidth]{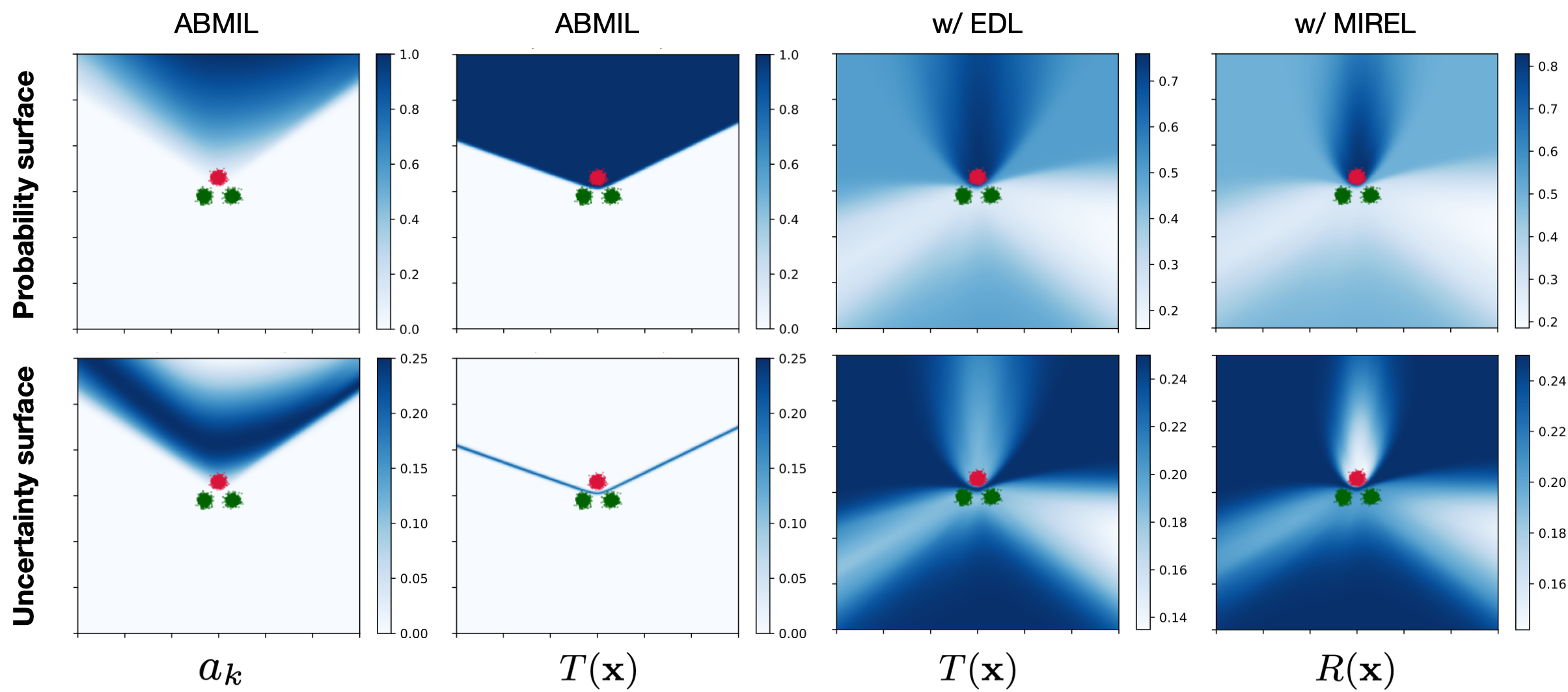}}
\caption{Visualization of the instance prediction given by different weakly-supervised estimators. A synthetic MIL dataset with 2D instances is used in this experiment. The points colored in red and green are positive and negative instance, respectively. Note that, unlike standard fully-supervised settings, there are no complete instance labels to use for training.}
\label{apx-syn-abmil-alpha0}
\end{center}
\vskip -0.2in
\end{figure*}

\end{document}